\documentclass[11pt]{article}

\usepackage[round]{natbib}
\setcitestyle{authoryear} 
\usepackage{authblk}

\usepackage{amsmath}
\usepackage{amssymb}
\usepackage{amsthm}
\usepackage{mathtools}

\usepackage[utf8]{inputenc}
\usepackage[margin=1.3in]{geometry}
\usepackage{times}
\usepackage{hyperref}
\usepackage{url}
\usepackage{multirow}
\usepackage{xr}
\usepackage{fontawesome}
\usepackage{bbm}
\usepackage{todonotes}
\usepackage{listings}
\usepackage[capitalise]{cleveref}
\usepackage{stfloats} 
\usepackage{float}
\usepackage{placeins} 
\usepackage{xcolor}
\usepackage{changepage}
\usepackage{bbold}

\usepackage{algorithm}
\usepackage{algorithmicx}
\usepackage[noend]{algpseudocode}

\usepackage{tablefootnote} 
\usepackage[roman]{parnotes} 
\usepackage{tabulary}
\usepackage{booktabs}
\usepackage{tabularx}

\usepackage{tikz}
\usepackage{tikzsymbols}
\usetikzlibrary{shapes,decorations,arrows,calc,arrows.meta,fit,positioning}
\tikzset{
    -Latex,auto,node distance =1 cm and 1 cm,semithick,
    state/.style ={ellipse, draw, minimum width = 0.7 cm},
    point/.style = {circle, draw, inner sep=0.04cm,fill,node contents={}},
    bidirected/.style={Latex-Latex,dashed},
    el/.style = {inner sep=2pt, align=left, sloped}
}

\newtheorem{theorem}{Theorem}[section]
\newtheorem{lemma}[theorem]{Lemma}

\newtheorem{proposition}[theorem]{Proposition}

\newtheorem{assumption}[theorem]{Assumption}
\newtheorem{remark}[theorem]{Remark}
\newtheorem{example}[theorem]{Example}

\newcommand{\bx}{\mathbf{x}}
\newcommand{\by}{\mathbf{y}}
\newcommand{\bz}{\mathbf{z}}

\newcommand{\bM}{\mathbf{M}}

\newcommand{\bW}{\mathbf{W}}
\newcommand{\bX}{\mathbf{X}}

\newcommand{\balpha}{\boldsymbol{\alpha}}

\newcommand{\beps}{\boldsymbol{\epsilon}}

\newcommand{\bmu}{\boldsymbol{\mu}}

\newcommand{\bPhi}{\boldsymbol{\Phi}}

\DeclareMathOperator*{\argmax}{arg\,max}
\DeclareMathOperator*{\argmin}{arg\,min}

\DeclarePairedDelimiter{\braces}{\lbrace}{\rbrace}
\DeclarePairedDelimiter{\paren}{(}{)}

\DeclarePairedDelimiter{\abs}{|}{|}

\DeclarePairedDelimiter{\norm}{\|}{\|}

\newcommand{\R}{\mathbb{R}}
\renewcommand{\P}{\mathbb{P}}
\newcommand{\E}{\mathbb{E}}
\newcommand{\indicator}{\mathbb{1}}
\newcommand{\Var}{\operatorname{Var}}
\newcommand{\Cov}{\operatorname{Cov}}
\newcommand{\KL}{\operatorname{KL}}

\newcommand{\tree}{\mathfrak{T}}

\newcommand{\cell}{\mathcal{C}}

\newcommand{\leaf}{\mathcal{L}}

\newcommand{\depth}{l}
\newcommand{\idx}{m}

\newcommand{\leaves}{\mathfrak{L}}

\newcommand{\partition}{\mathcal{P}}

\newcommand{\partitionspace}[1][\sparsity]{\mathfrak{P}_{#1}}

\newcommand{\indep}{\perp\!\!\!\perp}

\renewcommand{\xspace}{\mathcal{X}}
\newcommand{\binspace}{\braces{\pm 1}^\nfeats}
\newcommand{\ctsspace}{[0,1]^\nfeats}
\newcommand{\xmeasure}{\nu}
\newcommand{\sparsity}{s}
\newcommand{\nfeats}{d}
\newcommand{\nsamples}{n}

\newcommand{\coordindices}{[\nfeats]}

\newcommand{\data}[1][n]{\mathcal{D}_{#1}}

\newcommand{\funcclass}{\mathfrak{F}}

\newcommand{\featindices}{J}
\newcommand{\cellfeatindices}{J}
\renewcommand{\path}{\mathcal{Q}}

\newcommand{\MSP}{\textsf{MSP}}
\newcommand{\SID}{\textsf{SID}}

\newcommand{\suppset}{S^*}

\newcommand{\splitoperator}{\mathcal{T}}
\newcommand{\cellsampindices}{I}

\newcommand{\wavelet}{\psi}
\newcommand{\cutset}{\mathcal{B}}

\newcommand{\event}{\Omega}
\newcommand{\fgreedy}{\hat{f}_{\operatorname{G}}}
\newcommand{\frf}{\hat{f}_{\operatorname{RF}}}
\newcommand{\fcart}{\hat{f}_{\operatorname{CART}}}
\newcommand{\ferm}{\hat{f}_{\operatorname{ERM}}}
\newcommand{\fna}{\hat{f}_{\operatorname{NA}}}
\newcommand{\rfparam}{\Theta}
\newcommand{\monomial}{\chi}
\newcommand{\modifieddata}{\check{\mathcal{D}}_\nsamples}
\newcommand{\TV}{\operatorname{TV}}
\newcommand{\traversal}{T}
\newcommand{\fcoefs}{\mathcal{S}}

\newcommand{\stabmsp}{\textsf{SMSP}}

\newcommand{\rid}{\Delta^{1/2}}
\newcommand{\ridhat}{\hat\Delta^{1/2}}
\newcommand{\cellcollection}{{\mathfrak{C}^*}}
\newcommand{\Corr}{\operatorname{Corr}}

\newcommand{\piecewise}{g}
\newcommand{\alg}{\mathcal{A}}
\newcommand{\splitfeature}{k}
\newcommand{\splitcriterion}{\mathcal{O}}

\newcommand{\mtryset}{\mathcal{M}}
\newcommand{\Imp}{\mathcal{I}}
\newcommand{\objective}{\mathfrak{R}}
\newcommand{\constraint}{\Gamma}
\newcommand{\sparsefuncs}{\funcclass_\sparsity}
\newcommand{\fnae}{\hat f_{\operatorname{NA-E}}}
\newcommand{\fgre}{\hat f_{\operatorname{G-E}}}

\newcommand{\querypath}{\mathcal{Q}}
\newcommand{\ybound}{M}
\newcommand{\graph}{\mathcal{G}}
\newcommand{\traversalsize}{\sparsity_{\operatorname{T}}}
\newcommand{\secondcellcollection}{\mathfrak{C}}

\newcommand{\comment}[2]{{#2}}

\title{Statistical-Computational Trade-offs for Recursive Adaptive Partitioning Estimators}

\author[1]{Yan Shuo Tan\thanks{yanshuo@nus.edu.sg}}
\author[2]{Jason M. Klusowski\thanks{jason.klusowski@princeton.edu}}
\author[3]{Krishnakumar Balasubramanian\thanks{kbala@ucdavis.edu}}
\affil[1]{\small Department of Statistics and Data Science, National University of Singapore}
\affil[2]{\small Department of Operations Research and Financial Engineering, Princeton University}
\affil[3]{\small Department of Statistics, University of California, Davis}

\date{}

\begin{document}

\maketitle

\begin{abstract}
    Models based on recursive adaptive partitioning such as decision trees and their ensembles are popular for high-dimensional regression as they can potentially avoid the curse of dimensionality. Because empirical risk minimization (ERM) is computationally infeasible, these models are typically trained using greedy algorithms. 
Although effective in many cases, these algorithms have been empirically observed to get stuck at local optima. 
We explore this phenomenon in the context of learning sparse regression functions over $\nfeats$ binary features, showing that when the true regression function $f^*$ does not satisfy \cite{abbe2022merged}'s Merged Staircase Property (MSP), \comment{red}{a form of heredity restriction similar to that used in classical ANOVA modeling}, greedy training requires $\exp(\Omega(\nfeats))$ samples to achieve low estimation error. 
Conversely, when $f^*$ does satisfy MSP, greedy training can attain small estimation error with only $O(\log \nfeats)$ samples. 
This dichotomy mirrors that of two-layer neural networks trained with stochastic gradient descent (SGD) in the mean-field regime, thereby establishing a head-to-head comparison between SGD-trained neural networks and greedy recursive partitioning estimators. 
Furthermore, ERM-trained recursive partitioning estimators achieve low estimation error with $O(\log \nfeats)$ samples irrespective of whether $f^*$ satisfies MSP, thereby demonstrating a statistical-computational trade-off for greedy training. Our proofs are based on a novel interpretation of greedy recursive partitioning using stochastic process theory and a coupling technique that may be of independent interest.
\end{abstract}

\section{Introduction}

Decision tree models are piecewise constant supervised learning models obtained by recursive adaptive partitioning of the covariate space.
Although classical, they remain among the most important supervised learning models because they are highly interpretable \citep{rudin2021interpretable,murdoch2019definitions} and yet are flexible enough to afford the potential for high prediction accuracy. 
This potential is maximized when decision trees are combined in ensembles via random forests (RFs) \citep{breiman2001random} or gradient boosting \citep{friedman2001greedy}. 
These algorithms are widely recognized as having state-of-the-art performance on moderately-sized tabular datasets \citep{caruana2008empirical,fernandez2014we,olson2018data}, even outperforming state-of-the-art deep learning methods \citep{grinsztajn2022tree}, despite the amount of attention lavished on the latter. 
Such datasets are common in many settings such as bioinformatics, healthcare, economics, and social sciences.
Naturally, decision trees and their ensembles receive widespread use via their implementation in popular machine learning packages such as \texttt{ranger} \citep{wright2015ranger}, \texttt{scikit-learn} \citep{pedregosa2011scikit}, \texttt{xgboost} \citep{chen2016xgboost}, and \texttt{lgbm} \citep{ke2017lightgbm}.
Decision trees have also been adapted to a variety of tasks beyond regression and classification, including survival analysis \citep{ishwaran2008random}, heterogeneous treatment effect estimation \citep{athey2016recursive}, time series analysis, and multi-task learning.

While alternatives exist, most decision tree models used in practice make binary, axis-aligned splits at each partitioning stage.
In this paper, we study the \emph{statistical-computational trade-offs} of these objects, focusing on regression trees.
We assume a nonparametric regression model under random design:
\begin{equation}
    \label{eq:nonparametric_regression_model}
    Y_i = f^*(\bX_i) + \varepsilon_i, \quad i =1,2,\ldots,\nsamples.
\end{equation}
Here, the covariate space $\xspace$ is a compact subset of $\R^\nfeats$, $\bX_1,\bX_2,\ldots\bX_\nsamples$ are drawn i.i.d. from a distribution $\xmeasure$ on $\xspace$, $\varepsilon_1,\varepsilon_2,\ldots,\varepsilon_\nsamples$ are noise variables drawn i.i.d. from a zero-mean 
noise distribution on $\R$, with $\varepsilon_i \indep\bX_i \eqqcolon (X_{i1},X_{i2},\ldots,X_{i\nfeats})$, and $f^*\colon\xspace \to \R$ is the conditional expectation function of the response $Y$ given $\bX = \bx$.
The observed data is denoted as $\data \coloneqq \braces*{(\bX_i,Y_i)\colon i=1,\ldots,\nsamples}$.
We will compare and contrast various tree-based estimators for $f^*$ , evaluating the accuracy of an estimate $\hat f(-;\data)$ using the $L^2$ risk (or estimation error)
\begin{equation} \nonumber
    R(\hat f(-;\data),f^*) \coloneqq \E_{\bX\sim\xmeasure}\braces*{\paren*{\hat f(\bX;\data) - f^*(\bX)}^2},
\end{equation}
and the accuracy of an estimator $\hat f(-;-)$ using the expected $L^2$ risk, $\E_{\data}\braces*{ R\paren*{\hat f(-;\data),f^*}}$.
Note that throughout this paper, we will use the convention that random variables are denoted in upper case while any fixed value they may attain is denoted in lower case.
Furthermore, vectors will be denoted using boldface, whereas scalars will be denoted using regular font.
We denote $\coordindices \coloneqq \braces*{1,2,\ldots,\nfeats}$. 
To supplement standard Big-$O$ notation, we use $\tilde{O}$ to suppress poly-logarithmic factors.


\subsection{High-Dimensional Consistency}



To establish statistical-computational trade-offs for regression trees, we analyze their performance in a high-dimensional setting ($n, d \to \infty$) with sparsity constraints on $f^*$, i.e., if $f^*$ only depends on a constant number of $\sparsity$ covariates. Informally, we say that an estimator avoids the curse of dimensionality, or that it is \emph{high-dimensional consistent}, if its expected $L^2$ risk converges to zero under a scaling regime satisfying $\log \nsamples = o(\nfeats)$.
It is easy to see that high-dimensional consistency requires feature selection:
Assuming relatively balanced splits, the average depth of a tree is $O(\log \nsamples)$.
Hence, $\log n = o(d)$ implies that not every covariate can receive a split.
As such, totally randomized trees, as well as other such algorithms that grow decision trees independently of the observed responses, are not high-dimensional consistent.

Several works 
have showed that CART and RFs are high-dimensional consistent given various condition on $f^*$ (e.g., \cite{syrgkanis2020estimation} for binary covariates and assuming submodularity, \cite{klusowski2024large} for arbitrary covariates and assuming additivity, \cite{chi2022asymptotic} for arbitrary covariates and assuming a condition called sufficient impurity decrease, which we will define more formally later in this paper). 
This helps to explain the successful application of RFs to high-dimensional problems such as genomics in which $\nsamples \ll \nfeats$.

On the other hand, these works were not able to show whether their assumed conditions were \emph{necessary} for high-dimensional consistency. 
\cite{syrgkanis2020estimation} used the exclusive-or (XOR) function to argue heuristically that some condition is indeed necessary (see Section 5 therein).
To understand this argument, we focus on the setting of binary covariates ($\xspace = \binspace$) under the uniform distribution.
In this setting, the XOR function can be written as $f^*(\bx) = x_1x_2$.
Meanwhile, every tree node corresponds to a subcube $\cell \subset \binspace$ and every potential split of $\cell$ divides it evenly into two subcubes of lower dimension.
With each split fully determined by the choice of covariate index $k=1,\ldots,\nfeats$, CART chooses the index of the covariate with the maximum plug-in squared correlation with the response within $\cell$:
\begin{equation}
    \label{eq:ID_intro}
    \widehat\Corr^2\braces*{Y,X_k~|~\bX \in \cell} = \frac{\paren*{\sum_{\bX_i \in \cell} X_{ik}(Y_i - \bar Y_{\cell})}^2}{\sum_{\bX_i \in \cell} (X_{ik} - (\bar{X}_k)_\cell)^2 \sum_{\bX_i \in \cell} (Y_i - \bar{Y}_\cell)^2};
\end{equation}
(see \cite{klusowski2021nonparametric} and \cite{klusowski2020sparse}, Equation (5)).
Unless $\cell$ already depends on either $x_1$ or $x_2$, however, we have zero correlation in the infinite-sample limit, i.e.,
\begin{equation}
\label{eq:marginal_signal}
    \Corr^2\braces*{f^*(\bX),X_k~|~\bX \in \cell} = \E\braces*{X_1X_2X_k~|~\bX \in \cell}^2 = 0,
\end{equation}
for $k=1,\ldots,\nfeats$.
This implies that CART cannot differentiate between relevant and irrelevant covariates, instead tending to make non-adaptive random splits.
Hence, it cannot avoid the curse of dimensionality.
As an ensemble of CART trees, RF suffers from the same limitations.

Beyond some heuristic calculations, \cite{syrgkanis2020estimation} did not make this argument rigorous, nor did they generalize beyond the XOR example.
\begin{quote}
\normalsize\emph{In this paper, we seek to bridge this gap and
establish a necessary and nearly sufficient condition on $f^*$ that characterizes when CART is high-dimensional consistent, and, in doing so, explicitly compare its risk performance against that of ERM and two-layer neural networks trained with SGD.}
\end{quote}
We do this in the context of binary covariates under the uniform distribution, which although unlikely to describe any real dataset, provides a useful setting where both results and calculations can be stated cleanly and precisely, thereby illustrating the essence of the strengths and limitations of using a greedy splitting rule.

\subsection{``Almost'' Characterization Using the Merged-Staircase Property}

For any subset $S \subset \coordindices$ of the coordinate indices, let $\chi_S$ be the monomial function defined by $\chi_S(\bx) = \prod_{j \in S} x_j$.\footnote{If $|S| = s$, such a function is sometimes called an $s$-parity in the theoretical computer science literature.}
It is a classical result from Fourier analysis \citep{stein2011fourier} that every function $f^* \colon \binspace \to \R$ can be uniquely written in the form
\begin{equation} \label{eq:fourier_representation}
    f^* = \sum_{i=1}^r \alpha_{S_i}\chi_{S_i},
\end{equation}
for some nonnegative integer $r$, subsets $S_1,S_2,\ldots,S_r$ of $\braces*{1,2,\ldots,\nfeats}$ and nonzero coefficients $\alpha_{S_j}, j=1,2,\ldots,r$.
We say that $f^*$ satisfies the \emph{Merged-Staircase Property} (\MSP) \citep{abbe2022merged} if the subsets $S_1,S_2,\ldots,S_r$ can be reordered such that for any $i \in \braces{1,2,\ldots,r}$, we have
\begin{equation} \label{eq:MSP}
    \abs*{S_i \backslash \cup_{j=1}^{i-1}S_j} \leq 1.\footnote{The merged-staircase property was first introduced by \cite{abbe2022merged}. See also Section \ref{subsec:comparisons_with_NNs}.}
\end{equation}
Note that any monomial of degree $k > 1$ does not satisfy the \textsf{MSP}.
In particular, the XOR function does not satisfy the \textsf{MSP}.
On the other hand, a positive example satisfying MSP is provided by the function
\begin{equation} \label{eq:MSP_example}
    f^*(\bx) = x_1 + x_2 + x_1x_2x_3.
\end{equation}

\comment{red}{For further intuition, note that \eqref{eq:fourier_representation} can be interpreted as an ANOVA decomposition for $f^*(\bX)$~\citep{wahba1990spline}. 
Each term $\chi_{S_i}$ is a contrast that captures either a main effect (if $S_i$ is a singleton) or a $k$-way interaction (if $S_i$ has cardinality $k > 1$). 
Meanwhile, the coefficient $\alpha_{S_i}$ represent the effect size.
The \textsf{MSP} imposes a type of heredity\footnote{\comment{red}{The \textsf{MSP} is related but not equivalent to the conditions of weak and strong heredity often enforced by ANOVA modeling in practice~\citep{wu2011experiments}. Weak heredity, which asserts that an interaction is present only if at least one of its lower-order parents is present, is equivalent to the Staircase Property formulated by \cite{abbe2021staircase} to analyze neural networks, and is therefore stronger than the \MSP. \cite{abbe2021staircase} later generalized the definition to the \MSP, which turns out to offer a better characterization of learnability.}} constraint on the pattern of interactions present in the ANOVA model: We can order the nonzero effects so that each newly-introduced effect brings in at most one new variable that has not appeared in any previously included effect.
In other words, interactions cannot appear ``out of nowhere'' combining many fresh variables at once.}

Our main theorem, stated informally, is the following:

\begin{theorem}[Informal]
\label{thm:informal_intro}
    When $\xspace = \binspace$, $\xmeasure = \operatorname{Unif}\paren*{\binspace}$, and $f^*$ depends only on $\sparsity$ covariates, then the expected $L^2$ risk of the CART estimator $\fcart$ satisfies:
    \begin{itemize}
        \item (Necessity) If $f^*$ does not satisfy \textsf{MSP}, then $\E_{\data}\braces*{ R\paren*{\fcart(-;\data),f^*}} = \Omega(1)$ whenever $\nsamples = \exp(O(\nfeats))$.
        \item (Near sufficiency) If $f^*$ satisfies \textsf{MSP} and its Fourier coefficients $\braces*{\alpha_{S_i}}_{i=1}^r$ are generic, then $\E_{\data}\braces*{ R\paren*{\fcart(-;\data),f^*}} = O(2^s/n)$ whenever $\nsamples = \Omega(2^s\log d)$.
    \end{itemize}
    Furthermore, regardless of whether $f^*$ satisfies \MSP, the expected $L^2$ risk of the ERM tree estimator satisfies $\E_{\data}\braces*{ R\paren*{\ferm(-;\data),f^*}} = O(2^s\log\nfeats/n)$.
\end{theorem}
The formal statements of these results are provided in the forthcoming Theorems \ref{thm:non_msp_lower_bound}, \ref{thm:MSP_upper_bound}, and \ref{thm:robust_lower_bounds}, together with Propositions \ref{prop:info_theory_lower_bound} and \ref{prop:erm_upper_bound}.
\comment{red}{Our upper bound is attained when using a minimum impurity decrease stopping condition (see~\eqref{eq:min_imp_dec} for details) with a fixed threshold $\gamma$ that is chosen above the noise level but below the signal strength. The lower bound does not depend on $\gamma$ or on the particular stopping/pruning choices.}

Here, genericity of $\braces*{\alpha_{S_i}}_{i=1}^r$ holds with probability one whenever they are drawn from a distribution with a density \comment{red}{(see the discussion in Section \ref{sec:characterizing_CART}).}
Up to this asterisk, we see that high-dimensional consistency for CART can be completely characterized by \textsf{MSP}, a combinatorial condition on Fourier coefficients.
In addition, we see in fact that CART behaves similarly to ERM trees when \MSP~holds, and similarly to non-adaptive trees when it does not, thereby also establishing \emph{performance gaps}:
In the former case, CART sharply improves upon non-adaptive trees, and in the latter case, it performs much worse compared to ERM.
Finally, we remark that our lower bounds when $f^*$ does not satisfy \MSP~also hold more broadly for RFs, as well as for trees and ensembles grown with other greedy recursive partitioning strategies.\footnote{It is possible to extend the upper bound to various forms of RFs, but we leave this to future work.}

\subsection{Marginal Signal Bottleneck}

When a regression function $f^*$ does not satisfy \MSP, its Fourier decomposition must contain a term $\alpha_{S_i}\monomial_{S_i}$ that violates \eqref{eq:MSP}.
This term creates problems for CART in a similar way as the XOR function.
Let us call a covariate $X_k$ \emph{marginally undetectable} in $\cell$ if \eqref{eq:marginal_signal} holds.
As argued before, when $X_k$ is relevant but marginally undetectable, CART cannot distinguish it from irrelevant covariates.
One can show that there exists covariate indices $k_1,k_2 \in S_i$ such that $X_{k_1}$ is marginally undetectable
unless $\cell$ has already has split on $x_{k_2}$ and $X_{k_2}$ is marginally undetectable
unless $\cell$ has already has split on $x_{k_1}$.
We call this chicken and egg problem \emph{marginal signal bottleneck}.
It results in a highly probable event in which both these covariates remain marginally undetectable at every iteration of CART, leading to high-dimensional inconsistency.

On the other hand, it is easy to show that \MSP~together with genericity implies \citet{chi2022asymptotic}'s sufficient impurity decrease (\SID) condition.
In the setting of uniform binary features, the \SID~condition asserts a lower bound for what we call the \emph{marginal signal},
\begin{equation}
    \label{eq:SID}
    \max_{1 \leq k \leq \nfeats}\Corr^2\braces*{f^*(\bX),X_k~|~\bX \in \cell} \geq \lambda,
\end{equation}
that holds uniformly over all subcubes $\cell$ on which $\Var\braces*{f^*(\bX)|\bX \in\cell} > 0$.
When \eqref{eq:SID} holds for \emph{some} $\cell$, then with enough samples, CART is able to identify and select a relevant covariate for splitting $\cell$.
When \eqref{eq:SID} hold for \emph{all} subcubes (i.e., \SID), we are able to guarantee that CART is uniformly able to identify and select relevant covariates across every iteration loop of the algorithm, thereby yielding high-dimensional consistency.


\subsection{Interpreting CART as a Stochastic Process}

Despite the simplicity of the lower bound argument outlined in the previous section, making it rigorous in finite samples, even in the special case of the XOR function, has proved elusive until now.
To illustrate the difficulty, consider a fixed query point $\bx \in \binspace$, and consider $\path$, the root-to-leaf path taken by $\bx$.
We have to show that there is a highly probable event on which $\path$ does not split on $x_1$ or $x_2$.
The usual approach to such a task is to attempt to prove concentration of the split criteria \eqref{eq:ID_intro} over all nodes (subcubes) along the path.
However, because of the data adaptivity of CART splits, we have little control over which subcubes appear, and we would be forced to derive concentration bounds that are uniform over all subcubes.
Unfortunately, this desired result is too strong and does not hold.

To resolve this, we interpret $\path$ as a stochastic process.
The measurement at ``time point'' $t$ comprises the values $\Corr^2\braces*{f^*(\bX),X_k~|~\bX \in \cell_t}$, $k=1\ldots,\nfeats$, where $\cell_t$ is the node in $\path$ at depth $t$.
We next couple this stochastic process to one that makes totally random splits, which allows us to use symmetry to calculate the probability that it does not split on $x_1$ or $x_2$.
On the event that the two processes are coupled, this also holds for the CART path $\path$.

Although we have outlined a lower bound for the XOR function, it can be readily adapted to any function not satisfying \MSP.
Furthermore, there is nothing particularly special about the CART criterion.
The proof relies on the fact that splits are determined by considering only the marginal distributions $X_k|~\cell,Y$ for $k=1,\ldots,\nfeats$ (as opposed to the joint distribution $\bX|~\cell,Y$) and hence extends to a larger class of greedy recursive partitioning algorithms, which we formally define in Section \ref{sec:framework}.
As far as we are aware, the proof technique is completely novel.


\subsection{Robust Lower Bounds under ``Soft'' Bottlenecks}

Thus far, we have seen that when $f^*$ does not satisfy \MSP, it experiences marginal signal bottleneck, which leads to high-dimensional inconsistency.
In some sense, however, being non-\MSP~is not a robust property because a small perturbation of the zero coefficients of such a function results in one that does satisfy \MSP.
As such, our argument regarding the weaknesses of greedy recursive partitioning strategies would be stronger if we could show that their performance degrades smoothly as the regression function $f^*$ approaches the class of non-\MSP~functions.
Fortunately, this is indeed the case, as reflected by the following informal theorem.

\begin{theorem}[Informal]
    When $\xspace = \binspace$, $\xmeasure = \operatorname{Unif}\paren*{\binspace}$, $f^*$ depends only on $\sparsity$ covariates, then the expected $L^2$ risk of the CART estimator $\fcart$ satisfies 
    \begin{align*}
    E_{\data}\braces*{ R\paren*{\hat f(-;\data),f^*}} = \Omega(1),~\text{whenever}~\nsamples = \min\braces*{O\paren*{1/\min_{f\notin\MSP}R(f, f^*)},\exp(O(\nfeats))}.
    \end{align*}
\end{theorem}


The value, $\min_{f\notin\MSP}R(f, f^*)$, can be thought of as measuring the width of a ``soft'' bottleneck.
For example, if we set $f^*(\bx) = \alpha x_1 + x_2 + x_1x_2x_3$, then we have $\min_{f\notin\MSP}R(f, f^*) = \alpha^2$.
We thus see that the sample complexity lower bound scales inversely with the bottleneck width, thereby extending the performance gap between greedy and ERM trees to a larger range of settings.
Finally, we remark that \MSP~can be re-interpreted in graph theoretic terms as form of connectedness.
Under this interpretation, $\min_{f\notin\MSP}R(f, f^*)$ is equal to the weight of the minimum vertex cut.
We elaborate upon this perspective in Section \ref{sec:greedy_and_non_msp}.

\subsection{Comparisons with Neural Networks Trained by SGD}
\label{subsec:comparisons_with_NNs}
\comment{red}{We now compare and connect our results with the literature on training neural networks. 
Readers who are only interested in decision trees and ensembles can safely skip this sub-section.}

In computational learning theory, the problem of ``learning'' (i.e., achieving small estimation error for) sparse Boolean monomials, also called \emph{parities}, from i.i.d. noiseless observations (i.e., $\varepsilon_i = 0$ in \eqref{eq:nonparametric_regression_model}) is known to be statistically easy but computationally hard \citep{barak2022hidden}, and hence has served as a useful benchmark for learning algorithms and computational frameworks.
Recently, this has become a benchmark for studying neural networks (NNs) trained using SGD, which, when combined with our results, allows us to perform, to our knowledge, the first theoretical head-to-head comparison between this class of algorithms and greedy regression trees and ensembles.


Indeed, the \textsf{MSP}~was introduced by \cite{abbe2022merged} to generalize Boolean monomials.
They showed that in the mean-field regime (very wide two-layer NNs with very small step size, see also \cite{mei2018mean}), $C\nfeats$ iterations of online SGD (one sample per iteration) are sufficient (for some $C$) to achieve small estimation error whenever $f^*$ satisfies \MSP~and is generic but is insufficient (for any $C$) when $f^*$ does not satisfy \MSP.
Intuitively, functions that do not satisfy \MSP~produce optimization landscapes with saddle points, which mean-field online SGD struggles to escape from.
Comparing with Theorem \ref{thm:informal_intro}, this creates an interesting analogy between mean-field online SGD and greedy regression trees and ensembles.

On the other hand, later works inadvertently emphasized the \emph{differences} between greedy regression trees and ensembles and NNs trained outside the mean-field regime and beyond the $O(\nfeats)$ iteration horizon.
In the classification setting, \cite{Glasgow2024sgd} showed that online minibatch SGD on a two-layer NN can learn the XOR function with $\tilde{\Theta}(\nfeats)$ samples and iterations.
\cite{kou2024matching} showed that for any $k$, online \emph{sign}-SGD with a batch size of $\tilde{O}(\nfeats^{k-1})$ can learn a $k$-parity (i.e., degree $k$ monomial) within $O(\log\nfeats)$ iterations, for a total sample complexity of $\tilde{O}(\nfeats^{k-1})$.
Since our lower bound in Theorem \ref{thm:informal_intro} holds even with noiseless observations, this establishes a rigorous performance gap between the two algorithm classes for some non-\MSP~functions.

Of further interest is \cite{abbe2023sgd}'s conjecture characterizing the sample complexity required to learn any given Boolean function in terms of its ``leap complexity''.
If this conjecture is true, NNs trained with SGD outperform greedy regression trees and ensembles on all non-\MSP~ functions but perform more poorly in comparison on \MSP~functions.
We discuss these comparisons more elaborately in Appendix \ref{sec:related_work}.

\section{Regression Trees}
\label{sec:framework}




In this section, we formally define regression tree models.
We do so in the context of binary covariates, which simplifies both the definitions of these objects and the notation required to describe them.
The reader is referred to \cite{friedman2001elements} on how to adapt these definitions to a more general context.

\subsection{Cells, Partitions, and Splits}
A \emph{cell} $\cell$ is a subcube of $\binspace$ and is defined as 
\begin{equation}
    \label{eq:cell_def}
    \cell = \braces*{\bx \in \binspace \colon x_j = z_j~\text{for}~j \in \featindices(\cell)}
\end{equation}
for a subset of covariate indices $\featindices(\cell) \subset [\nfeats]$ and a sequence of signs $z_j \in \braces*{\pm 1}$, $j \in \featindices(\cell)$.
Given a choice of covariate index $k \notin \featindices(\cell)$ and sign $z \in \braces*{\pm 1}$, we define the \emph{split operator} as producing the output $\splitoperator_{k,z}(\cell) \coloneqq \braces*{\bx \in \cell \colon x_k = z}$.
$\splitoperator_{k,1}(\cell)$ and $\splitoperator_{k,-1}(\cell)$ are called the \emph{right} and \emph{left} child of cell $\cell$.  
A \emph{partition} $\partition$ is a collection of disjoint cells whose union comprises $\binspace$.
The \emph{size} of $\partition$ is the number of cells it contains and is denoted $|\partition|$.
We say that another partition $\partition'$ is a \emph{refinement} of $\partition$, denoted $\partition \subset \partition'$, if for any $\cell' \in \partition'$, there exists $\cell \in \partition$ such that $\cell' \subset \cell$.
We say that $\partition'$ is a \emph{simple refinement} of $\partition$, denoted $\partition \subset_s \partition'$ if $\partition\backslash\partition' = \braces{\cell}$, $\partition'\backslash \partition = \lbrace \cell_L',\cell_R'\rbrace$ and $\cell = \cell_L' \cup \cell_R'$.
In other words, $\partition'$ is obtained from $\partition$ by applying the split operator to $\cell$.
We use $\cellsampindices(\cell)$ to denote the indices of observations whose covariate vectors lie in $\cell$, i.e., $\cellsampindices(\cell) \coloneqq \braces*{i \in [\nsamples] \colon \bx_i \in \cell}$, and let $N(\cell) \coloneqq \abs*{\cellsampindices(\cell)}$.
Furthermore, we denote the mean response in $\cell$ as $\bar Y_{\cell} \coloneqq \frac{1}{N(\cell)} \sum_{i \in \cellsampindices(\cell)} Y_i$.

\subsection{Regression Tree Models}

Note that every piecewise constant function $\piecewise$ on $\binspace$ can be parameterized using a partition $\partition$ with a fixed ordering of its constituent cells $\cell_1,\cell_2,\ldots,\cell_{|\partition|}$, and a set of labels $\bmu = (\mu_1,\mu_2,\ldots,\mu_{|\partition|})$ representing the values taken by $\piecewise$ on each cell.
In other words, we have
\begin{equation} 
    g(\bx;\partition,\bmu) = \sum_{j=1}^{|\partition|} \mu_j \indicator\braces*{\bx \in \cell_j}.
\end{equation}
A binary, axis-aligned \emph{regression tree model} is a special type of piecewise constant function, one whose partition $\partition$ arises from recursive applications of split operators $\splitoperator_{\cdot,\cdot}$ to $\binspace$ and the resulting cells.
More precisely, there exists a sequence of simple refinements, 
$\lbrace \binspace \rbrace = \partition_1 \subset_s \partition_2 \subset_s \cdots \subset_s \partition_{|\partition|} = \partition$.
This recursive structure can be represented as a labeled binary tree $\tree$ in which each node corresponds to a cell obtained in the recursion and parent-child relationships correspond to applications of $\splitoperator_{\cdot,\cdot}$.
The \emph{root} of $\tree$ is the entire covariate space $\binspace$.
As is common in the literature, we will abuse terminology and sometimes use the term ``node'' to refer to its corresponding cell.
Each internal node is labeled with the covariate used to split it, called its \emph{split covariate} and denoted as $\splitfeature(\cell)$.
The non-internal nodes $\leaf_1,\leaf_2,\ldots,\leaf_{|\partition|}$ are called \emph{leaves}, have empty labels, and comprise the final partition $\partition = \partition(\tree)$.
For convenience, we also slightly abuse notation, writing $\abs*{\tree} = \abs*{\partition(\tree)}$ and $g(-;\tree,\bmu) = g(-;\partition(\tree),\bmu)$.
We call $\tree$ the \emph{tree structure} of the model and $\bmu$ its \emph{leaf labels}.
The \emph{depth} of $\tree$ is the maximum length of a \emph{root-to-leaf path}, that is, a sequence of nodes $\cell_0 \supset \cell_1 \supset \cdots \supset \cell_l$, where $\cell_j$ is a parent of $\cell_{j+1}$ in $\tree$ for $j=0,1,\ldots,{l-1}$, $\cell_0$ is the root and $\cell_l$ is a leaf node.



\subsection{ERM Tree Models}
Given the observations $\data$, the $L^2$ empirical risk of a regression function $f\colon \binspace \to \R$ is defined as
\begin{equation} \label{eq:empirical_risk}
    \hat R(f;\data) \coloneqq \frac{1}{\nsamples}\sum_{i=1}^\nsamples \paren*{Y_i - f(\bX_i)}^2.
\end{equation}
The empirical risk minimizer or \emph{ERM tree model} is the regression tree model minimizing this functional subject to a constraint set $\constraint$ on the tree structure and upper bound $M$ on the possible leaf label values:
\begin{equation} \label{eq:ERM_equation}
    \hat f_{\text{ERM}}(-;\data) \coloneqq \argmin_{\tree \in \constraint,\;\norm{\bmu}_\infty \leq M} \hat R\paren*{g(-;\tree,\bmu);\data}.
\end{equation}
Note that for any fixed tree, the leaf labels that minimize the empirical risk are exactly the mean responses $\bar Y_{\leaf_j}$, so long as these values do not exceed $M$ in absolute value.

\subsection{Regression Tree Algorithms}

For computational reasons, regression tree models are usually defined as the output of an algorithm rather than as the solution to an ERM problem \eqref{eq:ERM_equation} or any other global optimization program.
We call such an algorithm a \emph{regression tree algorithm}, treating it as a function $\alg$ that takes in the data $\data$ together with a random seed $\Theta$ and returns a regression tree model $g(-;\tree,\bmu) = \alg(\data,\Theta)$.
Since $g$ is fully determined by its tree structure and leaf labels, we can write the algorithm as a tuple $\alg = (\alg_t,\alg_l)$, where $\tree = \alg_t(\data,\Theta)$ and $\bmu = \alg_l(\data,\Theta,\alg_t(\data,\Theta))$.
We call $\alg_t$ and $\alg_l$ the tree structure component and leaf labels component of the algorithm respectively.
In most algorithms used in practice, $\alg_l$ sets the label of each leaf to be the mean response within that leaf.
Notwithstanding recent work commenting about the possible limitations of this approach \citep{tan2021cautionary,agarwal2022hierarchical}, we hence assume this to be the case unless stated otherwise, thereby putting the focus squarely on the tree structure component $\alg_t$.



\subsection{CART}

\cite{breiman1984classification}'s \emph{CART} is a regression tree algorithm that grows a tree greedily in a top-down iterative fashion as follows:
Starting with the entire space $\binspace$, each new cell $\cell$ obtained in the algorithm is first checked for a stopping condition.
If the condition is met, $\cell$ is not split and is fixed as a leaf in the final tree.
Otherwise, its split covariate is selected to be the maximizer of the \emph{impurity decrease}, which is defined as
\begin{equation} \label{eq:impurity_decrease}
    \hat\Delta(k; \cell, \data) \coloneqq \Imp(\cell;\data) 
    - \frac{N(\splitoperator_{k,1}(\cell))}{N(\cell)}\Imp\paren*{\splitoperator_{k,1}(\cell);\data}
    - \frac{N(\splitoperator_{k,-1}(\cell))}{N(\cell)}\Imp\paren*{\splitoperator_{k,-1}(\cell);\data},
\end{equation}
where $\Imp(\cell;\data) \coloneqq \frac{1}{N(\cell)}\sum_{i \in \cellsampindices(\cell)} \paren*{Y_i - \bar Y_{\cell}}^2$ is called the \emph{impurity} of the cell $\cell$.

The term ``impurity'' here refers to Gini impurity, which is related to the probability of misclassification in classification trees.
In a regression tree setting, we see that it is exactly equal to the empirical variance of the response within the cell.
The second and third terms on the right hand side in \eqref{eq:impurity_decrease} sum to the residual variance in $\cell$ after regressing out $X_k$, which means that the impurity decrease can be written as 
\begin{equation}
    \label{eq:impurity_decrease_equivalence}
    \hat\Delta(k; \cell, \data) = \Imp(\cell;\data)\cdot \widehat\Corr^2\braces*{Y,X_k~|~\bX \in \cell}    
\end{equation}
as alluded to earlier.


When all leaves satisfy the stopping condition, the current tree $\tree$ is sometimes pruned to a subtree $\tree' \subset \tree$.
$\tree'$ is then returned as the output to the algorithm.
Several different stopping conditions and pruning rules are available \citep{breiman1984classification,pedregosa2011scikit}.
Our lower bounds will not depend on these choices so we will not go into further detail on what these are.
Finally, for our upper bounds, we will make use of the \emph{minimum impurity decrease} stopping condition, which stops splitting $\cell$ when it satisfies
\begin{equation}
\label{eq:min_imp_dec}
    \frac{N(\cell)}{\nsamples}\max_{k \in \coordindices\backslash\cellfeatindices(\cell)}\hat\Delta(k; \cell, \data) < \gamma
\end{equation}
for some prespecified threshold $\gamma$.
We call an output of the algorithm a $\emph{CART model}$, and denote it as $\fcart(-;\data)$.

\subsection{Random Forests}

A \emph{random forest} model is an ensemble of randomized CART trees.
Specifically, for each CART tree, we fit it to a bootstrapped version of the dataset $\data$, denoted $\data^*$ and furthermore, at each iteration, a random subset $\mtryset \subset \coordindices$ of some prespecified size is drawn uniformly at random with the split covariate is selected from among these covariates, i.e., we set \begin{equation}
    \splitfeature(\cell) \coloneqq \argmax_{k \in \mtryset\backslash\cellfeatindices(\cell)} \hat\Delta(k;\cell,\data^*).
\end{equation}
Slightly abusing notation, we use
$\fcart(-;\data,\theta)$ to denote the CART model fitted with a random seed $\theta$.
A random forest model with $M$ trees is then
\begin{equation}
    \frf(-;\data,\Theta) \coloneqq \frac{1}{M} \sum_{m=1}^M \fcart(-;\data,\theta_m),
\end{equation}
where $\Theta = (\theta_1,\theta_2,\ldots,\theta_M)$ and $\theta_1,\theta_2,\ldots,\theta_M$ are drawn i.i.d..

\subsection{Greedy Tree Models}

In order to isolate the aspect of CART that leads to its limitations with non-\MSP~functions, we make a more general definition.
We say that $\alg$ is a \emph{greedy tree algorithm} if, similar to CART, it grows a tree greedily in a top-down fashion, but uses a possibly different criterion $\splitcriterion(k;\cell,\data)$ instead of impurity decrease \eqref{eq:impurity_decrease} to determine its splits.
We constrain $\splitcriterion(k;\cell,\data)$ to depend on $\data$ only via the values of the $k$-th covariate and the response for observations in $\cell$, i.e.,
\begin{equation}
    \label{eq:greedy_criterion}
    \splitcriterion(k;\cell,\data) = \mathcal{F}\paren*{\braces*{(X_{ik},Y_i)}_{i \in \cellsampindices(\cell)}}
\end{equation}
for a fixed function $\mathcal{F}$ that is invariant under the map $\braces*{(X_{ik},Y_i)}_{i \in \cellsampindices(\cell)} \mapsto \braces*{(-X_{ik},Y_i)}_{i \in \cellsampindices(\cell)}$.\footnote{This condition is used in the proof of Lemma \ref{lem:conditional_expectation_for_node_no_of_samples}.}
Since
\begin{equation}
    \Imp(\splitoperator_{k,z}(\cell);\data) = \sum_{i \in \cellsampindices(\cell)} \left( \frac{\abs{X_{ik} + z}}{2}Y_i - \frac{\sum_{i \in \cellsampindices(\cell) }\abs{X_{ik}+z}Y_i}{\sum_{i \in \cellsampindices(\cell)} \abs{X_{ik} + z}}\right)^2,
\end{equation}
for $z=-1,1$, we see that CART is indeed an instance of a greedy tree algorithm.
We call an output of $\alg$ a \emph{greedy tree model} and denote it as $\fgreedy$.
Furthermore, if $\fgreedy$ is the base learner in an ensemble constructed analogously to a random forest, we call the resulting ensemble a \emph{greedy tree ensemble} and denote it as $\fgre$.

\section{Covariate Selection and High-Dimensional Inconsistency}

\subsection{High-Dimensional Consistency}

We will compare and contrast the performance of ERM trees and greedy trees in a sparse, high-dimensional setting.
To formalize this setting, we make the following assumption for the rest of this paper.


\begin{assumption}
    \label{assumption}
    Assume a binary covariate space $\xspace = \binspace$ with $\xmeasure$ the uniform measure.
    Let $f^*_0\colon \lbrace\pm 1\rbrace^\sparsity \to \R$ be any function.
    Let $\suppset \subset \coordindices$ be a subset of size $\abs*{\suppset} = \sparsity$.
    We consider the nonparametric model \eqref{eq:nonparametric_regression_model} with the regression function $f^*\colon \binspace \to \R$ defined via $f^*(\bx) = f^*_0(\bx^{\suppset})$.
\end{assumption}


We evaluate the performance of an estimator $\hat f$ for $f^*$ using the expected $L^2$ risk, where the expectation is taken over the randomness with respect to $\data$ as well as potential algorithmic randomness $\Theta$:
\begin{equation}
    \objective\paren*{\hat f, f^*_0,\nfeats,\nsamples} \coloneqq \E_{\data,\Theta}\braces*{R\paren*{\hat f(-;\data,\Theta),f^*}}.
\end{equation}
We say that $\hat f$ has \emph{high-dimensional consistency} with respect to $f_0^*$ if there is an infinite sequence $\nfeats_1,\nfeats_2,\ldots$ such that $\lim_{\nsamples \to \infty} \objective\paren*{\hat f, f^*_0,\nfeats_{\nsamples},\nsamples} = 0$ and $\lim_{\nsamples \to \infty}\log \nsamples / \nfeats_\nsamples = 0$.
In other words, estimators that are not high-dimensional consistent require, in order to estimate $f^*$, a number of samples that grows exponentially in the ambient dimension.
Note that this definition is intentionally weak, so as to emphasize the weakness of greedy trees under non-\MSP~conditions.

\subsection{Reduction to Covariate Selection Along Query Paths}
\label{subsec:covariate_selection}

As argued heuristically in the introduction, the ability to perform some sort of covariate selection is necessary for high-dimensional consistency.
We begin to make this intuition rigorous by introducing the notion of a \emph{query path}. 
Given a choice of (i) a query point $\bx \in \binspace$, (ii) a regression tree structure algorithm $\alg_t$, (iii) a dataset $\data$, and (iv) a random seed $\Theta$ for $\alg_t$, there is a unique leaf $\leaf$ in the tree structure $\alg_t(\data,\Theta)$ that contains $\bx$.
We call the path in $\alg_t(\data,\Theta)$ from the root to $\leaf$ the query path of $\bx$, and denote it using $\querypath(\bx;\data,\Theta)$.
Of special interest to us will be the covariates split on along the query path, that is
\begin{equation}
    J(\bx;\data,\Theta) \coloneqq \braces*{ k(\cell) \colon \cell \in \querypath(\bx;\data,\Theta)}.
\end{equation}
Note that to simplify notation, we have elided the dependence of $\alg_t$, whose choice will be clear from the context.
The role of the query path in expected risk lower bounds is established in the following two lemmas.

\begin{lemma}
\label{lem:covariate_selection_tree}
    Let $\hat f$ be a regression tree model fit using a regression tree algorithm $\alg$.
    Then under Assumption \ref{assumption}, its expected $L^2$ risk satisfies $\objective\paren*{\hat f, f^*_0,\nfeats,\nsamples} \geq (1-\delta)\Var\braces*{f^*_0(\bX)}$, where $\delta \coloneqq \P\braces*{J(\bX;\data,\Theta)\cap \suppset \neq \emptyset }$.
\end{lemma}

\begin{proof}
    We start with the trivial lower bound
    \begin{equation} \label{eq:query_path_helper1}
        R\paren*{\hat f(-;\data,\Theta),f^*} \geq \E_{\bX\sim\xmeasure}\braces*{\paren*{\hat f(\bX;\data,\Theta) - f^*(\bX)}^2\indicator\braces*{J(\bX;\data,\Theta)\cap \suppset = \emptyset }}.
    \end{equation}
    The collection $\braces*{J(\bx;\data,\Theta)\cap \suppset = \emptyset }_{\bx \in \binspace}$ can be thought of as a collection of random events indexed by $\bx \in \binspace$.
    Whenever one of these events holds for some fixed value $\bx$, $\hat f(\bx;\data,\Theta)$ is constant with respect to changes in $\bx_{\suppset}$, while by definition $f^*(\bx)$ is constant with respect to changes in $\bx_{\coordindices\backslash\suppset}$.
    Utilizing these facts together with the uniform distribution of $\bX$ allows us to decompose the conditional expectation of the squared loss as
    \begin{equation} \label{eq:query_path_helper2}
    \begin{split}
        & \E_{\bX}\braces*{\paren*{\hat f(\bX;\data,\Theta) - f^*(\bX)}^2~
            \vline~\bX_{\coordindices\backslash\suppset} = \bx_{\coordindices\backslash\suppset}} \\
            &= \Var\braces*{f_0^*(\bX)} + \paren*{\hat f(\bx;\data,\Theta) - \E\braces*{f_0^*(\bX)}}^2 \\
            &\geq \Var\braces*{f_0^*(\bX)}.
    \end{split}
    \end{equation}
    Making the further observation that the function $\indicator\braces*{J(\bx;\data,\Theta)\cap \suppset = \emptyset }$ is constant with respect to $\bx_{\suppset}$, 
    we can use properties of conditional expectations to rewrite the right hand side of \eqref{eq:query_path_helper1} as
    \begin{equation}
        \E_{\bX\sim\xmeasure}\braces*{\E_{\bX}\braces*{\paren*{\hat f(\bX;\data,\Theta) - f^*(\bX)}^2~
            \vline~\bX_{\coordindices\backslash\suppset}}\indicator\braces*{J(\bX;\data,\Theta)\cap \suppset = \emptyset }}.
    \end{equation}
    Plugging in the bound \eqref{eq:query_path_helper2}, taking a further expectation with respect to $\data$ and $\Theta$, and simplifying the resulting expression completes the proof.
\end{proof}

\begin{lemma}
\label{lem:covariate_selection_ensemble}
    Let $\hat f$ be a regression tree ensemble model in which each tree is fit using a regression tree algorithm $\alg$.
    In addition to Assumption \ref{assumption}, assume that the response variable is bounded almost surely, i.e., $|Y| \leq \ybound$. 
    Then its expected $L^2$ risk satisfies $\objective\paren*{\hat f, f^*_0,\nfeats,\nsamples} \geq (1-\kappa\delta)\Var\braces*{f^*_0(\bX)}$, where $\delta \coloneqq \max_{\bx}\P\braces*{J(\bx;\data,\Theta)\cap \suppset \neq \emptyset }$ and $\kappa \coloneqq 2\ybound/\Var\braces*{f^*_0(\bX)}^{1/2}$.
\end{lemma}

\begin{proof}
    We first write
    \begin{equation}
        \hat f(\bx;\data,\Theta) = \frac{1}{B}\sum_{b=1}^B \tilde f(\bx;\data,\theta_b),
    \end{equation}
    where $\tilde f$ is the base learner of the ensemble.
    Taking expectations with respect to $\data$ and $\Theta$ gives the chain of equalities:
    \begin{equation}
        \begin{split}
            \E_{\data,\Theta}\braces*{\frac{1}{B}\sum_{b=1}^B \tilde f(\bx;\data,\theta_b)}
            & = \E_{\data,\theta_1}\braces*{\tilde f(\bx;\data,\theta_1)} \\
            & = \underbrace{\E_{\data,\theta_1}\braces*{\tilde f(\bx;\data,\theta_1)\indicator\braces*{J(\bx;\data,\theta_1)\cap \suppset = \emptyset }}}_{g(\bx)} \\
            & \quad + \underbrace{\E_{\data,\theta_1}\braces*{\tilde f(\bx;\data,\theta_1)\indicator\braces*{J(\bx;\data,\theta_1)\cap \suppset \neq \emptyset }}}_{h(\bx)}.
        \end{split}
    \end{equation}
    For ease of notation, we label the two quantities on the right hand side as $g(\bx)$ and $h(\bx)$ respectively.
    Note also that $\abs*{h(\bx)} \leq M\delta$.
    Using Jensen's inequality followed by the tower property of conditional expectations, we may then compute 
    \begin{equation}
    \label{eq:query_path_helper3}
    \begin{split}
        & \E_{\bX,\Theta,\data}\braces*{\paren*{\hat f(\bX;\data,\Theta) - f^*(\bX)}^2} \\
        & \geq \E_\bX\braces*{\paren*{g(\bX) + h(\bX) - f^*(\bX)}^2} \\
        & = \E_\bX\braces*{\E_{\bX}\braces*{\paren*{g(\bX) + h(\bX) - f^*(\bX)}^2~\vline~ \bX_{\coordindices\backslash\suppset}}}.
    \end{split}
    \end{equation}
    We next notice that $g(\bx)$ is constant with respect to changes in $\bx_{\suppset}$.
    This means that, similar to \eqref{eq:query_path_helper2}, the conditional expectation within \eqref{eq:query_path_helper3} can be lower bounded by a variance term, which we expand as follows:
    \begin{equation} \label{eq:query_path_helper4}
        \begin{split}
            & \E_{\bX}\braces*{\paren*{g(\bX) + h(\bX) - f^*(\bX)}^2~\vline~ \bX_{\coordindices\backslash\suppset}=\bx_{\coordindices\backslash\suppset}} \\
            & \geq \Var_{\bX}\braces*{h(\bX) - f^*(\bX)~|~\bX_{\coordindices\backslash\suppset}=\bx_{\coordindices\backslash\suppset}} \\
            & \geq \Var\braces*{f_0^*(\bX)} - 2\Var\braces*{f_0^*(\bX)}^{1/2}\Var_{\bX}\braces*{h(\bX)~|~\bX_{\coordindices\backslash\suppset}=\bx_{\coordindices\backslash\suppset}}^{1/2}.
        \end{split}
    \end{equation}
    We now plug this back into \eqref{eq:query_path_helper3} 
    to get
    \begin{equation}
        \objective\paren*{\hat f, f^*_0,\nfeats,\nsamples} \geq (1-2\ybound\delta/\Var\braces*{f^*_0(\bX)}^{1/2})\Var\braces*{f^*_0(\bX)}. \qedhere
    \end{equation}    
\end{proof}

\section{Greedy Trees and the XOR Function}

Before presenting our main results in full generality,
we first state and discuss how to obtain an expected risk lower bound for the XOR function $f_0^*(x_1,x_2) = x_1x_2$ in the sparse high-dimensional setting (Assumption \ref{assumption}).
As mentioned in the introduction, the XOR function does not contain any marginal signal and hence presents perhaps the simplest setting in which CART has been observed to perform much worse than ERM trees.
The example has indeed been discussed before in the literature, but no formal proof regarding this phenomenon was previously available.

\begin{proposition}
\label{prop:xor_lower_bound}
Under Assumption \ref{assumption}, let $f_0^*$ be the XOR function.
    For any $0 < \delta < 1$,
    suppose $\nsamples \leq 2^{\delta(\nfeats-1)/2 - 2}$.
    Then the expected $L^2$ risk of any greedy tree model satisfies
    \begin{equation}
        \objective\paren*{\fgreedy, f^*_0,\nfeats,\nsamples} \geq (1-\delta)\Var\braces*{f_0^*(\bX)}.
    \end{equation}
    If furthermore, $|Y| \leq \ybound$ almost surely, then the $L^2$ risk for any greedy tree ensemble model satisfies
    \begin{equation}
        \objective\paren*{\fgre, f^*_0,\nfeats,\nsamples} \geq (1-\kappa\delta)\Var\braces*{f_0^*(\bX)},
    \end{equation}
    where $\kappa \coloneqq 2\ybound/\Var\braces*{f_0^*(\bX)}^{1/2}$.
    In particular, $\fgreedy$ and $\fgre$ are high-dimensional inconsistent for $f^*_0$.
\end{proposition}

To prove this statement, given Lemma \ref{lem:covariate_selection_tree} and Lemma \ref{lem:covariate_selection_ensemble}, we need only provide an appropriate upper bound for the covariate selection probability along a query path.


\begin{lemma}
    \label{lem:xor_feature_selection_prob}
    Under Assumption \ref{assumption}, let $f_0^*$ be the XOR function and suppose $\alg_t$ is greedy.
    Then for each fixed query point $\bx \in \binspace$ and random seed $\Theta$, the probability of covariate selection is upper bounded as $\P_{\data}\braces*{J(\bx;\data,\Theta)\cap \suppset \neq \emptyset } \leq \frac{2(\log_2 \nsamples + 2)}{\nfeats - 1}$.
\end{lemma}

\begin{proof}
    Without loss of generality, assume that $\suppset = \braces*{1,2}$.
    The proof idea is to view the query path $\querypath(\bx;\data,\Theta)$ as a stochastic process indexed by node depth.
    In other words, writing the path as $\querypath(\bx;\data,\Theta)$ and the sequence $\cell_0 \supset \cell_1 \supset \cdots \supset \cell_l$, the information at time $t$, for $t=0,1,\ldots $ comprises the node $\cell_t$ as well as its auxiliary information: the split criterion values $\mathcal{O}(k;\cell_t,\data)$ for $k \notin \cellfeatindices(\cell_t)$ and split covariate $k_t(\bx;\data,\Theta) \coloneqq k(\cell_t)$.
    We will couple this process to two other query path processes, both of which can be shown to select covariates completely at random.
    On the high probability event that the coupling holds, we are thus able to bound the covariate selection probability of the original query path.
    
    \textit{Step 1: Defining the path coupling.}
    Define the modified dataset $\data^{(1)} \coloneqq \braces*{(\bX_{i,-1},Y_i)}_{i=1}^\nsamples$, i.e., so that the first covariate is dropped from each observation.
    Define $\data^{(2)}$ analogously, but by dropping $X_2$ instead of $X_1$.
    We will analyze the query paths, $\querypath(\bx;\data^{(1)},\Theta)$ and $\querypath(\bx;\data^{(2)},\Theta)$, obtained from these modified datasets and compare them with the original.

    \textit{Step 2: Coupled paths select covariates randomly.}
    For $a=1,2$, observe that for any fixed value of $\bx_{-a}$, we have
    \begin{equation}
        \P\braces*{X_1X_2 = 1 ~\vline~ \bX_{-a} = \bx_{-a}} = \frac{1}{2}.
    \end{equation}
    This implies that the response $Y$ is independent of the covariates in the modified dataset $\data^{(a)}$.
    Combining this with the fact that the covariates are drawn uniformly from $\braces{\pm 1}^{\nfeats-1}$ as well as the symmetry of the splitting criterion, we see that the distribution of $\data^{(a)}$ is invariant to permutation of the covariate indices.
    As such, $J(\bx;\data^{(a)},\Theta)$ is the prefix of a uniformly random permutation of $\coordindices\backslash\braces{a}$.

    \textit{Step 3: Bounding covariate selection in coupled paths.}
    For any $b \in \coordindices\backslash\braces{a}$, we use the previous step to compute
    \begin{equation}
        P\braces*{b \in J(\bx;\data^{(a)},\Theta)~\vline~\abs{J(\bx;\data^{(a)},\Theta)}} \leq \frac{\abs{J(\bx;\data^{(a)},\Theta)}}{\nfeats - 1}.
    \end{equation}
    Taking a further expectation and using Lemma \ref{lem:bound_for_tree_depth} gives
    \begin{equation}
    \begin{split}
        \P\braces*{b \in J(\bx;\data^{(a)},\Theta)} & = \E\braces*{\P\braces*{b \in J(\bx;\data^{(a)},\Theta)~\vline~\abs{J(\bx;\data^{(a)},\Theta)}}} \\
        & \leq \frac{\log_2 \nsamples + 2}{\nfeats - 1}.
    \end{split}
    \end{equation}

    \textit{Step 4: Decoupling implies covariate selection.}
    We claim the following inclusion of events:
    \begin{equation}
    \label{eq:xor_decoupling_covariate_selection}
        \bigcup_{a=1}^2\braces*{J(\bx;\data,\Theta) \neq J(\bx;\data^{(a)},\Theta)} \subset \braces*{1 \in J(\bx;\data^{(2)},\Theta)}\cup \braces*{2 \in J(\bx;\data^{(1)},\Theta)}.
    \end{equation}
    To prove this, assuming the left hand side, let $t$ be the smallest index for which $k_t(\bx;\data^{(a)},\Theta) \neq k_t(\bx;\data,\Theta)$ for some $a$, which we assume without loss of generality is $a=1$.
    This implies in particular that the depth $t$ node along all three query paths are equal.
    Label this node as $\cell_t$.
    We now examine the possible values for $k_t(\bx;\data,\Theta)$.
    For $j \neq 1$, we have
    \begin{equation}
        \mathcal{O}(j,\cell_t,\data) = \mathcal{F}\paren*{\braces*{(X_{ij},Y_i)}_{i \in I(\cell)}} = \mathcal{O}(j,\cell_t,\data^{(1)}).
    \end{equation}
    Recall that $k_t(\bx;\data,\Theta)$ is defined to be the maximum of the left hand side over $j \notin \cellfeatindices(\cell_t)$.
    Meanwhile, $k_t(\bx;\data^{(1)},\Theta)$ is defined to be the maximum of the right hand side over $j \notin \cellfeatindices(\cell_t)\cup\braces{1}$.
    As such, we see that the only way in which these can be different values is if $k_t(\bx;\data,\Theta) = 1$.
    This would imply, however, that $k_t(\bx;\data^{(2)},\Theta) = 1$ as we wanted.

    \textit{Step 5: Completing the proof.}
    Suppose $J(\bx;\data,\Theta)\cap \suppset \neq \emptyset$.
    If coupling does not hold (i.e., the left hand side of \eqref{eq:xor_decoupling_covariate_selection} is true), then we see that either $1 \in J(\bx;\data^{(2)},\Theta)$ or $2 \in J(\bx;\data^{(2)},\Theta)$.
    On the other hand if coupling does hold, then we again have either $1 \in J(\bx;\data,\Theta) = J(\bx;\data^{(2)},\Theta)$ or $2 \in J(\bx;\data,\Theta) = J(\bx;\data^{(1)},\Theta)$.
    In conclusion, we have
    \begin{equation}
    \begin{split}
        \P_{\data}\braces*{J(\bx;\data,\Theta)\cap \suppset \neq \emptyset } & \leq \P\braces*{1 \in J(\bx;\data^{(2)},\Theta)} + \P\braces*{2 \in J(\bx;\data^{(1)},\Theta)} \\
        & \leq \frac{2(\log_2 \nsamples + 2)}{\nfeats - 1}.\qedhere
        \end{split}
    \end{equation}
\end{proof}

\begin{remark}
    The technique of controlling a stochastic process by coupling it to another one which allows for more convenient calculations was also used by \cite{tan2019online} to analyze the convergence of online SGD for the problem of phase retrieval, and by \cite{abbe2023polynomial, abbe2022on} to analyze gradient descent on leap functions.
\end{remark}

\section{Greedy Trees and Non-MSP Functions}
\label{sec:greedy_and_non_msp}

The XOR function is an example of a function not satisfying \cite{abbe2022merged}'s merged-staircase property (\MSP).
As promised, the lower bound for the XOR function can be extend to a general lower bound for this class of functions.
To motivate why this is the case as well as the type of bounds we may expect, we first further investigate what it means for $f^*$, written as in \eqref{eq:fourier_representation}, to not satisfy \MSP.

Construct a directed hypergraph whose vertex set is $\fcoefs \coloneqq \braces*{S_1,S_2,\ldots,S_r}\cup{\emptyset}$, in other words, the subsets of $\coordindices$ corresponding to nonzero Fourier coefficients of $f^*$, together with the empty set (if not already present).
Recall that an edge in a hypergraph is a relationship between a pair of subsets of vertices.
In $\graph$, we put an edge $\braces*{S_{j_1},S_{j_2},\ldots,S_{j_{k-1}}} \to S_{j_{k}}$ if
\begin{equation}
    \abs*{S_{j_k}\backslash \cup_{i=1}^{k-1}S_{j_i}} \leq 1.
\end{equation}
Finally, put a weight function $w\colon \fcoefs \to \R$ defined by $w(S) \coloneqq \alpha_S^2$.
We refer to the weighted graph $\mathcal{G}_{f^*}$ as the \emph{Fourier graph} of $f^*$. 

It is clear that \MSP~is equivalent to $\graph_{f^*}$ being a connected graph.
Now let $\fcoefs_{\MSP}$ denote the connected component of the empty set $\emptyset$ in $\graph_{f^*}$ and set $\suppset_{\MSP} \coloneqq \cup_{S \in \fcoefs_{\MSP}} S$.
By adapting Lemma \ref{lem:xor_feature_selection_prob}, we can show that any covariate $X_k$ with index not in $\suppset_{\MSP}$ is marginally undetectable, and thus rarely selected, on any query path of a greedy tree.
We then modify Lemma \ref{lem:covariate_selection_tree} and Lemma \ref{lem:covariate_selection_ensemble} to translate lower bounds on the covariate selection probability into more nuanced expected risk lower bounds that can be stated in terms of the \MSP~residual of $f^*$, defined as $r_{\MSP} \coloneqq \sum_{S\in \fcoefs\backslash\fcoefs_{\MSP}}\alpha_S\monomial_S$ (see Lemma \ref{lem:covariate_selection_tree_v2} and Lemma \ref{lem:covariate_selection_ensemble_v2}).

Naively following this recipe will produce a result with a sample dependence of $2^{\delta \nfeats/\sparsity + O(1)}$.
In other words, we have a weaker lower bound on the necessary sample size as the support size of $f^*$ increases.
This dependence does not always make sense:
For instance, the bound obtained for the function $f^* \coloneqq \monomial_{\coordindices}$ is trivial, even though this is in some sense the most complicated function possible.
To overcome this limitation, we make a new definition.
We say a subset $T \subset \coordindices\backslash \suppset_{\MSP} $ is a \emph{traversal} of $\fcoefs\backslash\fcoefs_{\MSP}$ if for any $S \in \fcoefs\backslash\fcoefs_{\MSP}$, we have $\abs{T\cap S} \geq 2$.
It is easy to see that a traversal always exists.
Furthermore, expected risk lower bounds can be established so long as we can derive lower bounds on the probability that the traversal is selected along a query path (see Lemma \ref{lem:nonmsp_feature_selection_prob}).
Putting the pieces together yields the first main result of our paper.

\begin{theorem}
    \label{thm:non_msp_lower_bound}
    Under Assumption \ref{assumption}, let $f_0^* = \sum_{S \in \fcoefs} \alpha_S\monomial_S$ be any non-\MSP~function, and let $r_{\MSP} \coloneqq \sum_{S\in \fcoefs\backslash\fcoefs_{\MSP}}\alpha_S\monomial_S$ denote its \MSP~residual.
    Let $\traversalsize$ be the minimum size of a traversal for $\fcoefs\backslash\fcoefs_{\MSP}$ and let $\sparsity_{\MSP} \coloneqq \abs*{\suppset_{\MSP}}$.
    For any $0 < \delta < 1$,
    suppose $\log_2\nsamples \leq \delta(\nfeats-\sparsity_{\MSP}-\traversalsize+1)/\traversalsize - 2$.
    Then the expected $L^2$ risk of any greedy tree model satisfies
    \begin{equation}
        \objective\paren*{\fgreedy, f^*_0,\nfeats,\nsamples} \geq (1-\delta)\Var\braces*{r_{\MSP}(\bX)}.
    \end{equation}
    If furthermore, $|Y| \leq \ybound$ almost surely, 
    then the $L^2$ risk for any greedy tree ensemble model satisfies
    \begin{equation}
    \label{eq:nonmsp_ensemble_lower_bound}
        \objective\paren*{\fgre, f^*_0,\nfeats,\nsamples} \geq (1-\kappa\delta)\Var\braces*{r_{\MSP}(\bX)},
    \end{equation}
    where $\kappa \coloneqq 2\ybound/\Var\braces*{f_0^*(\bX)}^{1/2}$.
    In particular, $\fgreedy$ and $\fgre$ are high-dimensional inconsistent for $f^*_0$.
\end{theorem}

\begin{proof}
    See Appendix \ref{sec:non_msp_lower_bound_proof}.
\end{proof}



\begin{remark}
    It is possible to obtain similar lower bounds that hold with constant probability rather than in expectation.
    These results are available in Appendix \ref{sec:high_prob_lower_bounds}.
\end{remark}

\begin{remark}
    By orthogonality properties, we have
    \begin{equation}
        \Var\braces*{r_{\MSP}(\bx)} = \sum_{S \in \fcoefs\backslash\fcoefs_{\MSP}} \alpha_S^2 = w(\fcoefs\backslash\fcoefs_{\MSP}).
    \end{equation}
    In other words, the lower bound value is essentially equal to the total weight of vertices in $\mathcal{G}_{f^*}$ not connected to the empty set.
    This quantitative connection between lower bounds and Fourier graph properties will be further enriched in Section \ref{sec:robust_lower_bounds}.
\end{remark}

\section{ERM Trees and Minimax Rates} 
\label{sec:erm}

To help to contextualize our lower bounds for greedy tree models, we first state minimax lower bounds for the estimation problem defined by Assumption \ref{assumption}.
It is well known that for sparse regression problems over continuous covariates, consistency is possible under a scaling regime where the sample size $\nsamples$ is at most polylogarithmic in the dimension $\nfeats$.
This holds both when the regression function is linear \citep{wainwright2019high} and when it is nonlinear \citep{wasserman2005rodeo}.
Over binary covariates, standard information theory techniques can be applied to derive a similar minimax lower bound. 


\begin{proposition}
\label{prop:info_theory_lower_bound}
    Under Assumption \ref{assumption}, further suppose that $\varepsilon_1,\varepsilon,\ldots,\varepsilon_\nsamples \sim_{\text{i.i.d.}} \mathcal{N}(0,\sigma^2)$ and $2^s\log((d-s)/4) \geq 16\log 2$.
    Let $\sparsefuncs$ denote all real-valued functions on $\braces*{\pm 1}^\sparsity$.
    Then we have the minimax lower bound
    \begin{equation}
        \inf_{\hat f}\max_{f_0^* \in \sparsefuncs} \objective\paren*{\hat f, f^*_0,\nfeats,\nsamples} = \Omega\left( \frac{2^s \sigma^2\log d}{n}\right).
    \end{equation}
\end{proposition}

\begin{proof}
    See Appendix \ref{sec:info_theory_lower_bound}.
\end{proof}





Up to logarithmic factors, this minimax rate (in $\nsamples$, $\nfeats$, and $\sparsity$) turns out to be achievable by ERM trees, as shown in the following proposition.
In particular, we see that ERM trees are high-dimensional consistent for any $f_0^*$. 

\begin{proposition}
    \label{prop:erm_upper_bound}
    Under Assumption \ref{assumption}, further suppose that $\varepsilon_1,\varepsilon_2,\ldots,\varepsilon_\nsamples \sim_{\text{i.i.d.}} \mathcal{N}(0,\sigma^2)$ and $\norm*{f_0^*}_\infty \leq M$.
    Consider the ERM tree model \eqref{eq:ERM_equation} with the constraint set $\constraint$ comprising all tree structures of depth at most $\sparsity$.
    For any $f^*_0$, its expected $L^2$ risk satisfies the upper bound
    \begin{equation}
        \label{eq:ERM_rate}
        \objective\paren*{\ferm, f^*_0,\nfeats,\nsamples} = O\left(\frac{(\sigma^2+M^2)(2^s\log\nfeats + \log\nsamples)}{n}\right).
    \end{equation}
    In particular, $\ferm$ has high-dimensional consistency with respect to $f_0^*$.
\end{proposition}

\begin{proof}
    See Appendix \ref{sec:erm_proofs}.
\end{proof}

\begin{remark}
    The Gaussian assumption on the noise can be relaxed to a sub-Gaussian assumption.
    This extension however adds technical complications to the proof and is hence omitted.
    In comparison, note that our lower bounds for greedy trees and ensembles have not required any additional assumptions on the noise distribution.
\end{remark}



\section{CART and MSP Functions}

\subsection{Characterizing High-Dimensional Consistency for CART}
\label{sec:characterizing_CART}

In this section, we will show that \MSP~is a nearly sufficient condition on a regression function $f^*$ in order for CART to have high-dimensional consistency.
Together with 
results from Section \ref{sec:greedy_and_non_msp}, this would complete the characterization we desire.

When $f^*$ satisfies \MSP, its associated hypergraph is connected.
It is easy to verify that this implies that for any cell $\cell$ on which $f^*$ is nonconstant, there exists a term $\alpha_S\monomial_S$ in its Fourier expansion such that $S\backslash \cellfeatindices(\cell) = \braces{j}$ for some $j$.
Unfortunately, this is not enough to guarantee that $X_j$ is marginally detectable, due to the possibility of cancellations.
To see how these may arise, consider the function
\begin{equation}
    f^*(\bx) = x_1 + x_2 + x_1x_2 + x_2x_3.
\end{equation}
Then on the cell $\cell \coloneqq \braces*{ \bx \in \binspace \colon x_1 = -1}$, the restriction of $f^*$ to $\cell$ satisfies
\begin{equation}
\label{eq:cancellation}
    f^*|_\cell(\bx) = -1 + x_2 + (-1)x_2 + x_2x_3 = -1 + x_2x_3.
\end{equation}
In other words, $f^*|_\cell$ is non-\MSP, and by the results of the previous section, CART as well as other greedy tree methods cannot make further progress.


Fortunately, in some sense, exact cancellations happen rarely.
To be more precise, first note that \MSP~is a condition on the Fourier support $\fcoefs$ of $f^*$ rather than on $f^*$ itself.
Starting with such a collection, we see that the existence of a cancellation such as that in \eqref{eq:cancellation} is a linear constraint on the coefficients $\braces*{\alpha_S}_{S \in \fcoefs}$.
If we draw these values from any distribution on $\R^{|\fcoefs|}$ that is absolutely continuous with respect to Lebesgue measure, the constraint hence forms a measure zero set, as does the union over all possible constraints. \comment{red}{Thus, our results may equivalently be read deterministically for any fixed set of coefficients outside this exceptional set.}

We say that a function $f^*$ satisfies the \emph{stable merged-staircase property} (\stabmsp) if, for any $\cell \subset \binspace$, the restriction $f^*|_\cell$ is \MSP.
We have thus shown that under a random perturbation of its coefficients, any \MSP~function is \stabmsp~almost surely (see Proposition \ref{prop:genericity}).

Next, one can prove that $f^*$ is \stabmsp~if and only if it satisfies the \SID~condition \eqref{eq:SID} for some value of $\lambda > 0$ (see Proposition \ref{prop:stable_msp_and_SID}).
Previous work, specifically \cite{chi2022asymptotic}, who first introduced the \SID~condition, as well as \cite{mazumder2024convergence}, provide upper bounds for the estimation error of CART given the \SID~condition.
These results and proofs are stated for continuous covariates, but can easily be adapted to a discrete covariate setting.
We revisit this last statement in the next subsection.
Assuming this for now, our discussion thus far yields the following characterization theorem.

\begin{theorem}
    Let $\fcoefs$ be a collection of subsets of $\suppset$.
    Draw coefficients $\braces*{\alpha_S}_{S \in \fcoefs}$ from any distribution on $\R^{|\fcoefs|}$ that is absolutely continuous with respect to Lebesgue measure, and set $f^*_0 = \sum_{S \in \fcoefs}\alpha_S\monomial_S$.
    Then almost surely, the following holds under Assumption \ref{assumption}:
    \begin{itemize}
        \item If $\fcoefs$ satisfies \MSP, then the CART estimator $\fcart$ is high-dimensional consistent with respect to $f_0^*$.
        \item If $\fcoefs$ does not satisfy \MSP, then the CART estimator $\fcart$ is high-dimensional inconsistent with respect to $f_0^*$.
    \end{itemize}
\end{theorem}

\begin{remark}
    While \cite{syrgkanis2020estimation} obtained a family of upper bounds for CART in a binary covariate setting, their assumed condition (see Assumption 4.1 their paper) turns out to be stronger than the \SID~condition.
    Their assumption comprises a type of submodularity for the population impurity decrease values for a \emph{fixed} covariate, more precisely, that $\Delta(k;\cell) \geq \lambda \Delta(k;\cell')$, whenever $\cell' \subset \cell$ and for all $k \notin \cellfeatindices(\cell')$.
    Here
    \begin{equation}
        \Delta(k;\cell) \coloneqq \Corr^2\braces*{Y,X_k~|~\bX \in \cell}\Var\braces*{f^*(\bX)~|~\bX \in \cell}.
    \end{equation}
    For an example of an \SID~function that does not satisfy this assumption, consider $f^*(\bx) = x_1 + x_1x_2$, in which $\Delta\paren*{2,\binspace} = 0$ and $\Delta\paren*{2, \braces*{\bx \in \binspace \colon x_1=1}} = 1$.
\end{remark}

\begin{remark}
    The failure of the directed hypergraph structure $\mathcal{G}_{f^*}$ to fully reflect the statistical properties of the regression function $f^*$ under some circumstances is reminiscent of the potential for DAGs representing a Markov factorization of a distribution to fail to be faithful to its conditional independence structure \citep{uhler2013geometry}.
\end{remark}

\begin{remark}
    It is natural to ask whether our lower bounds can be extended to handle all $f^*$ that do not satisfy \stabmsp~(or equivalently the \SID~condition).
    This turns out to be impossible---there exist functions that do not satisfy \stabmsp, but for which high-dimensional consistency holds.
    For instance, consider $f^*(\bx) = x_1 + 2x_2 + 2x_1x_2 + x_2x_3$.
    Since $f^*|_{\braces*{x_1=-1}} = x_2x_3$, this does not satisfy \stabmsp, but on the other hand, CART would always first split on $x_2$ given enough samples, in which case this cancellation never occurs. Indeed, by reformulating CART as a sequential greedy approximation algorithm for convex optimization in a Hilbert space, \cite{klusowski2024large} showed high-dimensional consistency for additive models and arbitrary covariates without the \SID~condition.
\end{remark}

\subsection{Sharp Upper Bounds for CART}

The \SID~assumption does not always lead to sharp bounds.
Consider for instance the $\sparsity$th-order AND function $f^*(\bx) = \indicator\braces*{x_j = 1 \colon j =1,2,\ldots,\sparsity}$.
We may compute $\Corr^2\braces*{f^*(\bX),X_j} = (2^\sparsity-1)^{-1}$, which implies that $f^*$ does not satisfy \SID~with any $\lambda$ below this value. Note that the \SID~condition is used to guarantee that, for every cell $\cell$, there exists a potential split that would decrease the residual variance within $\cell$ by a constant factor of at most $1-\lambda$. 
\comment{red}{Combined with concentration of impurity decrease values, one can then form a recursive inequality of the form $J(l+1) \leq (1-\lambda/(1+\epsilon))J(l)$, where $J(l)$ denotes the squared bias of the CART tree grown to depth $l$ and $\epsilon$ is a finite-sample error. Iterating gives $J(l) \leq J(0)\exp(-\lambda l/(1+\epsilon))$, while the variance at depth $l$ is $O(2^l/n)$, up to poly-log factors in $n$ and $d$. Balancing bias and variance requires $\exp(-\lambda l / (1+\epsilon)) \approx 2^l/n$, i.e., $l \approx \log n / (\lambda/(1+\epsilon) + \log 2)$, in which case we obtain the optimized risk bound $O(n^{-\lambda / ((1+\epsilon)\log 2 + \lambda)})$.  
Substituting $\lambda = (2^\sparsity - 1)^{-1}$ shows that the exponent $\lambda / ((1+\epsilon)\log 2 + \lambda)$ is of the order $2^{-\sparsity}$, so the bound is trivial unless $\log_2 n = \Omega(2^\sparsity)$.}  
We will soon show that this is unnecessarily pessimistic, with the true sample complexity of the order $\log_2 \nsamples = 2\sparsity + o(1)$.

On the other hand, while the AND function has poor marginal signal strength at the root, the marginal signal strength improves as the tree gets deeper and more of the relevant covariates are selected.
Indeed, we have $\Corr^2\braces*{f^*(\bX),X_j~|~X_1=X_2=\cdots=X_{j-1}=1} = (2^{\sparsity-j}-1)^{-1}$.
By only quantifying the worst case signal strength, the \SID~condition is unable to capture and exploit this nuance.


To provide a sharper bound, we need to provide a new definition.
We say that $f^*$ has the \emph{$\lambda$-stable merged staircase property}, denoted $f^* \in \stabmsp(\lambda)$, if for any cell $\cell \subset \binspace$ such that $f|_\cell$ is not constant, we have
\begin{equation} \label{eq:stab_msp}
    \max_{k \in \coordindices\backslash \cellfeatindices(\cell)}\Cov^2\braces*{f^*(\bX),X_k~|~\bX \in \cell}2^{-|\cellfeatindices(\cell)\cap \suppset|} \geq \lambda.
\end{equation}
Compared with the SID formula \eqref{eq:SID}, this definition is stated in terms of the squared covariance and further modulates it with the base 2 exponential of the number of relevant covariates already split on.
This quantity more accurately captures the signal-to-noise ratio for the problem of covariate selection within $\cell$.
Indeed, rather than attempt to control the amount of bias reduction at each split, we will instead lower bound the probability of selecting a relevant covariate.
Formalizing this approach yields the following theorem.


\begin{theorem}
    \label{thm:MSP_upper_bound}
    Under Assumption \ref{assumption}, let $f_0^* \in \stabmsp(\lambda)$.
    Suppose $\varepsilon$ is sub-Gaussian with parameter $\sigma$ \citep{wainwright2019high}.
    Denote 
    \begin{equation}
        \tau \coloneqq 18(9\norm{f^*}_\infty^2 + \sigma^2)\paren*{(s+2)\log 3 + \log(2\nfeats\nsamples)}
    \end{equation}
    and suppose the training sample size satisfies $n > \frac{8\tau}{\lambda}$.
    Suppose further that we fit $\fcart$ with the stopping rule given by a minimum impurity decrease value $\gamma$ satisfying $\frac{\tau}{n} \leq \gamma < \Big(\sqrt{\frac{\lambda}{2}} - \sqrt{\frac{\tau}{n}}\Big)_+^2$.
    Such a value $\gamma$ exists given the lower bound on $\nsamples$.
    We then have the expected $L^2$ risk upper bound
    \begin{equation}
        \label{eq:msp_upper_expectation}
        \objective\paren*{\fgreedy, f^*_0,\nfeats,\nsamples} = O\paren*{\frac{2^{\sparsity}(\norm{f^*}_\infty^2 + \sigma^2)(\sparsity + \log \nsamples)}{n}}.
    \end{equation}
    In particular, $\fcart$ is high-dimensional consistent for $f_0^*$.
\end{theorem}

\begin{proof}
    See Appendix \ref{sec:upper_bounds_proofs}.
\end{proof}

Returning to the example of the AND function, it is easy to check that we may set $\lambda$ in \eqref{eq:stab_msp} to be $2^{-2\sparsity}$.
Plugging this into Theorem \ref{thm:MSP_upper_bound} shows that \eqref{eq:msp_upper_expectation} holds whenever $\nsamples = \tilde{\Omega}(\sparsity2^{2\sparsity})$ in agreement with what was mentioned earlier.

\section{Robust Lower Bounds for Greedy Trees} 
\label{sec:robust_lower_bounds}

In this section, we will prove lower bounds for \MSP~functions that depend on how far they are from being non-\MSP.
Since being non-\MSP~is a somewhat fragile property and can be destroyed by small perturbations of the zero coefficients of the function, we call these robust lower bounds. 
In doing so, we will also strengthen the connection between the Fourier graph of a function $f^*$ and its estimation properties using greedy trees and ensembles.
These lower bounds will provide further performance gaps between greedy trees and ERM trees.
Finally, comparing the lower bounds with the results from the previous section enable us to get a sense of how sharp both sets of bounds are.

To motivate the theorem and our proof technique, we first revisit the $s$th order AND function discussed in the previous section.
We have already seen that $\tilde{O}(\sparsity2^{2\sparsity})$ samples are sufficient for CART to achieve small estimation error.
While a much smaller number than what was previously achievable using the \SID~assumption, it is still larger than the upper bound of $\tilde{O}(2^s)$ achievable using ERM trees (Proposition \ref{prop:erm_upper_bound}).
Nonetheless, the weak upper bound turns out to be relatively sharp.
To see this, we write
\begin{equation}
        \indicator\braces*{x_j = 1 \colon j =1,2,\ldots,\sparsity} = \prod_{j=1}^\sparsity \paren*{\frac{x_j + 1}{2}}
         = 2^{-\sparsity} + \sum_{j=1}^\sparsity  2^{-\sparsity} x_j + \text{higher order terms}.
\end{equation}
From this representation, we see that all first order terms have small coefficients when $\sparsity$ is relatively large, and if removed, will turn the function into a non-\MSP~one.

This argument can be made rigorous via the technique of coupling.
Let $f^*$ denote the AND function and let define $\check f^*$ via $\check f^*(\bx) = f^*(\bx) - \sum_{j=1}^\sparsity  2^{-\sparsity} x_j$.
Now define a dataset $\modifieddata \coloneqq \braces*{(\bX_i,\check Y_i)}_{i=1}^\nsamples$ with $\check Y_i = \check f^*(\bX_i) + \varepsilon_i$ for $i=1,2,\ldots,\nsamples$.
If $\varepsilon_1,\varepsilon_2,\ldots,\varepsilon_\nsamples \sim \mathcal{N}(0,\sigma^2)$, it is easy to check that the Kullback–Leibler divergence between the original and modified datasets satisfies
\begin{equation}
    \KL\paren*{\data\| \modifieddata} = \frac{\sparsity2^{-2\sparsity}\nsamples}{2\sigma^2}.
\end{equation}
Using Pinsker's inequality, this gives a bound on the total variation distance, which we use to argue that there is a coupling between the two datasets with a probability at least $1-\sqrt{\frac{s2^{-2s}n}{4\sigma^2}}$ event on which they are equal.
The probability of this event is thus at least $1/2$ unless $\nsamples = \Omega(2^{2s}/s)$, and on this event, a lower bound on the estimation error for $\check f^*$ translates to one on the estimation error for $f^*$.

This coupling technique can be extended to handle other \MSP~functions.
Considering the Fourier graph of the AND function, we see that the vertices $\braces*{1},\braces{2},\ldots,\braces{\sparsity}$ form a \emph{vertex cut} whose removal disconnects the vertex $\emptyset$ from the rest of the graph.
More generally, if we can identify a vertex cut $\cutset$ for the Fourier graph of some function $f^* = \sum_{S \in \fcoefs} \alpha_S\monomial_S$, its weight $w(\cutset)$ quantifies the $L^2$ distance between $f^*$ and the non-\MSP~function $f^*_{-\cutset} \coloneqq \sum_{S \in \fcoefs\backslash\cutset} \alpha_S\monomial_S$.
Any estimation error lower bound for $f^*_{-\cutset}$ can therefore be translated into a lower bound for $f^*$ on the probability at least $1-\sqrt{\frac{w(\cutset)\nsamples}{2\sigma^2}}$ event that datasets generated with these regression functions can be coupled.

To state this more formally, we first note that the Fourier graph of $f^*_{-\cutset}$ is equivalent to the induced subgraph of $\fcoefs\backslash\cutset$ in $\mathcal{G}_{f^*}$, denoted $\mathcal{G}_{f^*}[\fcoefs\backslash\cutset]$.
Denote the connected component of the empty set $\emptyset$ in $\mathcal{G}_{f^*}[\fcoefs\backslash\cutset]$ as $\fcoefs_{-\cutset,\MSP}$ and set $\suppset_{-\cutset,\MSP} \coloneqq \cup_{S \in \fcoefs_{-\cutset,\MSP}} S$.
Furthermore, let $\fcoefs_{-\cutset,-\MSP} \coloneqq \fcoefs\backslash\paren*{\fcoefs_{-\cutset,\MSP}\cup\cutset}$ denote the subcollection comprising all vertices disconnected from $\emptyset$ in $\mathcal{G}_{f^*}[\fcoefs\backslash\cutset]$.
We have the following elaboration on Theorem \ref{thm:non_msp_lower_bound}.

\begin{theorem}
    \label{thm:robust_lower_bounds}
    Under Assumption \ref{assumption}, suppose $\varepsilon_1,\varepsilon_2,\ldots,\varepsilon_\nsamples \sim_{\text{i.i.d.}} \mathcal{N}(0,\sigma^2)$, and let $\mathcal{G}_{f_0^*}$ be the Fourier graph of $f_0^*$ with weight function $w$.
    Let $\cutset \subset \mathcal{S}$ be a subset of vertices of $\mathcal{G}_{f_0^*}$.
    Let $\traversalsize$ be the minimum size of a traversal $\traversal$ for $\fcoefs_{-\cutset,-\MSP}$ such that $\traversal \subset \coordindices \backslash \suppset_{-\cutset,\MSP}$ and let $s_{-\cutset,\MSP} \coloneqq \abs*{\suppset_{-\cutset,\MSP}}$.
    For any $0 < \delta < 1$, suppose
    \begin{equation}
        \log_2 n \leq \min\braces*{{\frac{\delta\paren*{\nfeats - s_{-\cutset,\MSP} - \traversalsize}}{2\traversalsize} - 2},\; \log_2\frac{\delta^2\sigma^2}{w(\cutset)} - 1}.
    \end{equation}
    Then the expected $L^2$ risk of any greedy tree model satisfies
    \begin{equation}
        \objective\paren*{\fgreedy, f^*_0,\nfeats,\nsamples} \geq (1-\delta)w\paren*{\fcoefs_{-\cutset,-\MSP}}.
    \end{equation}
\end{theorem}

\begin{proof}
    See Appendix \ref{sec:robust_lower_bounds_proofs}.
\end{proof}

As seen from the AND function, the sample complexity bound in Theorem \ref{thm:robust_lower_bounds} can be relatively sharp.
This is not always the case however.
Consider, for instance, $f^* = \sum_{j=1}^{\sparsity - 1} \monomial_{\braces*{1,2,\ldots,j}} + \alpha \monomial_{\braces*{1,2,\ldots,\sparsity}} + \sum_{j=\sparsity+1}^{2\sparsity} \monomial_{\braces*{1,2,\ldots,j}}$.
Applying Theorem \ref{thm:robust_lower_bounds} to this function gives a sample complexity lower bound of $\min\braces*{\Omega(1/\alpha^2),\exp(\Omega(\nfeats))}$.
On the other hand, we have
\begin{equation}
    \max_{k \in \braces*{s,s+1,\ldots,\nfeats}}\Cov^2\braces*{f^*(\bX),X_k~|~X_1,X_2,\ldots,X_{s-1}}2^{-s} = \alpha^2 2^{-s},
\end{equation}
which implies that the sample complexity upper bound from Theorem \ref{thm:MSP_upper_bound} is $\Omega(2^s/\alpha^2)$.
One may show that this upper bound is sharp---due to the structure of $f^*$, with enough samples, CART will only attempt to split on $x_s$ at depth $s$, in which case the sample size available for determining this split is approximately $n/2^s$.
This nuance is not captured by our lower bound, yet any attempt to address it seems to require sacrificing some aspect of the generality of our result.
We hence leave this effort to future work.

\begin{remark}
    Since \eqref{eq:nonmsp_ensemble_lower_bound}, the lower bounds for greedy tree ensembles, requires a bounded noise assumption, it does not immediately extend to a robust lower bound, which requires a Gaussian noise assumption.
    This can be overcome by conditioning on the highly probable event on which $\max_{1 \leq i \leq \nsamples}\abs*{Y_i} \leq C\sigma\sqrt{\log\nsamples}$ for some universal constant $C$.
\end{remark}

\section{Extensions of Lower Bounds}
\label{sec:extensions}

The results in our paper have been presented in the setting of binary covariates under the uniform distribution.
Nonetheless, our lower bound proof techniques also apply more generally, which we illustrate in this section.
First, recall that our strategy has been to (i) derive an upper bound on the probability of selecting relevant covariates (c.f., Lemma \ref{lem:xor_feature_selection_prob}) and then (ii) turn this into a lower bound on estimation error (c.f., Lemma \ref{lem:covariate_selection_tree} and Lemma \ref{lem:covariate_selection_ensemble}).
The arguments for the second step are very general and depend only on the independence between the relevant covariates and the irrelevant covariates.
As such, aside from this consideration, the barrier to any extension lies with the the first step.
With this in mind, we illustrate two extensions of the bounds to continuous covariates.
That is, we study the setting $\xspace = \ctsspace$, $\xmeasure = \operatorname{Unif}\paren*{\ctsspace}$.
Note also that in the continuous setting, CART and RFs are invariant to rescaling of the covariate values, such that any analysis of the uniform distribution immediately extends to distributions with continuous densities and independent covariates.

\begin{example}
    Consider the original CART and RF algorithms \citep{friedman2001elements}.
    The XOR function can be extended to continuous covariates as a step function:
    \begin{equation}
        f^*(\bx) = \indicator\braces*{x_1x_2 \geq 1/2} - \indicator\braces*{x_1x_2 < 1/2}.
    \end{equation}
    Similar to discrete covariate setting, there is marginal signal bottleneck, and we may completely imitate the proof of Lemma \ref{lem:xor_feature_selection_prob}, except that we can no longer use the results from Appendix \ref{sec:query_path_length} to bound the query path length.
    We can enforce this by constraining the algorithm to make relatively balanced splits on irrelevant covariates (i.e., whenever $\cell_L$ and $\cell_R$ are the children resulting from a chosen split, the log ratio $\log\paren*{\xmeasure(\cell_L)/\xmeasure(\cell_R)}$ has to be bounded with high probability).
    With this, we can obtain 
    an estimation error lower bound for the CART and RF estimators whenever $\log_2\nsamples \leq C\nfeats$ for some universal $C$.
\end{example}

Observe that the more general lower bounds in Theorem \ref{thm:non_msp_lower_bound} and Theorem \ref{thm:robust_lower_bounds} also generalize, with the modification that each monomial $\monomial_S$ is redefined as 
\begin{equation}
    \monomial_S(\bx) \coloneqq \prod_{j \in S} 2^{-1/2}\paren*{\indicator\braces*{x_j \geq 1/2} - \indicator\braces*{x_j < 1/2}}.
\end{equation}
However, unlike the binary covariate setting in which these monomials form a basis for the entire space of $L^2$ functions, these are but a small subset of $L^2\paren*{\ctsspace}$, which makes the theory less satisfying.

On the other hand, we note that $L^2\paren*{\ctsspace}$ possesses a wavelet basis, which is constructed as follows:
For each depth $\depth=0,1,2,\ldots$, for $\idx=1,2,\ldots,2^{\depth}$, define the univariate function
\begin{equation}
    \wavelet_{\depth,\idx}(x) \coloneqq - 2^{\depth/2}\indicator\braces*{\frac{m-1}{2^{l}} < x \leq \frac{m-1/2}{2^{l}}} + 2^{\depth/2}\indicator\braces*{\frac{m-1/2}{2^{l}} < x \leq \frac{m}{2^{l}}}.
\end{equation}
As we vary $l$ and $m$, this forms the Haar wavelet basis for $L^2\paren*{[0,1]}$, and we can extend it to a basis for $L^2\paren*{\ctsspace}$ by taking products.
Just as the Fourier basis gives a convenient basis for analyzing CART when applied to binary covariates, the product Haar wavelet basis is convenient for analyzing dyadic CART, whose splits are constrained to be at the midpoint of each side of a given cell.
Indeed, the impurity decrease for splitting on a covariate $X_k$ in a cell $\cell$ can be written as $\widehat\Cov^2\braces*{Y,\wavelet_{l,m}(X_k)~|~\bX \in \cell}$ for an appropriate choice of $(l,m)$.

\begin{example}
    Consider the sine function $f^*(\bx) = \sin(\pi x_1)$.
    One can show that for any cell $\cell$, unless $\cell$ has already split on $x_1$, $f^*(\bX)$ is conditionally independent of $\wavelet_{0,1}(X_1)$ in $\cell$.
    Using a modification of our coupling technique, one can therefore show that dyadic CART, when fitted to data generated from this function, will make random splits.
    We therefore obtain a constant lower bound for the estimation error of dyadic CART and RF estimators whenever $\log_2\nsamples \leq C\nfeats$ for some universal $C$.
\end{example}

To see how far we can push this result, we next construct the equivalent of the Fourier graph.
To this end, we first index each basis element using a collection of tuples, i.e., given $\Lambda \coloneqq \braces*{(k_j,l_j,m_j)\colon j=1,2,\ldots,r}$, define
$\monomial_{\Lambda}(\bx) \coloneqq \prod_{j=1}^r \wavelet_{l_j,m_j}(x_{k_j})$.
Now given any function $f^* \in L^2\paren*{\ctsspace}$, we may write it as $f^*\coloneqq  \sum_{\Lambda \in \mathfrak{L}} \alpha_\Lambda \monomial_\lambda$.
Define the directed hypergraph $\mathcal{G}$ with vertex set $\mathfrak{L}$ and put an edge $\braces*{\Lambda^{(1)},\Lambda^{(2)},\ldots,\Lambda^{(q)}} \to \Lambda^{(q+1)}$ if the following hold: (i) the union $\cup_{j=1}^q \Lambda^{(j)}$ contains at most one entry for each covariate, (ii) $\abs*{\Lambda^{(q+1)}\backslash \cup_{j=1}^q \Lambda^{(j)}} \leq 1$, and (iii) if $\Lambda^{(q+1)}\backslash \cup_{j=1}^q \Lambda^{(j)} = \braces*{(k,l,m)}$, then $(k,l-1,\lceil m/2 \rceil) \in \cup_{j=1}^q \Lambda^{(j)}$.
Similar to the binary case, we define a weight function via $w(\Lambda) \coloneqq \alpha_\Lambda^2$ and call $(\mathcal{G},w)$ the \emph{wavelet graph} of $f^*$.

The natural analogue of \MSP~is connectedness in the wavelet graph and one would expect the absence of this condition to imply high-dimensional inconsistency.
Unfortunately, disconnectedness alone does not imply the conditional independence we need for our previous proofs to go through.
For instance, if we set $f^*(\bx) = \wavelet_{1,1}(x_1)$, then $\braces*{(1,1,1)}$ is disconnected from $\emptyset$ in the wavelet graph.
While this implies that the random variable $\psi_{0,1}(X_1)$, which corresponds to a first split on $x_1$, is uncorrelated with the response variable $Y$ in any cell, the two variables are still clearly dependent.
We leave how to address this issue to future work.

\comment{red}{Finally, we note that while our current analysis is tailored to the uniform distribution and relies on the structure of Fourier or wavelet bases, recent work in computational learning theory has pursued more distribution-agnostic complexity characterizations. For example, \citet{joshi2024on} analyze the statistical and gradient query complexity of learning sparse functions under general product measures.}

\section{Overcoming Marginal Signal Bottleneck}


The proof of our lower bounds isolates a source of the potential poor performance of greedy trees:
It stems from their use of marginal projections of the dataset (c.f., Equation \eqref{eq:greedy_criterion}) to determine splits instead of using its joint distribution.
While this increases computational efficiency, it runs into problems when large regression function heterogeneity is masked by weak marginal signal. 
We called this problem \emph{marginal signal bottleneck}.
In this section, we discuss the ability of other classes of regression tree algorithms to avoid the marginal signal bottleneck problem and the computation price each incurs in doing so.

Within the family of greedy algorithms, \cite{loh2002regression}'s GUIDE algorithm tests for pairwise interactions when selecting splits and is thus able to estimate the XOR function.
On the other hand, it is unable to estimate higher-order monomials.
\comment{red}{One can generalize this idea in a binary feature setting to permit splits on $k$-way monomials $\prod_{j \in S} x_j$ with $|S|\leq k$. 
Such an algorithm can indeed learn leap-$k$ regression functions (e.g. order $k$ monomials, see also Appendix \ref{sec:related_work}), under a genericity condition on Fourier coefficients, with a much lower sample complexity of $O(k\log d)$, but at a much higher computational cost of $O(nd^k)$ per split.
However, for functions beyond leap-$k$ (e.g., monomials of order larger than $k$), all conditional correlations of degree $\leq k$ vanish, so consistency still requires $\exp(\Omega(d))$ samples.}
Alternatively, CART can be modified to allow for splits that are oblique, that is, along linear hyperplanes $\sum_{j=1}^d \beta_j x_j $ depending on multiple covariates \citep{cattaneo2022convergence}.
This creates fewer opportunities for signal bottleneck---for instance, the XOR function in continuous covariates is no longer an issue.
However, optimizing over oblique splits is vastly more computationally intensive and also reduces the interpretability of the resulting model. 

Bayesian CART \citep{chipman1998bayesian} and its ensemble version, Bayesian Additive Regression Trees (BART) \citep{chipman2010bart} put regression trees in a Bayesian framework and use a Metropolis-Hastings strategy to sample from its posterior.
Each MCMC iteration can be thought of as a stochastic update of the tree structure and can potentially escape local minima in which deterministic methods may get stuck.
However, \cite{tan2024computational} was able to show that the sampler mixes slowly and in fact becomes increasingly greedy as the training data increases, leading to problems with estimating the XOR function.


Due to improvements in hardware as well as generic optimization solvers, directly solving the ERM problem for decision trees has become much more feasible than before and is currently an active area of research.
Various approaches have been formulated, including mixed integer linear programming \citep{bertsimas2017optimal,verwer2019learning,zhu2020scalable,aghaei2021strong} and dynamic programming \citep{hu2019optimal,lin2020generalized,aglin2020learning,demirovic2022murtree,zhang2023optimal}.
Most of these works focus solely on algorithms for classification trees, with \cite{zhang2023optimal} the only work we are aware of that as introduced an algorithm for optimal regression trees. On the theoretical front, it is interesting to examine if this class of algorithms can provably overcome the limitations of greedy trees and their ensembles, under polynomial runtime constraints.

\section{Simulations}
\label{sec:simulations}

To supplement our theory, we perform a simulation study in which we generate training data according to \eqref{eq:nonparametric_regression_model}, with the covariates distributed uniformly on $\braces{\pm 1}^d$ and the regression function chosen to be $f^*(\bx) = x_1x_2 + \alpha x_1$.
This corresponds to the XOR function when $\alpha = 0$, while setting nonzero $\alpha$ allows for some amount of marginal signal.
We vary $d \in \braces{10, 20, 30, 40, 50}$, $\log_2 n \in \braces{7,9,11,13,15}$, $\alpha \in \braces{0, 0.02}$, and $\sigma^2 \in \braces{0, 0.1}$.
We use two metrics to evaluate the performance of CART on each dataset:
\begin{itemize}
    \item Expected mean squared error (MSE) $R(\fcart, f^*)$, where $\fcart$ is fitted with a minimum impurity decrease stopping criterion \eqref{eq:min_imp_dec}. Here, the threshold $\gamma$ is selected based on a held-out validation set.
    \item Split coverage of each feature $X_k$ by a CART tree $\tree$ that is grown to purity, that is, with no stopping criterion.
    Split coverage is defined as:
    \begin{equation} \label{eq:coverage}
    \text{coverage}(k,\tree) \coloneqq \sum_{\leaf \in \partition(\tree)} 2^{-\text{depth}(\leaf)}\indicator\braces*{\text{path from root to}~\leaf~\text{splits on}~X_k}.
\end{equation}
For the uniform measure, $\text{coverage}(k,\tree)$ computes the probability that the root-to-leaf path in $\tree$ followed by a random query point $\bX$ splits on the feature $X_k$.
\end{itemize}

Selected results from the experiment, averaged over 200 replicates, are shown in Figure \ref{fig:heatmap_mse} and Figure \ref{fig:heatmap_coverage}.
In each panel of Figure \ref{fig:heatmap_mse}, we plot how expected MSE varies across the grid of $n$ and $d$ values, stratifying on all other hyperparameters of the data generating process.
In both panels, we focus on the noiseless setting ($\sigma^2 = 0$), under which the null model satisfies $R(0,f^*) = 1 + \alpha^2$.
The left panel corresponds to the XOR setting ($\alpha=0$), while the right panel corresponds to $\alpha=0.02$.
Note that while $d$ varies on a linear scale, $n$ varies on a logarithmic scale.
Nonetheless, in the left panel, we see that the expected MSE values increase along the bottom-left to top-right diagonal, reflecting that the sample complexity required to ``learn'' the regression function has exponential dependence on the ambient dimension, which is in agreement with Proposition \ref{prop:xor_lower_bound}.
This exponential dependence disappears in the right panel, even though the energy of the additional term, $0.02X_1$, is much smaller than that of the main term $X_1X_2$.
This agrees with the high-dimensional consistency guarantee provided in Theorem \ref{thm:MSP_upper_bound}.
Figure \ref{fig:heatmap_coverage} similarly plots how split coverage varies across the grid of $n$ and $d$ values, fixing $\alpha = 0$ and $\sigma^2 = 0$.
For this figure, however, the left panel corresponds to the coverage of $X_2$, a relevant feature, while the right panel corresponds to the coverage of $X_3$, which is an irrelevant feature.
In the left panel, we see that coverage of $X_2$ has a decreasing trend across the bottom-left to top-right diagonal, which again shows that the sample complexity required for reliably performing feature selection grows exponentially in $d$, in agreement with Lemma \ref{lem:xor_feature_selection_prob}.
The right panel gives reference values for coverage of an irrelevant feature. One may observe that these values are upper bounded by $\log_2 n/d$.
Additional simulation results, including those for $\sigma^2 = 0.01$, are shown in Appendix \ref{sec:sims_appendix}.

\begin{figure}
    \centering
    \includegraphics[width=\linewidth]{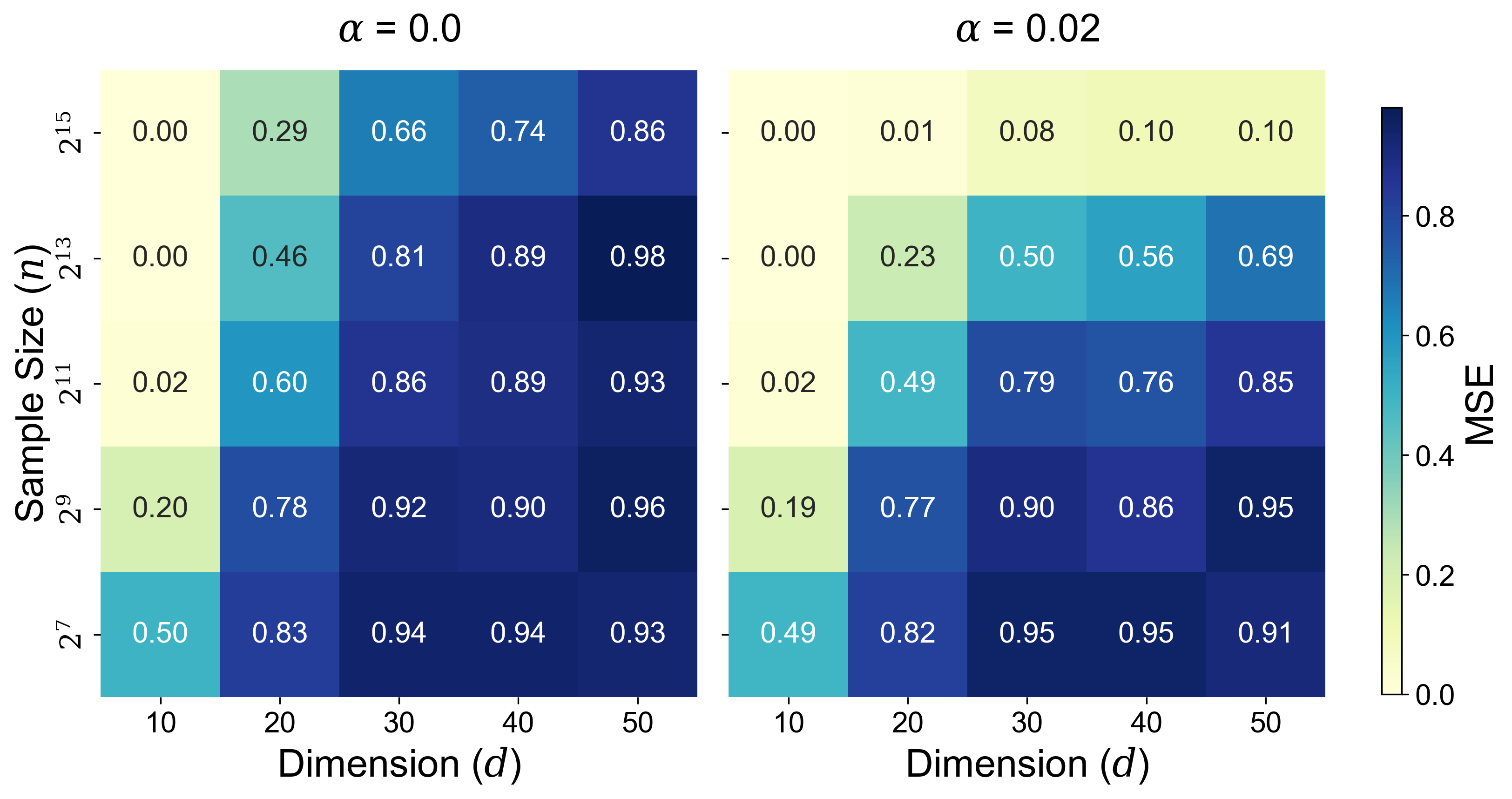}
    \caption{Expected MSE $R(\fcart,f)$ for CART fitted to $Y = X_1X_2 + \alpha X_1$, with $\alpha = 0$ (left panel) and $\alpha = 0.02$ (right panel). CART is fitted with a minimum impurity decrease stopping criterion. We observe high-dimensional consistency in the right panel since the true function satisfies \MSP and lack of high-dimensional consistency in the left panel since the true function does not satisfy \MSP.}
    \label{fig:heatmap_mse}
\end{figure}

\begin{figure}
    \centering
    \includegraphics[width=\linewidth]{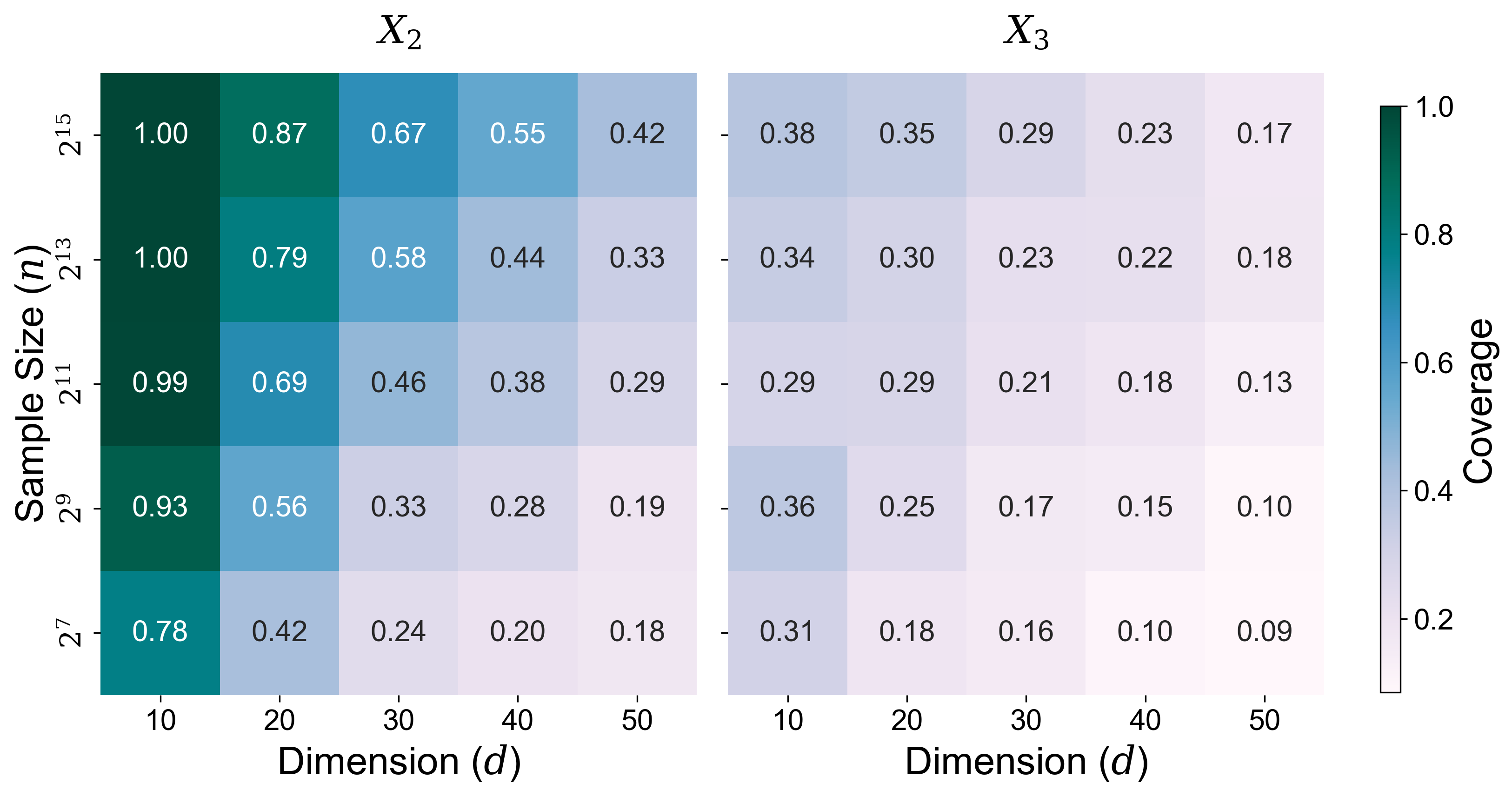}
    \caption{Split coverage for features $X_2$ and $X_3$ for CART fitted to $Y = X_1X_2$, without any stopping criterion. Split coverage is defined in \eqref{eq:coverage}. We observe that the sample size required to reliably split on the relevant feature grows exponentially in $d$.}
    \label{fig:heatmap_coverage}
\end{figure}


\section{Discussion}
\label{sec:conclusion}

In this paper, we have provided rigorous estimation error lower bounds for CART and RFs that show they have the potential to get stuck in local minima. 
In fact, we have been able to go much further by evaluating the performance of regression tree algorithms in terms of high-dimensional consistency.
We have used the \MSP~to completely characterize the regression functions for which CART is high-dimensional consistent and secondly established performance gaps between greedy and globally optimal tree algorithms over a large class of regression functions.

Our work adds to the small but growing literature on negative results for regression trees and ensembles.
\cite{tang2018when} showed that RFs can be inconsistent when its constituent trees are not diverse enough.
\cite{tan2021cautionary} proved that individual regression trees are inefficient at estimating additive structure, regardless of the way in which they are optimized.
\cite{cattaneo2022pointwise} demonstrate that when trees are fitted with CART to a constant signal model, they tend to place splits near the boundaries of the covariate space. 
This tendency can result in nodes containing very few samples---a significant problem when uniform control over estimation accuracy is desired, such as in the context of heterogeneous treatment effect estimation.
Finally, \cite{tan2024computational} proved that mixing times for Bayesian Additive Regression Trees can increase with the training sample size.

These results, especially when used as part of a head-to-head comparison against other methods, sound a note of caution that regression trees and ensembles should not be used blindly, despite their popularity---in certain situations, their output can be unexpectedly inaccurate or vastly inferior to that of other methods.
They also suggest ways in which these methods can be modified to overcome their shortcomings or justify changes that have already been proposed. 

Finally, while we have proved lower bounds for greedy axis-aligned regression trees, it is natural to ask whether lower bounds of a similar flavor hold for oblique regression trees as well as other more general classes of regression trees and ensembles under computational constraints.
We hope our work will serve as a first step in clarifying the statistical-computational trade-offs inherent within the class of greedy algorithms fitting regression trees and ensembles.

\subsection*{Acknowledgments}

The authors would like to thank Yingying Fan, Samory Kpotufe, Wei-Yin Loh, Rahul Mazumder, and Denny Wu for insightful discussions. Tan gratefully acknowledges support from NUS Start-up Grant A-8000448-00-00 and MOE AcRF Tier 1 Grant A-8002498-00-00. Klusowski gratefully acknowledges financial support from the National Science Foundation (NSF) through CAREER DMS-2239448. Balasubramanian gratefully acknowledges support from NSF through grants DMS-2053918 and DMS-2413426.

\clearpage
\bibliography{00_refs}

\appendix
\appendix
\counterwithin{equation}{section}
\counterwithin{theorem}{section}

\section{Consistency of ERM trees, Greedy Trees, and Non-Adaptive Trees}\label{sec:additionalrelwork}

Machine learning models are often fit by solving \emph{empirical risk minimization (ERM)}.
Solutions to ERM over appropriate classes of decision tree functions can be shown to be consistent, but unfortunately are computationally infeasible to obtain.
\comment{red}{For instance, ERM for classification trees is known to be \emph{NP-complete} in the worst case~\citep{laurent1976constructing}  when the attribute count $d$ is part of the input. Their reduction depends only on the choice of split variables and tree topology, so the same hardness applies verbatim to the squared‐error regression objective studied in this work. If $d$ is fixed, the input space has only $2^{d}$ patterns, so one can enumerate the $\le d^{2^{s}}$ depth-$s$ trees and evaluate each in $O(n)$ time, yielding a (large-constant) polynomial algorithm; hardness disappears in this constant-$d$ corner case~\citep{donoho1997cart,chatterjee2021adaptive}. Whereas, our focus in this work is on the high-dimensional regime ($d\to\infty$, $\log n/d\to0$) where exhaustive ERM needs $O(n d^{2^{s}})$ operations while CART fits a depth-$s$ tree in $O(nd)$.}

Practitioners make use of \emph{greedy algorithms} such as CART ~\citep{breiman1984classification}, ID3~\citep{quinlan1986induction}, or C4.5~\citep{quinlan1993c4}.
Such algorithms grow a tree in a top-down manner, with the split for a given node chosen by optimizing a local node-specific objective.
Once made, a split is never reversed.\footnote{Decision trees are sometimes pruned as a one-off post-processing procedure after they are fully grown.}
These heuristics seem to work well in practice---both RFs and gradient boosting are built out of CART trees---but are notoriously difficult to analyze.
Indeed, early efforts to prove consistency of greedy regression and classification trees under standard smoothness conditions on $f^*$ had to further assume that the bias of the tree converged to zero~\citep{breiman1984classification,chaudhuri1995generalized}.
More recent works were able to replace this assumption with more reasonable alternatives on $f^*$ such as additivity ~\citep{scornet2015,klusowski2020sparse,klusowski2024large} or a submodularity condition \citep{syrgkanis2020estimation}.
\cite{chi2022asymptotic} further introduced and derived consistency under the sufficient impurity decrease (SID) condition \eqref{eq:SID}.
\cite{mazumder2024convergence} later improved upon the consistency rates of \cite{chi2022asymptotic} and showed that the SID condition is satisfied for larger classes of models.

\comment{red}{In another line of work, researchers have studied decision trees whose splits are \emph{non-adaptive}, meaning that the partitioning is chosen independently of the response $Y$. Examples include purely random-split trees~\citep{biau2008consistency,biau2012analysis,klusowski2021sharp}, Mondrian trees which split according to a Bayesian model of the covariate distribution~\citep{roy2008mondrian,lakshminarayanan2014mondrian,mourtada2017universal,cattaneo2023inference}, and $k$-Potential Nearest Neighbors~\citep{lin2006random,biau2010layered,scornet2016random,biau2016random,shi2024multivariate}. These methods are valuable as theoretical baselines but are generally not used in practice. In particular, purely random forests fall in this class, whereas methods such as honest trees remain adaptive in their split selection and thus do not. Despite having consistency guarantees under fewer assumptions, non-adaptive trees typically underperform greedy, label-driven procedures in terms of adaptivity to heterogeneous truth.}



\section{Lower Bounds for Non-Adaptive Trees}
\label{sec:non_adaptive_trees}

We say that a regression tree algorithm $\alg$ is \emph{non-adaptive} if $\alg_t$ does not depend on the response data $Y_1,Y_2,\ldots,Y_\nsamples$ (see Section \ref{sec:framework}).
We call the output of such an algorithm a non-adaptive tree model, denoting it using $\fna(-;\data,\Theta)$.
This class of tree models include \emph{completely random trees}, which are grown iteratively, similar to CART, except that at each iteration, the split feature for a new cell $\cell$ is chosen uniformly at random from $\coordindices\backslash\cellfeatindices(\cell)$.
This can also be generalized to incorporate non-uniform split probabilities.
Another subclass comprises \emph{Mondrian trees} \citep{roy2008mondrian}, which are defined via a recursive generative process that selects splits according to the length of each side of a cell and uses a stochastic stopping condition.
These tree models are more appropriately defined over continuous covariates.
As is the case with greedy trees, non-adaptive trees can be combined in ensembles, with the resulting estimators described as some type of random forests \citep{lakshminarayanan2014mondrian,biau2012analysis}.
We will denote an ensemble of non-adaptive tree models using $\fnae(-;\data,\Theta)$.

While most standard non-adaptive tree algorithms have some form of symmetry in terms of how they handle different covariates, this is not a necessary component of the definition.
In order to derive estimation error lower bounds for non-adaptive trees, we therefore first augment Assumption \ref{assumption} to further assume that the relevant covariate set $\suppset$ is generated uniformly at random.

\begin{assumption}
    \label{assumption_nonadaptive}
    Assume a binary covariate space $\xspace = \binspace$ with $\xmeasure$ the uniform measure.
    Let $f^*_0\colon \lbrace\pm 1\rbrace^\sparsity \to \R$ be any function.
    Let $\suppset \subset \coordindices$ be a subset of size $\abs*{\suppset} = \sparsity$ that is drawn uniformly at random.
    We consider the nonparametric model \eqref{eq:nonparametric_regression_model} with the regression function $f^*\colon \binspace \to \R$ defined via $f^*(\bx) = f^*_0(\bx^{\suppset})$.
\end{assumption}

We also slightly modify our definition of expected $L^2$ risk to also include an expectation with respect to $\suppset$.
\begin{equation}
    \objective\paren*{\hat f, f^*_0,\nfeats,\nsamples} \coloneqq \E_{\suppset,\data,\Theta}\braces*{R\paren*{\hat f(-;\data,\Theta),f^*}}.
\end{equation}
Note that we will keep these modified definitions only for this section of the paper.

Our proof strategy for obtaining lower bounds for the expected $L^2$ risk will be similar to that use for greedy trees and ensembles.
To this end, observe that Lemma \ref{lem:covariate_selection_tree} holds for any regression tree algorithm.
Meanwhile, we have the following modified versions of Lemma \ref{lem:covariate_selection_ensemble} and Lemma \ref{lem:xor_feature_selection_prob}.

\begin{lemma}
\label{lem:covariate_selection_ensemble_nonadaptive}
    Let $\hat f$ be a regression tree ensemble model in which each tree is fit using a regression tree algorithm $\alg$.
    Then under Assumption \ref{assumption_nonadaptive}, its expected $L^2$ risk satisfies $\objective\paren*{\hat f, f^*_0,\nfeats,\nsamples} \geq \paren*{1-\kappa\delta^{1/2}}\Var\braces*{f^*_0(\bX)}$, where $\delta \coloneqq \P\braces*{J(\bX;\data,\Theta)\cap \suppset \neq \emptyset }$ and $\kappa \coloneqq 2\|f_0^*\|_{\infty}/\Var\braces*{f^*_0(\bX)}^{1/2}$.
\end{lemma}

\begin{proof}
    We may repeat most of the steps in Lemma \ref{lem:covariate_selection_ensemble} except that $h(\bx)$ is now also a function of $\suppset$, and while we now have $\E_{\suppset,\bX}\braces*{\abs*{h(\bX)}} \leq \norm*{f_0^*}_\infty \delta$ because of non-adaptivity, we do not have $\abs*{h(\bx)} \leq M\delta$.
    Furthermore, we calculate
    \begin{equation}
        \begin{split}
            \E_{\bX,\suppset}\braces*{\Var_{\bX}\braces*{h(\bX)~|~\bX_{\coordindices\backslash\suppset}=\bx_{\coordindices\backslash\suppset}}^{1/2}} & \leq \E_{\bX,\suppset}\braces*{h(\bX)^2}^{1/2} \\
            & \leq \|f_0^*\|_{\infty} \P_{\bX,\data,\Theta,\suppset}\braces*{J(\bX;\data,\Theta)\cap \suppset \neq \emptyset}^{1/2} \\
            & = \|f_0^*\|_{\infty}\delta^{1/2}.
        \end{split}
    \end{equation}
    Combining this with \eqref{eq:query_path_helper4} completes the proof.
\end{proof}

\begin{lemma}
    \label{lem:na_tree_query_path}
    Suppose $\alg_t$ is non-adaptive.
    Then under Assumption \ref{assumption_nonadaptive}, the probability of covariate selection is upper bounded as
    $\P\braces*{J(\bX;\data,\Theta)\cap \suppset \neq \emptyset } \leq (\sparsity/\nfeats)\log_2\nsamples$.
\end{lemma}

\begin{proof}
    We first use the tower property of conditional expectations to write:
    \begin{equation}
        \begin{split}
            \P\braces*{J(\bX;\data,\Theta)\cap \suppset \neq \emptyset } & = \E\braces*{\P\braces*{J(\bX;\data,\Theta)\cap \suppset \neq \emptyset ~|~\bX,\data,\Theta}}.
        \end{split}
    \end{equation}
    Next, since $\alg_t$ is non-adaptive, the tree structure $\tree = \alg_t(\data,\Theta)$ is independent of $\suppset$.
    This means that conditioned on $\bX,\data,\Theta$, $\abs*{J(\bX;\data,\Theta)\cap \suppset}$ is a hypergeometric random variable with parameters $\nfeats$, $\abs*{J(\bX;\data,\Theta)}$ and $\sparsity$.
    Using a union bound, we thus get
    \begin{equation}
        \P\braces*{J(\bX;\data,\Theta)\cap \suppset \neq \emptyset ~|~\bX,\data,\Theta} \leq \frac{\sparsity \abs*{J(\bX;\data,\Theta)}}{\nfeats}.
    \end{equation}
    Let $\leaf_1,\leaf_2,\ldots,\leaf_K$ be the leaves in $\tree$ (still conditioning on $\data$ and $\Theta$).
    Notice that $\abs*{J(\bX;\data,\Theta)}$ is the length of the query path of $\bX$ and is equal to $-\log_2 \xmeasure(\leaf_k)$ if $\bx \in \leaf_k$. 
    We therefore have
    \begin{equation}
        \E\braces*{\abs*{J(\bX;\data,\Theta)}~|~\data,\Theta} = -\sum_{k=1}^K \xmeasure(\leaf_k) \log_2(\xmeasure(\leaf_k)),
    \end{equation}
    which is the entropy of the finite measure with weights $\xmeasure(\leaf_1),\xmeasure(\leaf_2),\ldots,\xmeasure(\leaf_K)$, and hence bounded by $\log_2 K$.
    Finally, we note that there can be at most $\nsamples$ leaves in $\tree$, i.e., $K \leq \nsamples$.
\end{proof}

\begin{proposition}
    \label{prop:na_lower_bound}
    Under Assumption \ref{assumption_nonadaptive}, for any $0 < \delta < 1$,
    suppose $\nsamples \leq 2^{\delta\nfeats/\sparsity}$.
    Then the expected $L^2$ risk of any non-adaptive tree model satisfies
    \begin{equation}
        \label{eq:na_tree_lower_bound}
        \objective\paren*{\fna, f^*_0,\nfeats,\nsamples} \geq (1-\delta)\Var\braces*{f^*(\bx)}.
    \end{equation}
    If furthermore, $\|f^*_0\|_{\infty} \leq \ybound$, then the $L^2$ risk for any non-adaptive tree ensemble model satisfies
    \begin{equation}
        \label{eq:na_ensemble_lower_bound}
        \objective\paren*{\fnae, f^*_0,\nfeats,\nsamples} \geq \paren*{1-\kappa\delta^{1/2}}\Var\braces*{f^*(\bx)},
    \end{equation}
    where $\kappa \coloneqq 2\ybound/\Var\braces*{f_0^*(\bX)}$.
    In particular, $\fna$ and $\fnae$ are high-dimensional inconsistent for $f^*_0$.
\end{proposition}

\begin{proof}
    Equation \eqref{eq:na_tree_lower_bound} follows immediately from Lemma \ref{lem:covariate_selection_tree} and Lemma \ref{lem:na_tree_query_path}.
    Equation \eqref{eq:na_ensemble_lower_bound} follows similarly using Lemma \ref{lem:covariate_selection_ensemble_nonadaptive} and Lemma \ref{lem:na_tree_query_path}.
\end{proof}

\section{Length of Query Paths}
\label{sec:query_path_length}

Let $\querypath(\bx;\data,\Theta)$ denote the query path of a tree fitted using a greedy tree algorithm to a dataset $\data = \braces*{(\bX_i,Y_i) \colon i = 1,2,\ldots,\nsamples}$ with a uniform distribution on the covariate space $\binspace$.
Assume also that the responses $Y_1,Y_2,\ldots,Y_\nsamples$ are completely independent of the covariates.
We let $\cell_0, \cell_1,\ldots,\cell_l$ denote the cells in $\querypath(\bx;\data,\Theta)$, and further denote $n_t \coloneqq N(\cell_t)$ for $t=0,1,\ldots,l$.
Observe also that $l = \abs*{J(\bx;\data,\Theta)}$.

\begin{lemma}
    \label{lem:conditional_expectation_for_node_no_of_samples}
    For any nonnegative integer $t$, we have $\E\braces*{n_{t+1}|n_{t}} = \frac{n_{t}}{2}$.
\end{lemma}

\begin{proof}
    Recall that the randomness is generated purely from $\data$.
    Let $\mathcal{F}_t$ be the sigma algebra generated by all observed split criteria up to time $t$, i.e., $\mathcal{O}(k;\cell_j,\data)$, for $k \in \cellfeatindices(\cell_j)$, and $j=0,1,\ldots,t$.
    Conditioned on $\mathcal{F}_t$, the split variable $k(\cell_t)$ is a deterministic quantity.
    Now 
    condition on $n_t,\mathcal{F}_t$ and consider the conditional distribution
    $n_{t+1}~|~n_t,\mathcal{F}_t$.
    Let $\iota$ be a map acting on the space of datasets by setting $X_{i,k(\cell_t)} \gets -X_{i,k(\cell_t)}$ for $i=1,2,\ldots,\nsamples$.
    Note that this a measure-preserving transformation.
    In addition, since the nodes $\cell_0,\cell_1,\ldots,\cell_t$ and impurity decreases up to time $t$ are invariant to the sign flip, the subspace of datasets satisfying the conditioned values of $n_t,\mathcal{F}_t$ is an invariant set under the transformation.
    This implies that $\E\braces*{n_{t+1}~|~n_t,\mathcal{F}_t} = \E\braces*{n_{t+1}\circ\iota~|~n_{t},\mathcal{F}_t}$, or
    \begin{equation} \label{eq:number_of_samples_conditional_expectation}
        \E\braces*{n_{t+1}~|~n_{t},\mathcal{F}_t} = \E\braces*{\frac{n_{t+1} + n_{t+1}\circ\iota}{2}~\vline~n_{t},\mathcal{F}_t}.
    \end{equation}
    Finally, observe that $n_{t+1}$ and $n_{t+1}\circ\iota$ are simply the numbers of samples in the right and left children of $\cell_t$ when splitting on $x_{k(\cell_t)}$, which gives $n_{t+1} + n_{t+1}\circ \iota = n_{t}$.
    Plugging this into \eqref{eq:number_of_samples_conditional_expectation} and using the tower property of conditional expectations completes the proof.
\end{proof}

\begin{lemma}
    \label{lem:bound_for_tree_depth}
    For any nonnegative integer $t$, we have the tail bound
    \begin{equation}
    \label{eq:tree_depth_probability}
        \P\braces*{\abs*{J(\bx;\data,\Theta)} \geq t} \leq \min\braces*{\frac{n}{2^{t}},1}.
    \end{equation}
    Furthermore, we have the expectation bound
    \begin{equation}
    \label{eq:tree_depth_expectation}
        \E\braces*{\abs*{J(\bx;\data,\Theta)}} \leq \log_2 \nsamples + 2.
    \end{equation}
\end{lemma}

\begin{proof}
    Using Lemma \ref{lem:conditional_expectation_for_node_no_of_samples}, the fact that $n_0 = n$, and the tower property, we get $\E\braces*{n_t} = \frac{n}{2^{t}}$.
    Using Markov's inequality therefore gives
    \begin{equation}
        \P\braces*{n_t \geq 1} \leq \frac{n}{2^{t}}.
    \end{equation}
    Next, note that $\abs*{J(\bx;\data,\Theta)} \geq t$ implies that $n_t \geq 1$, which gives \eqref{eq:tree_depth_probability}.
    To obtain \eqref{eq:tree_depth_expectation}, we compute
    \begin{equation}
        \begin{split}
            \E\braces*{\abs*{J(\bx;\data,\Theta)}} & = \sum_{t=1}^\infty\P\braces*{\abs*{J(\bx;\data,\Theta)} \geq t} \\
            & \leq \sum_{t=1}^{\lfloor \log_2 \nsamples \rfloor} 1 + \sum_{t=\lfloor \log_2 \nsamples \rfloor + 1}^\infty \frac{\nsamples}{2^{t}} \\
            & \leq \lfloor \log_2 \nsamples \rfloor + \sum_{t=0}^\infty 2^{-t} \\
            & \leq  \log_2 \nsamples  + 2. \qedhere
        \end{split}
    \end{equation}
\end{proof}

\section{Proof of Theorem \ref{thm:non_msp_lower_bound}}
\label{sec:non_msp_lower_bound_proof}

The proof is similar to that of Proposition \ref{prop:xor_lower_bound} (the lower bound for the XOR function), but we need to adapt the ingredients (Lemma \ref{lem:covariate_selection_tree}, Lemma \ref{lem:covariate_selection_ensemble}, and Lemma \ref{lem:xor_feature_selection_prob}) to this more complicated setting.
These generalizations are provided in the rest of this section.
For the first statement, we make use of Lemma \ref{lem:covariate_selection_tree_v2} and bound $\delta$ using Lemma \ref{lem:nonmsp_feature_selection_prob}.
For the second statement, we make use of Lemma \ref{lem:covariate_selection_ensemble_v2} and bound $\delta$ using Lemma \ref{lem:nonmsp_feature_selection_prob}, bearing in mind the constraint on $\delta$.
\qed

\begin{lemma}
    \label{lem:nonmsp_feature_selection_prob}
    Under Assumption \ref{assumption}, let $f_0^*$ be a function not satisfying \MSP, and suppose $\alg_t$ is greedy.
    Let $\traversal \subset \coordindices\backslash\suppset_{\MSP}$ be a traversal for $\fcoefs\backslash\fcoefs_{\MSP}$ of size $\traversalsize$ and let $\sparsity_{\MSP} \coloneqq \abs*{\suppset_{\MSP}}$.
    Then for each fixed query point $\bx \in \binspace$, support set $\suppset$, and random seed $\Theta$, the probability of selecting any covariate in the traversal along the query path $\querypath(\bx;\data,\Theta)$ is upper bounded as $\P_{\data}\braces*{J(\bx;\data,\Theta)\cap \traversal \neq \emptyset } \leq \frac{\traversalsize(\log_2 \nsamples + 2)}{\nfeats - \sparsity_{\MSP}-\traversalsize +1}$.
\end{lemma}

\begin{proof}
    Without loss of generality, assume $\suppset_{\MSP} = \braces*{1,2,\ldots,\sparsity_{\MSP}}$ and
    \begin{equation}
        \traversal = \braces*{\sparsity_{\MSP}+1,\sparsity_{\MSP}+1,\ldots,\sparsity_{\MSP}+\traversalsize}.
    \end{equation}
    \textit{Step 1: Defining the path coupling.}
    For $a=1,2,\ldots,\traversalsize$, define the subset $\traversal^{(a)} \coloneqq \traversal\backslash\braces*{\sparsity_{\MSP}+a}$ and the modified dataset $\data^{(a)} \coloneqq \braces*{(\bX_{i,-\traversal^{(a)}},Y_i)}_{i=1}^\nsamples$, i.e., so that all covariates with indices in $\traversal^{(a)}$ are dropped from each observation.
    We will analyze the query paths, $\querypath(\bx;\data^{(a)},\Theta)$ obtained from these modified datasets and compare them with the original.

    \textit{Step 2: Coupled paths select covariates randomly.}
    For $a=1,2,\ldots,\traversalsize$, we claim that each response $Y_i$ is independent of the covariate vector $\bX_{i,-\traversal^{(a)}\cup\suppset_{\MSP}}$.
    To see this, consider $S \in \fcoefs$.
    If $S \in \fcoefs_{\MSP}$, then $\monomial_S$ does not depend on $\bX_{i,-\traversal^{(a)}\cup\suppset_{\MSP}}$ and is hence trivially independent.
    Otherwise, $S \in \fcoefs\backslash\fcoefs_{\MSP}$, and by the definition of a traversal, we have
    \begin{equation}
        \abs*{S\cap \traversal^{(a)}} \geq \abs*{S\cap\traversal} - 1 \geq 1.
    \end{equation}
    As such, for any fixed value of $\bx_{i,-\traversal^{(a)}\cup\suppset_{\MSP}}$, we have
    \begin{equation}
        \P\braces*{\monomial_S = 1 ~\vline~ \bX_{i,-\traversal^{(a)}\cup\suppset_{\MSP}} = \bx_{i,-\traversal^{(a)}\cup\suppset_{\MSP}}} = \frac{1}{2}.
    \end{equation}
    Combining this with the fact that the covariates are drawn uniformly from $\braces{\pm 1}^{\nfeats-1}$ as well as the symmetry of the splitting criterion, we see that the distribution of $\data^{(a)}$ is invariant to permutation of the covariate indices in $\coordindices\backslash\paren*{\traversal^{(a)}\cup\suppset_{\MSP}}$.
    As such, $J(\bx;\data^{(a)},\Theta)\backslash\suppset_{\MSP}$ is the prefix of a uniformly random permutation of $\coordindices\backslash\paren*{\traversal^{(a)}\cup\suppset_{\MSP}}$.

    \textit{Step 3: Bounding covariate selection in coupled paths.}
    For any $b \in \coordindices\backslash\paren*{\traversal^{(a)}\cup\suppset_{\MSP}}$, we use the previous step to compute
    \begin{equation}
        P\braces*{b \in J(\bx;\data^{(a)},\Theta)~\vline~\abs{J(\bx;\data^{(a)},\Theta)}} \leq \frac{\abs{J(\bx;\data^{(a)},\Theta)}}{\nfeats - \sparsity_{\MSP}-\traversalsize}.
    \end{equation}
    Taking a further expectation and using Lemma \ref{lem:bound_for_tree_depth} gives
    \begin{equation}
    \begin{split}
        \P\braces*{b \in J(\bx;\data^{(a)},\Theta)} & = \E\braces*{\P\braces*{b \in J(\bx;\data^{(a)},\Theta)~\vline~\abs{J(\bx;\data^{(a)},\Theta)}}} \\
        & \leq \frac{\log_2 \nsamples + 2}{\nfeats - \sparsity_{\MSP}-\traversalsize}.
    \end{split}
    \end{equation}

    \textit{Step 4: Decoupling implies covariate selection.}
    We claim the following inclusion of events:
    \begin{equation}
    \label{eq:non_msp_decoupling_covariate_selection}
        \bigcup_{a=1}^{\traversalsize}\braces*{J(\bx;\data,\Theta) \neq J(\bx;\data^{(a)},\Theta)} \subset \bigcup_{a=1}^{\traversalsize}\braces*{\sparsity_{\MSP} + a \in J(\bx;\data^{(a)},\Theta)}.
    \end{equation}
    To prove this, assuming the left hand side, let $t$ be the smallest index for which $k_t(\bx;\data^{(a)},\Theta) \neq k_t(\bx;\data,\Theta)$ for some $a$, which we assume without loss of generality is $a=1$.
    This implies in particular that the depth $t$ node along all $\traversalsize$ query paths are equal.
    Label this node as $\cell_t$.
    We now examine the possible values for $k_t(\bx;\data,\Theta)$.
    For $j \in \coordindices\backslash\paren*{\cellfeatindices(\cell_t)\cup\traversal^{(1)}}$, we have
    \begin{equation}
        \mathcal{O}(j,\cell_t,\data) = \mathcal{F}\paren*{\braces*{(X_{ij},Y_i)}_{i \in I(\cell)}} = \mathcal{O}(j,\cell_t,\data^{(1)}).
    \end{equation}
    Recall that $k_t(\bx;\data,\Theta)$ is defined to be the maximum of the left hand side over $j \notin \cellfeatindices(\cell_t)$.
    Meanwhile, $k_t(\bx;\data^{(1)},\Theta)$ is defined to be the maximum of the right hand side over $j \notin \cellfeatindices(\cell_t)\cup\traversal^{(1)}$.
    As such, we see that the only way in which these can be different values is if $b\coloneqq k_t(\bx;\data,\Theta) \in \traversal^{(1)}$.
    This would imply, however, that $k_t(\bx;\data^{(b-\sparsity_{\MSP})},\Theta) = b$ as we wanted.

    \textit{Step 5: Completing the proof.}
    Suppose $J(\bx;\data,\Theta)\cap \traversal \neq \emptyset$.
    If coupling does not hold (i.e., the left hand side of \eqref{eq:non_msp_decoupling_covariate_selection} is true), then by that same equation, we see that $\sparsity_{\MSP} + a \in J(\bx;\data^{(a)},\Theta)$ for some $a=1,2,\ldots,\traversalsize$.
    On the other hand if coupling does hold, then picking $\sparsity_{\MSP} + a \in J(\bx;\data,\Theta)\cap \traversal$, we have $\sparsity_{\MSP} + a \in J(\bx;\data,\Theta) = J(\bx;\data^{(a)},\Theta)$.
    In conclusion, we have
    \begin{equation}
    \begin{split}
        \P_{\data}\braces*{J(\bx;\data,\Theta)\cap \traversal \neq \emptyset } & \leq \sum_{a=1}^{\traversalsize} \P\braces*{\sparsity_{\MSP} + a \in J(\bx;\data^{(a)},\Theta)} \\
        & \leq \frac{\traversalsize(\log_2 \nsamples + 2)}{\nfeats - \sparsity_{\MSP}-\traversalsize+1}. \qedhere
    \end{split}
    \end{equation}
\end{proof}

\begin{lemma}
\label{lem:covariate_selection_tree_v2}
    Under Assumption \ref{assumption}, let $f_0^*=\sum_{S \in \fcoefs} \alpha_S\monomial_S$ be any function. 
    Let $\hat f$ be a regression tree model fit using a regression tree algorithm $\alg$.
    Let $\fcoefs' \subset \fcoefs$ be a subcollection and let $\traversal$ be any traversal for $\fcoefs'$ (i.e., $\abs{T\cap S} \geq 2$ for each $S \in \fcoefs')$.
    Then its expected $L^2$ risk satisfies $\objective\paren*{\hat f, f^*_0,\nfeats,\nsamples} \geq (1-\delta)\sum_{S \in \fcoefs'} \alpha_S^2$, where $\delta \coloneqq \max_{\bx}\P\braces*{J(\bx;\data,\Theta)\cap \suppset \neq \emptyset }$.
\end{lemma}

\begin{proof}
    We start with the trivial lower bound
    \begin{equation} \label{eq:query_path_helper1_v2}
        R\paren*{\hat f(-;\data,\Theta),f^*} \geq \E_{\bX\sim\xmeasure}\braces*{\paren*{\hat f(\bX;\data,\Theta) - f^*(\bX)}^2\indicator\braces*{J(\bX;\data,\Theta)\cap \traversal = \emptyset }}.
    \end{equation}
    The collection $\braces*{J(\bx;\data,\Theta)\cap \traversal = \emptyset }_{\bx \in \binspace}$ can be thought of as a collection of random events indexed by $\bx \in \binspace$.
    Whenever one of these event holds for some fixed value $\bx$, $\hat f(\bx;\data,\Theta)$ is constant with respect to changes in $\bx_{\traversal}$.
    Utilizing this fact together with the uniform distribution of $\bX$ allows us to decompose the conditional expectation of the squared loss as
    \begin{equation} \label{eq:query_path_helper2_v2}
    \begin{split}
        & \E_{\bX}\braces*{\paren*{\hat f(\bX;\data,\Theta) - f^*(\bX)}^2~
            \vline~\bX_{\coordindices\backslash\traversal} = \bx_{\coordindices\backslash\traversal}} \\
            &= \Var\braces*{f_0^*(\bX)~|~\bX_{\coordindices\backslash\traversal} = \bx_{\coordindices\backslash\traversal}} + \paren*{\hat f(\bx;\data,\Theta) - \E\braces*{f_0^*(\bX)~|~\bX_{\coordindices\backslash\traversal} = \bx_{\coordindices\backslash\traversal}}}^2 \\
            &\geq \Var\braces*{f_0^*(\bX)~|~\bX_{\coordindices\backslash\traversal} = \bx_{\coordindices\backslash\traversal}}.
    \end{split}
    \end{equation}
    Making the further observation that the function $\indicator\braces*{J(\bx;\data,\Theta)\cap \traversal = \emptyset }$ is constant with respect to $\bx_{\traversal}$ (i.e., is in the $\sigma$-algebra generated by $\bX_{\coordindices\backslash\traversal}$), 
    we can use conditional expectations to rewrite the right hand side of \eqref{eq:query_path_helper1_v2} as
    \begin{equation}
        \E_{\bX\sim\xmeasure}\braces*{\E_{\bX}\braces*{\paren*{\hat f(\bX;\data,\Theta) - f^*(\bX)}^2~
            \vline~\bX_{\coordindices\backslash\traversal}}\indicator\braces*{J(\bX;\data,\Theta)\cap \traversal = \emptyset }}.
    \end{equation}
    Plugging in the bound \eqref{eq:query_path_helper2_v2}, taking a further expectation with respect to $\data$ and $\Theta$ shows that
    \begin{equation}
        \objective\paren*{\hat f, f^*_0,\nfeats,\nsamples} \geq (1-\delta)\E\braces*{\Var\braces*{f^*(\bX)~|~\bX_{\traversal}}}.
    \end{equation}
    Finally, note that for any $S \in \fcoefs'$, since $\abs*{S\cap \traversal} \geq 1$, we have 
    \begin{equation}
        \E\braces*{\monomial_S~|~\bX_{\coordindices\backslash\traversal} = \bx_{\coordindices\backslash\traversal}} = 0.
    \end{equation}
    Using the law of total variance, we therefore have
    \begin{equation}
        \begin{split}
            \E\braces*{\Var\braces*{f^*(\bX)~|~\bX_{\traversal}}} & = \Var\braces*{f^*(\bX)} - \Var\braces*{\E\braces*{f^*(\bX)~|~\bX_{\coordindices\backslash\traversal}}} \\
            & = \Var\braces*{f^*(\bX)} - \Var\braces*{\sum_{S \in \fcoefs, S\cap \traversal = \emptyset} \alpha_S\monomial_S} \\
            & \geq \sum_{S \in \fcoefs'} \alpha_S^2. \qedhere
        \end{split}
    \end{equation}

\end{proof}

\begin{lemma}
\label{lem:covariate_selection_ensemble_v2}
    Under Assumption \ref{assumption}, let $f_0^*=\sum_{S \in \fcoefs} \alpha_S\monomial_S$ be any function. 
    Let $\hat f$ be a regression tree ensemble model fit using a regression tree algorithm $\alg$.
    Let $\fcoefs' \subset \fcoefs$ be a subcollection and let $\traversal$ be any traversal for $\fcoefs'$ (i.e., $\abs{T\cap S} \geq 2$ for each $S \in \fcoefs')$.
    Then its expected $L^2$ risk satisfies $\objective\paren*{\hat f, f^*_0,\nfeats,\nsamples} \geq (1-\kappa\delta)\sum_{S \in \fcoefs'} \alpha_S^2$, where $\delta \coloneqq \max_{\bx}\P\braces*{J(\bx;\data,\Theta)\cap \suppset \neq \emptyset }$ and $\kappa \coloneqq 2\ybound/\paren*{\sum_{S \in \fcoefs'} \alpha_S^2}^{1/2}$.
\end{lemma}

\begin{proof}
    We first write
    \begin{equation}
        \hat f(\bx;\data,\Theta) = \frac{1}{B}\sum_{b=1}^B \tilde f(\bx;\data,\theta_b),
    \end{equation}
    where $\tilde f$ is the base learner of the ensemble.
    Taking expectations with respect to $\data$ and $\Theta$ gives the chain of equalities:
    \begin{equation}
        \begin{split}
            \E_{\data,\Theta}\braces*{\frac{1}{B}\sum_{b=1}^B \tilde f(\bx;\data,\theta_b)}
            & = \E_{\data,\theta_1}\braces*{\tilde f(\bx;\data,\theta_1)} \\
            & = \underbrace{\E_{\data,\theta_1}\braces*{\tilde f(\bx;\data,\theta_1)\indicator\braces*{J(\bx;\data,\theta_1)\cap \traversal = \emptyset }}}_{g(\bx)} \\
            & \quad + \underbrace{\E_{\data,\theta_1}\braces*{\tilde f(\bx;\data,\theta_1)\indicator\braces*{J(\bx;\data,\theta_1)\cap \traversal \neq \emptyset }}}_{h(\bx)}.
        \end{split}
    \end{equation}
    For ease of notation, we label the two quantities on the right hand side as $g(\bx)$ and $h(\bx)$ respectively.
    Note also that $\abs*{h(\bx)} \leq M$.
    Using Jensen's inequality followed by the tower property of conditional expectations, we may then compute 
    \begin{equation}
    \label{eq:query_path_helper3_v2}
    \begin{split}
        & \E_{\bX,\Theta,\data}\braces*{\paren*{\hat f(\bX;\data,\Theta) - f^*(\bX)}^2} \\
        & \geq \E_\bX\braces*{\paren*{g(\bX) + h(\bX) - f^*(\bX)}^2} \\
        & = \E_\bX\braces*{\E_{\bX}\braces*{\paren*{g(\bX) + h(\bX) - f^*(\bX)}^2~\vline~ \bX_{\coordindices\backslash\traversal}}}.
    \end{split}
    \end{equation}
    We next notice that $g(\bx)$ is constant with respect to changes in $\bx_{\traversal}$.
    This means that, similar to \eqref{eq:query_path_helper2_v2}, the conditional expectation within \eqref{eq:query_path_helper3_v2} can be lower bounded by a variance term, which we expand as follows:
    \begin{equation}
    \label{eq:query_path_helper4_v2}
        \begin{split}
            & \E_{\bX}\braces*{\paren*{g(\bX) + h(\bX) - f^*(\bX)}^2~\vline~ \bX_{\coordindices\backslash\traversal}=\bx_{\coordindices\backslash\traversal}} \\
            & \geq \Var_{\bX}\braces*{h(\bX) - f^*(\bX)~|~\bX_{\coordindices\backslash\traversal}=\bx_{\coordindices\backslash\traversal}} \\
            & \geq \Var\braces*{f^*(\bX)~|~\bX_{\coordindices\backslash\traversal} = \bx_{\coordindices\backslash\traversal}} \\
            &\quad\quad- 2\Var\braces*{f^*(\bX)~|~\bX_{\coordindices\backslash\traversal} = \bx_{\coordindices\backslash\traversal}}^{1/2}\Var_{\bX}\braces*{h(\bX)~|~\bX_{\coordindices\backslash\traversal}=\bx_{\coordindices\backslash\traversal}}^{1/2}.
        \end{split}
    \end{equation}
    Now notice that
    \begin{equation}
    \label{eq:query_path_helper5_v2}
    \begin{split}
         \E_{\bX}\braces*{\Var_{\bX}\braces*{h(\bX)~|~\bX_{\coordindices\backslash\traversal}}^{1/2}}
         & \leq \E_{\bX}\braces*{h(\bX)^2}^{1/2} \\
         & \leq \ybound\delta.
    \end{split}
    \end{equation}
    Plugging \eqref{eq:query_path_helper4_v2} into \eqref{eq:query_path_helper3_v2} and simplifying using \eqref{eq:query_path_helper5_v2} gives
    \begin{equation}
    \begin{split}
        \objective\paren*{\hat f, f^*_0,\nfeats,\nsamples} & \geq \E\braces*{\Var\braces*{f^*(\bX)~|~\bX_{\coordindices\backslash\traversal}}} - 2M\delta\E\braces*{\Var\braces*{f^*(\bX)~|~\bX_{\coordindices\backslash\traversal}}}^{1/2} \\
        & = \sum_{S \in \fcoefs, S\cap\traversal = \emptyset} \alpha_S^2 - 2M\delta\paren*{\sum_{S \in \fcoefs, S\cap\traversal = \emptyset} \alpha_S^2}^{1/2} \\
        & \geq \sum_{S \in \fcoefs'} \alpha_S^2 - 2M\delta\paren*{\sum_{S \in \fcoefs'} \alpha_S^2}^{1/2},
    \end{split}
    \end{equation}    
    where the last inequality holds so long as $\paren*{\sum_{S \in \fcoefs'} \alpha_S^2}^{1/2} \geq M\delta$.
\end{proof}

\section{Proof of Proposition \ref{prop:info_theory_lower_bound}}
\label{sec:info_theory_lower_bound}


    We first construct a hypothesis set as follows.
    Using Lemma \ref{lem:packing_number_for_dary_cube}, there exists a $1/2$-separated set of $[d-s]^{2^{s-1}}$ of size at least $\exp(2^{s-2}\log((d-s)/4)$, which we denote as $\mathcal{M}$.
    Now for each $\bz \in \mathcal{M}$, let $\tree_{\bz}$ be the tree structure comprising splits on covariate $d-s + i + 1$ for all nodes at depth $i$ for $i=0,1,\ldots,s-2$, and whose splits at level $s-1$ are determined by $\bz$.
    For every tree structure $\tree_{\bz}$, define a leaf parameter vector $\bmu$ so that for each leaf $\leaf_i$, the leaf mean satisfies
    \begin{equation}
        \mu_i = \begin{cases}
            1 & \text{if}~\leaf_i~\text{is a right child}, \\
            -1 & \text{if}~\leaf_i~\text{is a left child}.
        \end{cases}
    \end{equation}
    Denote the resulting regression tree function as $f_\bz$.
    Our hypothesis set will be
    \begin{equation}
        \mathcal{G}_\delta \coloneqq \braces*{\delta \cdot f_\bz \colon \bz \in \mathcal{M}},
    \end{equation}
    where $\delta > 0$ is a parameter to be determined later.

    We now show that $\mathcal{G}_\delta$ is $\delta$-separated with respect to $\norm{-}_2$.
    First, each $\delta \cdot f_{\bz} \in \mathcal{G}_\delta$ has zero mean, i.e.,
    \begin{equation}
        \E \braces*{\delta \cdot f_{\bz}(\bX)} = \frac{\delta}{2^s}\sum_{j=1}^{2^s} \mu_j = 0,
    \end{equation}
    since half the leaves will have even parities and half of them will have odd parities.
    Furthermore, since $\abs{f_{\bz}} = 1$, we have
    \begin{equation}
        \norm{\delta \cdot f_{\bz}}_2^2 = \delta^2.
    \end{equation}
    Next, fix $\bz, \bz' \in \mathcal{M}$ with $\bz \neq \bz'$.
    Let $\braces*{\leaf_i}$ and $\braces*{\leaf_i'}$ denote the leaves of $\tree_{\bz}$ and $\tree_{\bz'}$ respectively, with sibling leaves having consecutive numbers.
    For $i=1,2,\ldots,2^{s-1}$, define the function $\phi_i$ by setting 
    \begin{equation}
        \phi_i(\bx) = \indicator\braces*{\bx \in \leaf_{2i}} - \indicator\braces*{\bx \in \leaf_{2i-1}},
    \end{equation}
    observing that $f_\bz = \sum_{i=1}^{2^{s-1}} \phi_i$.
    Define $\phi_i'$ similarly.
    We then see that $\phi_i$ and $\phi_j$ have disjoint supports when $i \neq j$ and furthermore that 
    \begin{equation}
        \E\braces*{\phi_i(\bX)\phi_i'(\bX)} = \begin{cases}
            2^{-s+1} & \text{if}~z_i = z_i', \\
            0 & \text{otherwise}.
        \end{cases}
    \end{equation}
    As such, we have
    \begin{equation}
        \E\braces*{f_{\bz}(\bX)f_{\bz'}(\bX)} = 2^{-s+1}\sum_{i=1}^{2^{s-1}} \indicator\braces*{z_i = z_i'} = 1 - d_H(\bz,\bz'),
    \end{equation}
    and
    \begin{equation}
        \norm*{\delta \cdot f_{\bz} - \delta \cdot f_{\bz'}}_2^2 = 2\delta^2 d_H(\bz,\bz') \geq \delta^2.
    \end{equation}

    We now apply the Fano method (Proposition 15.12 in \cite{wainwright2019high}) to get
    \begin{equation}
        \inf_{\hat f}\max_{f \in \funcclass} \E\braces*{\norm*{\hat f - f}_2^2} \geq (\delta/2)^2 \paren*{1 - \frac{I(Z;J) + \log 2}{\log M}},
    \end{equation}
    where $J$ is uniform on $[M]$ and $Z|J \sim p_J$.
    Here, $M = \abs*{\mathcal{M}}$, while $p_J$ when $J=i$ is the distribution of $\by$ generated using the regression function $\delta \cdot f_{\bz_i}$.
    Note that we have
    \begin{equation}
        \begin{split}
            \KL(p_i\|p_{j}) & = \E\braces*{\sum_{i=1}^\nsamples \KL\paren*{p_i(y_i|\bX_i) \|p_j(y_i|\bX_i)}} \\
            & = \E_\bX\braces*{\frac{n\norm*{\delta \cdot f_{\bz_i} - \delta \cdot f_{\bz_j}}_n^2}{2\sigma^2}} \\
            & = \frac{n\delta^2 \norm*{ f_{\bz_i} - f_{\bz_j}}_2^2}{2\sigma^2} \\
            & \leq \frac{n\delta^2}{\sigma^2}.
        \end{split}
    \end{equation}
    By the comments after Proposition 15.12 in \cite{wainwright2019high}, we have
    \begin{equation}
        I(Z;J) \leq \frac{1}{M^2}\sum_{i,j=1}^M \KL(p_i\|p_{j}) \leq \frac{n\delta^2}{\sigma^2}.
    \end{equation}
    Choosing $\delta^2 = 2^{s-4}\sigma^2\log((d-s)/4)/n$, we get
    \begin{equation}
        \frac{I(Z;J) + \log 2}{\log M} \leq \frac{1}{4} + \frac{\log 2}{2^{s-2}\log((d-s)/4)} \leq \frac{1}{2},
    \end{equation}
    which finishes the proof. \qed

\begin{lemma}
    \label{lem:packing_number_for_dary_cube}
    Consider the set $[m]^r$ for some $m$ and $r$, equipped with the Hamming distance metric
    $$
    d_H(\bx,\bx') = \frac{1}{r}\sum_{i=1}^r \indicator\braces*{x_i \neq x_i'}.
    $$
    There is a $1/2$-separated set in $[m]^r$ of size at least $\exp\paren*{\frac{r}{2}\log (m/4)}$.
\end{lemma}

\begin{proof}
    We use a volumetric argument.
    For any $\delta > 0$, a maximal $\delta$-separated set is also a $\delta$-cover.
    The packing number \citep{wainwright2019high} of the set can therefore be lower bounded as
    \begin{equation}
        M(\delta;[m]^r,d_H) \geq \frac{\text{Vol}([m]^r)}{\text{Vol}(B_\delta)}
    \end{equation}
    where $B_\delta$ is a Hamming ball in $[m]^r$ of radius $\delta$.
    The volume of this ball satisfies
    \begin{equation}
        \text{Vol}(B_\delta) \leq \binom{r}{\delta r} m^{\delta r}.
    \end{equation}
    As such, for $\delta = 1/2$, and using the bound $\binom{r}{r/2} \leq 2^r$, we may continue the above to get
    \begin{equation}
        \begin{split}
            \log M(1/2;[m]^r,d_H) \geq r\log m - \frac{r}{2}\log m - r \log 2. \qedhere
        \end{split}
    \end{equation}
\end{proof}

\section{Proof of Proposition \ref{prop:erm_upper_bound}}
\label{sec:erm_proofs}

We will use the techniques from Chapters 13 and 14 in \cite{wainwright2019high}.
To this end, we first remind the reader of some standard definitions.
Let $\bW \coloneqq (W_1,W_2,\ldots,W_{\nsamples})$ be a sequence of i.i.d. standard normal random variables and let $\beps \coloneqq (\varepsilon_1,\varepsilon_2,\ldots,\varepsilon_\nsamples)$ be a sequence of i.i.d. Rademacher random variables.
We denote the empirical norm of a function $f$ via $\|f\|_\nsamples \coloneqq \paren*{\frac{1}{n}\sum_{i=1}^nf(\bX_i)}^{1/2}$ and its $L^2$ norm via $\|f\|_2 \coloneqq \E_{\bX\sim\xmeasure}\braces*{f(\bX)}^{1/2}$.
The \emph{local Gaussian complexity} of a function class $\funcclass$ is defined as
\begin{equation}
\label{eq:local_Gaussian_complexity}
    \mathcal{G}_n(\delta;\funcclass) \coloneqq \E_{\bW}\braces*{\sup_{\substack{f \in \funcclass \\
    \norm{f}_n \leq \delta}} \abs*{\frac{1}{n}\sum_{i=1}^n W_i f(\bX_i)}},
\end{equation}
while the \emph{local Rademacher complexity} of $\funcclass$ is defined as
\begin{equation}
\label{eq:local_Rademacher_complexity}
    \bar{\mathcal{R}}_n(\delta;\funcclass) \coloneqq \E_{\beps,\data}\braces*{\sup_{\substack{f \in \funcclass \\
    \norm{f}_2 \leq \delta}} \abs*{\frac{1}{n}\sum_{i=1}^n \varepsilon_i f(\bX_i)}}. 
\end{equation}
We note that $\mathcal{G}_n(\delta;\funcclass)$ is a random quantity (it still demands on $\data$) while $\bar{\mathcal{R}}_n(\delta;\funcclass)$ is fully deterministic.
These quantities are used to provide bounds for nonparametric least squares estimation problems, which includes ERM over tree-based functions \eqref{eq:ERM_equation}.
It will also be helpful to define notation for the integrands in \eqref{eq:local_Gaussian_complexity} and \eqref{eq:local_Rademacher_complexity} as follows:
\begin{equation}
    \mathcal{G}_n(\bW; \delta;\funcclass) \coloneqq \sup_{\substack{f \in \funcclass \\
    \norm{f}_n \leq \delta}} \abs*{\frac{1}{n}\sum_{i=1}^n W_i f(\bX_i)},
\end{equation}
\begin{equation}
    \bar{\mathcal{R}}_n(\beps;\delta;\funcclass) \coloneqq \E_{\data}\braces*{\sup_{\substack{f \in \funcclass \\
    \norm{f}_2 \leq \delta}} \abs*{\frac{1}{n}\sum_{i=1}^n \varepsilon_i f(\bX_i)}}. 
\end{equation}

\begin{lemma}
\label{lem:linear_subspace_complexity}
    Let $\funcclass$ comprise a $q$-dimensional subspace, i.e., there exists an orthonormal basis $\phi_1,\phi_2,\ldots,\phi_q$ such that
    \begin{equation}
        \funcclass \coloneqq \braces*{\sum_{j=1}^q \alpha_j \phi_j \colon \balpha \in \R^q}.
    \end{equation}
    Then we have $\mathcal{G}_n(\delta;\funcclass) \leq \delta(q/\nsamples)^{1/2}$ almost surely and $\bar{\mathcal{R}}_n(\delta;\funcclass) \leq \delta(q/\nsamples)^{1/2}$.
    Furthermore, we have the following tail bounds with respect to $\bW$ and $\beps$:
    \begin{equation}
    \label{eq:local_Gaussian_complexity_concentration}
        \P\braces*{\abs*{\mathcal{G}_n(\bW; \delta;\funcclass) - \mathcal{G}_n(\delta;\funcclass)} \geq t} \leq 2\exp\paren*{-\frac{nt^2}{2\delta^2}},
    \end{equation}
    \begin{equation}
    \label{eq:local_Rademacher_complexity_concentration}
        \P\braces*{\abs*{\bar{\mathcal{R}}_n(\beps;\delta;\funcclass) - \bar{\mathcal{R}}_n(\delta;\funcclass)} \geq t} \leq C_1\exp\paren*{-\frac{C_2nt^2}{2\delta}},
    \end{equation}
    where $C_1$ and $C_2$ are universal constants.
\end{lemma}

\begin{proof}
    Given $\data$, we see that the set
    \begin{equation}
        \braces*{\paren*{f(\bX_1),f(\bX_2),\ldots,f(\bX_n)} \colon f \in \funcclass, \norm{f}_\nsamples \leq \delta}
    \end{equation}
    is the unit ball in a $q$-dimensional linear subspace of $\R^\nsamples$, scaled by $\delta\sqrt{\nsamples}$.
    A standard calculation therefore gives
    \begin{equation}
        \mathcal{G}_n(\delta;\funcclass) \leq \delta\sqrt{\frac{q}{\nsamples}}.
    \end{equation}
    Similarly, if we set $\bPhi$ to be the $\nsamples \times q$ matrix whose $(i,j)$th entry is equal to $\phi_j(\bX_i)$, we have
    \begin{equation}
    \begin{split}
        \sup_{\substack{f \in \funcclass \\
    \norm{f}_2 \leq \delta}} \abs*{\frac{1}{n}\sum_{i=1}^n \varepsilon_i f(\bX_i)} & = \frac{1}{n}\sup_{\norm{\alpha}_2 \leq \delta}\abs*{\beps^T\bPhi\balpha} \\
    & \frac{\delta}{n}\norm*{\beps^T\bPhi}_2.
    \end{split}
    \end{equation}
    We therefore have
    \begin{equation}
        \begin{split}
            \bar{\mathcal{R}}_n(\delta;\funcclass) & = \E_{\data,\beps}\braces*{\frac{\delta}{n}\norm*{\beps^T\bPhi}_2} \\
            & \leq \frac{\delta}{n} \E_{\data,\beps}\braces*{\norm*{\beps^T\bPhi}_2^2}^{1/2} \\
            & = \frac{\delta}{n}\E_{\data}\braces*{\operatorname{Trace}\paren*{\bPhi^T\bPhi}}^{1/2} \\
            & = \delta\sqrt{\frac{q}{\nsamples}}.
        \end{split}
    \end{equation}

    Next, note that given another vector $\bW' \coloneqq (W_1',W_2',\ldots,W_n')$, we have
    \begin{equation}
    \begin{split}
        \abs*{\mathcal{G}_n(\bW; \delta;\funcclass) - \mathcal{G}_n(\bW'; \delta;\funcclass)} & \leq \sup_{\substack{f \in \funcclass \\ \norm{f}_n \leq \delta}} \abs*{\frac{1}{n}\sum_{i=1}^n \paren*{W_i - W_i'}f(\bX_i)} \\
        & \leq \frac{1}{n}\norm{\bW - \bW'}_2 \sup_{\substack{f \in \funcclass \\ \norm{f}_n \leq \delta}} \norm*{f}_2 \\
        & \leq \frac{\delta}{\sqrt{\nsamples}} \norm{\bW - \bW'}_2,
    \end{split}
    \end{equation}
    where the second inequality follows from Cauchy-Schwarz.
    As such, the function $\bW \mapsto \mathcal{G}_n(\bW; \delta;\funcclass)$ is $\delta/\sqrt{\nsamples}$-Lipschitz, and we may use Gaussian Lipschitz concentration \citep{wainwright2019high} to obtain equation \eqref{eq:local_Gaussian_complexity_concentration}.
    Finally, we may use Talagrand's convex concentration inequality \citep{talagrand1995concentration} to get equation \eqref{eq:local_Rademacher_complexity_concentration}.
\end{proof}

We now specialize to our setting under Assumption \ref{assumption}.
Consider the function class $\funcclass$ comprising all functions that can be realized by regression trees of depth $\leq s$.
Then $f^* \in \funcclass$.
Let $\partitionspace$ be the collection of all partitions realized by trees of depth $s$.
Let $\funcclass_\partition$ be the space of functions that are piecewise constant on a partition $\partition$.
We then have the decomposition
\begin{equation}
    \funcclass = \bigcup_{\partition \in \partitionspace} \funcclass_{\partition}.
\end{equation}
Note also that $\funcclass_\partition$ is a $2^s$-dimensional vector space, since it is spanned by the indicators of the leaves of $\partition$ and there are $2^s$ leaves.
Let $\text{star}(\funcclass^*,0)$ denote the star hull of $\funcclass - f^*$ at 0.
In other words, define
\begin{equation}
    \text{star}(\funcclass^*,0) = \braces*{\alpha(f-f^*) \colon f \in \funcclass, 0 \leq \alpha \leq 1}.
\end{equation}
We similarly decompose
\begin{equation}
    \text{star}(\funcclass^*,0) = \bigcup_{\partition \in \partitionspace} \text{star}(\funcclass_{\partition}^*,0),
\end{equation}
observing that each term $\text{star}(\funcclass_{\partition}^*,0)$ is a subset of a $2^s + 1$-dimensional vector space.
Moreover, we have
\begin{equation}
    \abs*{\partitionspace} \leq d^{2^s},
\end{equation}
since every tree structure is uniquely determined by the choice of covariate split on at each internal node.
There are $2^s$ internal nodes and $d$ possible covariates to split on.

\begin{lemma}
\label{lem:local_complexity_treefuncs}
    Let $\funcclass^*$ be as defined above.
    Its local Gaussian complexity and local Rademacher complexity can be bounded as follows:
    \begin{equation} \label{eq:local_gaussian_complexity_treefuncs}
        \mathcal{G}_n(\delta;\textnormal{star}(\funcclass^*,0)) \lesssim \delta\sqrt{\frac{2^s\log d}{n}},
    \end{equation}
    \begin{equation} \label{eq:local_Rademacher_complexity_treefuncs}
        \bar{\mathcal{R}}_n(\delta;\textnormal{star}(\funcclass^*,0)) \lesssim \delta\sqrt{\frac{2^s\log d}{n}}.
    \end{equation}
\end{lemma}

\begin{proof}
    For any $r \geq 1$, we can write
    \begin{equation}
    \begin{split}
        & \mathcal{G}_n(\delta;\textnormal{star}(\funcclass^*,0)) \\
        & = \E_{\bW}\braces*{\sup_{\partition \in \partitionspace}  \mathcal{G}_n\paren*{\bW; \delta;\textnormal{star}(\funcclass_{\partition}^*,0)} } \\
        & \leq \E\braces*{\paren*{\sum_{\partition \in \partitionspace} \mathcal{G}_n\paren*{\bW; \delta;\textnormal{star}(\funcclass_{\partition}^*,0)}^r}^{1/r}} \\
        & \leq \paren*{\sum_{\partition \in \partitionspace}\E\braces*{ \mathcal{G}_n\paren*{\bW; \delta;\textnormal{star}(\funcclass_{\partition}^*,0)}^r}}^{1/r} \\
        & \leq \abs*{\partitionspace}^{1/r}\sup_{\partition \in \partitionspace}\paren*{ \mathcal{G}_n(\delta;\textnormal{star}(\funcclass_{\partition}^*,0)) + \E\braces*{\abs*{\mathcal{G}_n(\bW;\delta;\textnormal{star}(\funcclass_{\partition}^*,0)) - \mathcal{G}_n(\delta;\textnormal{star}(\funcclass_{\partition}^*,0))}^r}^{1/r} }.
    \end{split}
    \end{equation}
    Using equation \ref{eq:local_Gaussian_complexity_concentration}, we get
    \begin{equation}
        \E\braces*{\abs*{\mathcal{G}_n(\bW;\delta;\textnormal{star}(\funcclass_{\partition}^*,0)) - \mathcal{G}_n(\delta;\textnormal{star}(\funcclass_{\partition}^*,0))}^r}^{1/r} \lesssim \delta\sqrt{\frac{r}{\nsamples}}.
    \end{equation}
    Combining this with the bound $\mathcal{G}_n(\delta;\textnormal{star}(\funcclass_{\partition}^*,0)) \leq \delta\sqrt{\frac{2^s+1}{n}}$ from Lemma \ref{lem:linear_subspace_complexity} and selecting $r = 2^s\log\nfeats$, we get
    \begin{equation}
    \begin{split}
        \mathcal{G}_n(\delta;\textnormal{star}(\funcclass^*,0)) & \lesssim \delta\sqrt{\frac{2^{s+1}}{n}} + \delta\sqrt{\frac{2^s\log d}{n}} \\
        & \lesssim \delta\sqrt{\frac{2^s\log d}{n}}.
    \end{split}
    \end{equation}
    Finally, equation \eqref{eq:local_Rademacher_complexity_treefuncs} is proved similarly using \eqref{eq:local_Rademacher_complexity_concentration} and the bound $\bar{\mathcal{R}}_n(\delta;\textnormal{star}(\funcclass_{\partition}^*,0)) \leq \delta\sqrt{\frac{2^s+1}{n}}$.
\end{proof}


\paragraph{Proof of Proposition \ref{prop:erm_upper_bound}}
Let us redefine $\funcclass$ to include the additional constraint that $\|f\|_\infty \leq M$ for all $f \in \funcclass$.
Recall the following definitions of critical radius (page 466 in \cite{wainwright2019high}).
\begin{equation}
    \delta_n \coloneqq \min \braces*{\delta \colon \frac{\bar{\mathcal{R}}_n(\delta;\textnormal{star}(\funcclass^*,0))}{\delta} \leq \frac{\delta}{128C}},
\end{equation}
\begin{equation}
    \epsilon_n \coloneqq \min \braces*{\epsilon \colon \frac{\mathcal{G}_n(\delta;\textnormal{star}(\funcclass^*,0))}{\epsilon} \leq \frac{\epsilon}{2\sigma}}, 
\end{equation}
where $C$ is such that $\E\braces*{f^4(\bx)} \leq C^2\E\braces*{f^2(\bx)}$ for all $f \in \textnormal{star}(\funcclass^*,0)$.
Given our additional constraint on the supremum norm on functions in $\funcclass$, we can take $C = 2M$.
Since the additional constraint can only decrease the value of the local Gaussian and Rademacher complexities, Lemma \ref{lem:local_complexity_treefuncs} gives upper bounds
\begin{equation}
    \epsilon_n \lesssim \sigma\sqrt{\frac{2^s\log d}{n}}, \quad \delta_n \lesssim M\sqrt{\frac{2^s\log d}{n}}.
\end{equation}
Corollary 14.15 from \cite{wainwright2019high} then yields
\begin{equation}
    R\paren*{\ferm(-;\data), f^*} \leq \paren*{\sigma^2 + M^2}\frac{2^s\log d}{n}
\end{equation}
with probability at least $1 - c_1\exp\paren*{-\frac{c_2\sigma^22^s\log d}{\sigma^2 + C^2}}$.
This is not enough for our purposes and we need to extend this to a high probability statement.

First, we use Theorem 13.5 in \cite{wainwright2019high} with the choice $t = \max\braces*{\epsilon_n,\frac{2\sigma^2\log\nsamples}{\nsamples\epsilon_n}}$ to obtain a probability at least $1-1/n$ event on which
\begin{equation}
\begin{split}
    \norm*{\ferm(-;\data) - f^*}_n^2 & \leq 16t\epsilon_n \\
    & = \max\braces*{16\epsilon_n^2,\frac{32\sigma^2\log\nsamples}{\nsamples}} \\
    & \lesssim \frac{\sigma^2}{\nsamples}\max\braces*{2^\sparsity\log\nfeats,\log\nsamples}.
\end{split}
\end{equation}
Likewise, we use Theorem 14.12 in \cite{wainwright2019high} with the choice $\delta = \max\braces*{\delta_n, \sqrt{\frac{4M^2\log n}{c_1 n}}}$ to obtain a probability at least $1-1/n$ event on which
\begin{equation} \label{eq:erm_upper_helper}
    \begin{split}
        R\paren*{\ferm(-;\data), f^*} & \leq 2\norm*{\ferm(-;\data) - f^*}_n^2 + \delta^2 \\
        & \lesssim \frac{\sigma^2 + M^2}{\nsamples}\max\braces*{2^\sparsity\log\nfeats,\log\nsamples}.
    \end{split}
\end{equation}

Let $\mathcal{E}$ be the event on which \eqref{eq:erm_upper_helper} holds.
    We then write
    \begin{equation}
        \objective\paren*{\ferm, f^*_0,\nfeats,\nsamples} = \E\braces*{R\paren*{\ferm(-;\data),f^*}\indicator_\mathcal{E}} + \E\braces*{R\paren*{\ferm(-;\data),f^*}\indicator_{\mathcal{E}^c}}.
    \end{equation}
    The first term can be bounded from above by $\frac{\sigma^2 + M^2}{\nsamples}\max\braces*{2^\sparsity\log\nfeats,\log\nsamples}$ multiplied by a universal constant.
    To bound the second term, we first make the simple observation that 
    \begin{equation}
        \|\ferm(-;\data)\|_{\infty} \leq \max_{1 \leq i \leq \nsamples} \abs*{Y_i}, 
    \end{equation}
    so that
    \begin{equation}
    \label{eq:erm_upper_helper_2}
        R\paren*{\ferm(-;\data),f^*} \leq \paren*{\max_{1 \leq i \leq \nsamples} \abs*{Y_i} + \norm*{f^*}_\infty}^2.
    \end{equation}
    Using Talagrand's comparison inequality (see Exercise 8.6.5 in \cite{vershynin2018high}), we have
    \begin{equation}
        \max_{1 \leq i \leq \nsamples} \abs*{Y_i} \lesssim \sigma\paren*{\sqrt{\log \nsamples} + t}
    \end{equation}
    with probability at least $1-2e^{-t^2}$, for any $t > 0$.
    Plugging this into \eqref{eq:erm_upper_helper_2} and applying Cauchy-Schwarz twice gives
    \begin{equation}
        R\paren*{\ferm(-;\data),f^*} \lesssim \norm*{f^*}_\infty^2 + \sigma^2\log\nsamples + \sigma^2u
    \end{equation}
    with probability at least $1 - 2e^{-u}$, or alternatively,
    \begin{equation}
        \P\braces*{R\paren*{\ferm(-;\data),f^*} \geq u } \leq 2e^{-C(u-\norm{f^*}_\infty^2 + \sigma^2\log\nsamples)/\sigma^2}
    \end{equation}
    whenever $u \geq \norm*{f^*}_\infty^2 + \sigma^2\log\nsamples$, where $C$ is a universal constant.
    We then compute
    \begin{equation}
        \begin{split}
            \E\braces*{R\paren*{\ferm(-;\data),f^*}\indicator_{\mathcal{E}^c}} & = \int_0^\infty \P\braces*{R\paren*{\ferm(-;\data),f^*}\indicator_{\mathcal{E}^c} \geq u} du \\
            & \leq \int_0^\infty \min\braces*{\P\braces*{R\paren*{\ferm(-;\data),f^*} \geq u},\P\braces*{\mathcal{E}^c}} du \\
            & \leq \int_0^\infty \min\braces*{\frac{2}{n},2e^{-C(u-\norm{f^*}_\infty^2 + \sigma^2\log\nsamples)/\sigma^2}} du \\
            & \lesssim \frac{\sigma^2\log\nsamples + \norm*{f^*}_\infty^2}{\nsamples}.
        \end{split}
    \end{equation}
    Combining this with the bound on the first term, we get
    \begin{equation}
        \objective\paren*{\ferm, f^*_0,\nfeats,\nsamples} \lesssim \frac{(\sigma^2+M^2)(2^s\log\nfeats + \log\nsamples)}{n}
    \end{equation}
    as we wanted. \qed


\section{Stable MSP Properties}
\label{sec:stable_msp}


\begin{proposition}
\label{prop:genericity}
    For any collection of subsets $\fcoefs$ that satisfies \MSP, if we draw $\braces*{\alpha_S}_{S \in \fcoefs}$ from any probability measure on $\R^{|\fcoefs|}$ that is absolutely continuous with respect to Lebesgue measure, then the resulting function $f^* = \sum_{S \in \fcoefs} \alpha_S \monomial_S$ is \stabmsp~almost surely.
\end{proposition}

\begin{proof}
    Because $\fcoefs$ is \MSP, there is an ordering of its subsets so that \eqref{eq:MSP} is satisfied.
    Now fix some $\cell \in \binspace$, which we may write as
    \begin{equation}
        \cell = \braces*{\bx \in \binspace \colon x_j = z_j ~\forall j \in \cellfeatindices(\cell)}.
    \end{equation}
    We then have $f^*|_\cell = \sum_{S \in \fcoefs|_\cell} \alpha_S^\cell \monomial_S$, where
    \begin{equation}
        \alpha_S^\cell \coloneqq \sum_{U \in \fcoefs, U\backslash \cellfeatindices(\cell) = S} \paren*{\prod_{j \in U\cap \cellfeatindices(\cell)} z_j} \alpha_U,
    \end{equation}
    and
    \begin{equation}
        \fcoefs|_\cell \coloneqq \braces*{S \colon \alpha_S^\cell \neq 0}.
    \end{equation}
    Define the map $\Phi \colon S \mapsto S\backslash \cellfeatindices(\cell)$ and observe that $\fcoefs|_\cell \subset \Phi(\fcoefs)$.
    We first claim that $\Phi(\fcoefs)$ is \MSP.
    To see this, simply note that for any $i$, we have
    \begin{equation}
        \Phi(S_i)\backslash \cup_{j=1}^{i-1}\Phi(S_j) = \paren*{S_i\backslash\cup_{j=1}^{i-1}S_j}\backslash \cellfeatindices(\cell),
    \end{equation}
    and hence has size almost 1.
    As such, the event on which $\fcoefs|_\cell$ is not \MSP~is contained in the union of events
    \begin{equation}
        \bigcup_{T \in \Phi(\fcoefs)}\braces*{\sum_{U \in \Phi^{-1}(T)} \paren*{\prod_{j \in U\cap \cellfeatindices(\cell)} z_j} \alpha_U = 0}.
    \end{equation}
    This is a union of hyperplanes, which has measure zero by assumption.
    Further considering the union of all such events over all choices of $\cell \subset \binspace$ completes the proof.
\end{proof}

\begin{proposition}
\label{prop:stable_msp_and_SID}
    $f^*$ satisfies \stabmsp~if and only if it satisfies the \SID~condition \eqref{eq:SID} for some $\lambda > 0$.
\end{proposition}

\begin{proof}
    We shall prove the statement by induction on the number of covariates $\nfeats$, with the base case $\nfeats = 0$ trivial.
    Now suppose the equivalence is true up to some value $\nfeats-1$, and consider $f^* \colon \binspace \to \R$.
    Assume first that for any $\cell \subset \binspace$, $f^*|_\cell$ is \MSP.
    Then for any $\cell \neq \binspace$, by the inductive hypothesis, we have that
    \begin{equation}
        \max_{j \in \suppset\backslash \cellfeatindices(\cell)}\Corr^2\braces*{X_j, f^*(\bX)~|~\bX \in \cell} > 0.
    \end{equation}
    Next, note that for any $j \in \suppset\backslash\cellfeatindices(\cell)$, we have
    \begin{equation}
        \Corr\braces*{X_j,f^*(\bX)~|~\bX \in \cell} = \alpha_{j}^\cell.
    \end{equation}
    In particular, setting $\cell = \binspace$ and noting that by assumption, $\alpha_k \neq 0$ for some $k \in \coordindices$, we have
    \begin{equation} \label{eq:msp_nonzero_correlation}
        \Corr^2\braces*{X_k,f^*(\bX)} = \alpha_k^2 > 0,
    \end{equation}
    which completes the proof that $f^*$ satisfies \SID.
    Conversely, suppose that $f^*$ satisfies \SID.
    Then by the inductive hypothesis, we see that $f^*|_\cell$ is \MSP~for any $\cell \neq \binspace$ and it remains to prove that $f^*$ itself is \MSP.
    To see this, we first see that \SID~implies \eqref{eq:msp_nonzero_correlation} for some $k \in \coordindices$, which in turn means that $\braces{k} \in \fcoefs$.
    Set $\cell = \braces*{\bx \in \binspace \colon x_k = 1}$.
    Then since $f^*|_\cell$ is \MSP~by the inductive hypothesis, there is an ordering of indices such
    \begin{equation}
        \abs*{S_i^\cell\backslash\cup_{j=1}^{i-1}S_j^\cell} \leq 1
    \end{equation}
    for any $i=1,2,\ldots,r \coloneqq \abs*{\fcoefs|_\cell}$.
    But then by the definition of $S^\cell_i$, there exists some $S_i \in \fcoefs$ such that $S_i \backslash\braces*{k} = S^\cell_i$, or in other words, that either $S_i^\cell \in \fcoefs$ or $S_i^\cell\cup\braces*{k} \in \fcoefs$.
    Augmenting the sequence with $S_0 \coloneqq \braces{k}$ gives
    \begin{equation}
        \abs*{S_i\backslash\cup_{j=1}^{i-1}S_j} \leq 1
    \end{equation}
    for $i=0,1,\ldots,r$.
    It is easy to see that $\cup_{j=1}^r S_j = \suppset$, so we may tag on all unused subsets to the end of the indexing to satisfy the definition of \MSP.
\end{proof}

\section{Proof of Theorem \ref{thm:MSP_upper_bound}}
\label{sec:upper_bounds_proofs}

We start by making some definitions.
Given a cell $\cell$, denote
\begin{equation}
    p_k(\cell) \coloneqq \frac{N(\splitoperator_{k,1}(\cell))}{N(\cell)}.
\end{equation}
This is the empirical probability of the right child of $\cell$ conditioned on $\cell$.
Next, using the law of total variance, we may rewrite the formula for impurity decrease \eqref{eq:impurity_decrease} as follows:
\begin{equation} \label{eq:imp_dec_appendix}
    \hat\Delta(k; \cell, \data) = p_k(\cell)(1-p_k(\cell))\paren*{\bar Y_{\splitoperator_{k,1}(\cell)} - \bar Y_{\splitoperator_{k,-1}(\cell)}}^2.
\end{equation}
Define the population impurity decrease as
\begin{equation} \label{eq:pop_imp_dec_appendix}
    \Delta(k; \cell, \data) \coloneqq \frac{1}{4}\paren*{\E\braces*{Y~|~\bX \in \splitoperator_{k,1}(\cell)} - \E\braces*{Y~|~\bX \in \splitoperator_{k,-1}(\cell)}}^2.
\end{equation}
Also define the square root versions of these quantities:
\begin{equation} \label{eq:rid}
    \ridhat(k; \cell, \data) \coloneqq \sqrt{p_k(\cell)(1-p_k(\cell))}\paren*{\bar Y_{\splitoperator_{k,1}(\cell)} - \bar Y_{\splitoperator_{k,-1}(\cell)}}.
\end{equation}
\begin{equation}
    \rid(k; \cell, \data) \coloneqq \frac{1}{2}\paren*{\E\braces*{Y~|~\bX \in \splitoperator_{k,1}(\cell)} - \E\braces*{Y~|~\bX \in \splitoperator_{k,-1}(\cell)}}.
\end{equation}

We next define two collection of cells of interest:
\begin{equation}
    \secondcellcollection \coloneqq \braces*{\cell \subset \binspace \colon \cellfeatindices(\cell) \subset \suppset}.
\end{equation}
\begin{equation}
    \cellcollection \coloneqq \braces*{\cell \in \secondcellcollection \colon f^*|_{\cell}~\text{is not constant}}.
\end{equation}
We will use union bounds to prove uniform concentration of the root impurity decrease values $\rid(k; \cell, \data)$, the mean response $\bar Y_{\cell}$, and the cell sample counts $N(\cell)$ for cells in these collections.

\begin{lemma}
    \label{lem:conc_cell_quantities}
    For any cell $\cell \subset \binspace$ and $k \in \coordindices$, for any $\epsilon > 0$, we have
    \begin{equation} \label{eq:rid_conc}
        \P\braces*{\sqrt{N(\cell)}\abs*{\ridhat(k;\cell,\data) - \rid(k;\cell,\data)} \geq \epsilon} \leq 6\exp\paren*{-\frac{\epsilon^2}{18(9\norm{f^*}_\infty^2 + \sigma^2)}},
    \end{equation}
    \begin{equation} \label{eq:cell_avg_conc}
        \P\braces*{\sqrt{N(\cell)}\abs*{\bar Y_{\cell} - \E\braces*{Y~|~\bX \in \cell}} \geq \epsilon} \leq 2\exp\paren*{-\frac{\epsilon^2}{2(\norm{f^*}_\infty^2 + \sigma^2)}},
    \end{equation}
    and
    \begin{equation} \label{eq:cell_count_conc}
        \P\braces*{N(\cell) \leq  \frac{(1-\epsilon)n}{2^{|\cellfeatindices(\cell)|}}} \leq \exp\paren*{-\frac{n\epsilon^2}{2^{|\cellfeatindices(\cell)|+1}}}.
    \end{equation}
\end{lemma}

\begin{proof}
    First, notice that \eqref{eq:cell_count_conc} is just an application of Chernoff's inequality \citep{vershynin2018high}.
    Next, conditioned on $\bX \in \cell$, we see that $Y = f^*(\bX) + \epsilon$ has sub-Gaussian parameter $\sqrt{\norm{f^*}_\infty^2 + \sigma^2}$ \citep{wainwright2019high}.
    Equation \ref{eq:cell_avg_conc} thus follows immediately from Hoeffding's inequality conditioning on $N(\cell)$.
    
    Finally, to prove \eqref{eq:rid_conc}, we again condition on $N(\cell)$ and write
    \begin{equation} \nonumber
    \begin{split}
        \big|\ridhat(k;\cell,\data) & - \rid(k;\cell,\data)\big| \\
        & \leq \sqrt{p_k(\cell)(1-p_k(\cell))} \abs*{\bar Y_{\splitoperator_{k,1}(\cell)} - \E\braces*{Y~|~\bX \in \splitoperator_{k,1}(\cell)}} \\
        & \quad + \sqrt{p_k(\cell)(1-p_k(\cell))} \abs*{\bar Y_{\splitoperator_{k,-1}(\cell)} - \E\braces*{Y~|~\bX \in \splitoperator_{k,-1}(\cell)}} \\
        & \quad + \abs*{\sqrt{p_k(\cell)(1-p_k(\cell))} - \frac{1}{2}}\abs*{\E\braces*{Y~|~\bX \in \splitoperator_{k,1}(\cell)} - \E\braces*{Y~|~\bX \in \splitoperator_{k,-1}(\cell)}} \\
        & \leq \sqrt{\frac{N(\splitoperator_{k,1}(\cell))}{N(\cell)}}\abs*{\bar Y_{\splitoperator_{k,1}(\cell)} - \E\braces*{Y~|~\bX \in \splitoperator_{k,1}(\cell)}} \\
        & \quad + \sqrt{\frac{N(\splitoperator_{k,-1}(\cell))}{N(\cell)}}\abs*{\bar Y_{\splitoperator_{k,-1}(\cell)} - \E\braces*{Y~|~\bX \in \splitoperator_{k,-1}(\cell)}} \\
        & \quad + 2\norm{f^*}_\infty \abs*{\sqrt{p_k(\cell)(1-p_k(\cell))} - \frac{1}{2}} \\
        & \eqqcolon I + II + III.
    \end{split}
    \end{equation}
    Applying Hoeffding's inequality individually to the first two terms and combining the resulting tail bounds, we get
    \begin{equation} \label{eq:conc_eq1}
        \P\braces*{I + II \geq \frac{\epsilon}{3}} \leq 4\exp\paren*{-\frac{N(\cell)\epsilon^2}{18\paren*{\norm*{f^*}_\infty^2 + \sigma^2}}}.
    \end{equation}
    
    To bound the third term, we first observe that
    \begin{equation} \nonumber
        \begin{split}
            \abs*{\sqrt{p_k(\cell)(1-p_k(\cell))} - \frac{1}{2}} & = \frac{\abs*{p_k(\cell)(1-p_k(\cell)) - \frac{1}{4}}}{\sqrt{p_k(\cell)(1-p_k(\cell))} + \frac{1}{2}} \\
            & \leq \frac{(1-p_k(\cell))\abs*{p_k(\cell) - \frac{1}{2}} + \frac{1}{2}\abs*{(1-p_k(\cell)) - \frac{1}{2}}}{\sqrt{p_k(\cell)(1-p_k(\cell))} + \frac{1}{2}} \\
            & \leq 3 \abs*{p_k(\cell) - \frac{1}{2}}.
        \end{split}
    \end{equation}
    As such, applying Hoeffding's inequality again gives
    \begin{equation} \label{eq:conc_eq2}
        \P\braces*{III \geq \frac{\epsilon}{3}} \leq 2\exp\paren*{-\frac{N(\cell)\epsilon^2}{162\norm*{f^*}_\infty^2}}.
    \end{equation}
    Combining \eqref{eq:conc_eq1} and \eqref{eq:conc_eq2}, dividing $\epsilon$ by $\sqrt{N(\cell)}$ and then deconditioning on $N(\cell)$, we complete the proof of \eqref{eq:rid_conc}.
\end{proof}

\begin{lemma}
    \label{lem:CART_invariants}
    For any $0 < \delta < 1$, there is an event with probability at least $1-\delta$ over which the following inequalities hold:
    \begin{equation}
        \label{eq:relevant_rid_lower_bound}
        \min_{\cell \in \cellcollection} \paren*{\max_{k \in \suppset} \ridhat(k;\cell,\data) - \sqrt{2^{|\cellfeatindices(\cell)|}\lambda} + \sqrt{\frac{18(9\norm{f^*}_\infty^2 + \sigma^2)\paren*{(s+2)\log 3 + \log(2/\delta)}}{N(\cell)}}} \geq 0,
    \end{equation}
    \begin{equation}
        \label{eq:irrelevant_rid_upper_bound}
        \max_{\cell \in \secondcellcollection} \paren*{\max_{k \notin \suppset} \ridhat(k;\cell,\data) - \sqrt{\frac{18(9\norm{f^*}_\infty^2 + \sigma^2)\paren*{(s+2)\log 3 + \log(2\nfeats/\delta)}}{N(\cell)}}} \leq 0,
    \end{equation}
    and
    \begin{equation}
        \label{eq:cell_mean_uniform}
        \max_{\cell \in \secondcellcollection} \paren*{\abs*{\bar Y_{\cellsampindices(\cell)} - \E\braces*{Y~|~\bx \in \cell}} - \sqrt{\frac{2(\norm{f^*}_\infty^2 + \sigma^2)\paren*{(s+1)\log 3 +\log(2/\delta)}}{N(\cell)}}} \leq 0.
    \end{equation}
    Furthermore, with probability at least $1-\exp\paren*{\sparsity\log3 - \frac{\nsamples}{2^{\sparsity+3}}}$, we have
    \begin{equation} 
        \label{eq:cell_count_uniform}
        \min_{\cell \in \secondcellcollection} \frac{N(\cell)2^{|\cellfeatindices(\cell)|+1}}{\nsamples} \geq 1.
    \end{equation}
\end{lemma}

\begin{proof}
    First, note that we have 
    \begin{equation}
        \secondcellcollection = \bigcup_{S \subset \suppset} \braces*{\cell \colon \cellfeatindices(\cell) = S},
    \end{equation}
    so that
    \begin{equation}
        \abs*{\secondcellcollection} = \sum_{k=0}^\sparsity \binom{s}{k}2^k = 3^\sparsity.
    \end{equation}
    Using \eqref{eq:cell_count_conc} from Lemma \ref{lem:conc_cell_quantities} together with a union bound, we have $N(\cell) \geq \frac{n}{2^{|\cellfeatindices(\cell)|+1}}$ uniformly for all $\cell \in \cellcollection$ with probability at least $1-\exp\paren*{\sparsity\log3 - \frac{\nsamples}{2^{|\cellfeatindices(\cell)|+3}}}$, which proves \eqref{eq:cell_count_uniform}.
    
    Next, observe that for any $\cell \in \cellcollection$, we have
    \begin{equation}
        \Corr\braces*{x_k,f(\bx)~|~\bx \in \cell} = \frac{1}{2}\paren*{\E\braces*{Y~|~\bx \in \splitoperator_{k,1}(\cell)} - \E\braces*{Y~|~\bx \in \splitoperator_{k,-1}(\cell)}} = \rid(k;\cell,\data).
    \end{equation}
    By \eqref{eq:stab_msp} and the definition of $\cellcollection$, we thus have
    \begin{equation}
        \Delta(k_0(\cell);\cell,\data) \geq 2^{\abs*{\cellfeatindices(\cell)}}\lambda
    \end{equation}
    for some choice of $k_0(\cell) \in \suppset$.
    Using Lemma \ref{lem:conc_cell_quantities} together with \eqref{eq:stab_msp} and a union bound, we therefore get a probability at least $1-\delta/3$ event on which
    \begin{equation}
    \begin{split}
        \ridhat(k;\cell,\data) & \geq \rid(k;\cell,\data) - \sqrt{\frac{18(9\norm{f^*}_\infty^2 + \sigma^2)\paren*{(s+2)\log 3 + \log(2/\delta)}}{N(\cell)}} \\
        & \geq \sqrt{2^{|\cellfeatindices(\cell)|}\lambda} - \sqrt{\frac{18(9\norm{f^*}_\infty^2 + \sigma^2)\paren*{(s+2)\log 3 + \log(2/\delta)}}{N(\cell)}}
    \end{split}
    \end{equation}
    uniformly over all $\cell \in \cellcollection$.
    Reasoning similarly, we get another probability at least $1-\delta/3$ event on which
    \begin{equation}
        \max_{k \notin \suppset}\ridhat(k;\cell,\data) \leq \sqrt{\frac{18(9\norm{f^*}_\infty^2 + \sigma^2)\paren*{(s+2)\log 3 + \log(2\nfeats/\delta)}}{N(\cell)}}
    \end{equation}
    uniformly over all $\cell \in \secondcellcollection$.
    Finally, using \eqref{eq:cell_avg_conc} together with a union bound, we get \eqref{eq:cell_mean_uniform}.
\end{proof}

Theorem \ref{thm:MSP_upper_bound} will follow immediately from the more general version that we now state.

\begin{theorem}
    \label{thm:MSP_upper_bound_appendix}
    Under Assumption \ref{assumption}, let $f_0^* \in \stabmsp(\lambda)$.
    Suppose $\varepsilon$ is sub-Gaussian with parameter $\sigma$ \citep{wainwright2019high}.
    Fix $0 < \delta < 1$.
    Denote 
    \begin{equation}
        \tau \coloneqq 18(9\norm{f^*}_\infty^2 + \sigma^2)\paren*{(s+2)\log 3 + \log(2\nfeats/\delta)}
    \end{equation}
    and suppose the training sample size satisfies $n > \frac{8\tau}{\lambda}$.
    Suppose further that we fit $\fcart$ with the stopping rule given by a minimum impurity decrease value $\gamma$ satisfying $\frac{\tau}{n} \leq \gamma < \paren*{\sqrt{\frac{\lambda}{2}} - \sqrt{\frac{\tau}{n}}}_+^2$.
    Such a value $\gamma$ exists given the lower bound on $\nsamples$.
    Then with probability at least $1-\delta-\exp(s\log 3 - n2^{-s-3})$, we have
    \begin{equation}
        \label{eq:msp_upper_probability}
        R\paren*{\fcart(-;\data),f^*} \leq \frac{2^{\sparsity+1}(\norm{f^*}_\infty^2 + \sigma^2)\paren*{(s+1)\log 3 +\log(2/\delta)}}{n}.
    \end{equation}
    In addition, under the same conditions and the choice $\delta = 1/n$, we have
    \begin{equation}
        \label{eq:msp_upper_expectation_appendix}
        \objective\paren*{\fcart, f^*_0,\nfeats,\nsamples} = O\paren*{\frac{2^{\sparsity}(\norm{f^*}_\infty^2 + \sigma^2)(\sparsity + \log \nsamples)}{n}}.
    \end{equation}
\end{theorem}

\begin{proof}
    We will prove \eqref{eq:msp_upper_probability} in three steps.
    The first two steps show that with high probability, the tree $\tree_{\operatorname{CART}}$ fitted using CART with the stopping condition specified by Theorem \ref{thm:MSP_upper_bound} does not make any splits on irrelevant covariates and that $f^*$ is constant on its leaves.
    The third step will show that the mean responses over the leaves of $\tree_{\operatorname{CART}}$ concentrate around their true values with high probability.
    These together will give us the result that we want.
    
    \textit{Step 1: No bad splits.}
    Let $\event$ denote the event on which \eqref{eq:irrelevant_rid_upper_bound}, \eqref{eq:relevant_rid_lower_bound} and \eqref{eq:cell_mean_uniform} hold.
    We first claim that on $\event$, $\tree_{\operatorname{CART}}$ does not make any splits in irrelevant covariates.
    Suppose not, then there exists a split on an irrelevant covariate occurring at a minimum depth.
    By minimality, this split must occur on a cell $\cell$ in $\secondcellcollection$.
    On $\event$, we have by equation \eqref{eq:irrelevant_rid_upper_bound} that
    \begin{equation}
        \label{eq:no_bad_splits}
        \max_{k \notin \suppset} \frac{N(\cell)}{n}\hat\Delta(k;\cell,\data) \leq \frac{18(9\norm{f^*}_\infty^2 + \sigma^2)\paren*{(s+2)\log 3 + \log(2\nfeats/\delta)}}{n} < \gamma,
    \end{equation}
    which implies that $\cell$ was either split on a relevant covariate $k$ or that it was not split at all, thereby giving a contradiction.

    \textit{Step 2: $f^*$ is constant on leaves.}
    We have already shown that every leaf $\cell$ of $\tree_{\operatorname{CART}}$ is an element of $\secondcellcollection$.
    We now claim that, upon further conditioning on the event on which \eqref{eq:cell_count_uniform} holds, $\cell$ must also be an element of $\secondcellcollection\backslash \cellcollection$.
    Suppose not and instead $\cell \in \cellcollection$, then by combining equations \eqref{eq:relevant_rid_lower_bound} and \eqref{eq:cell_count_uniform}, we get
    \begin{equation}
        \begin{split}
            \max_{k \in \suppset} \frac{N(\cell)}{n}\hat\Delta(k;\cell,\data) & \geq \paren*{\sqrt{\frac{N(\cell)2^{|\cellfeatindices(\cell)|}\lambda}{n}} - \sqrt{\frac{18(9\norm{f^*}_\infty^2 + \sigma^2)\paren*{(s+2)\log 3 + \log(2\nfeats/\delta)}}{n}}}_+^2 \\
            & \geq \paren*{\sqrt{\frac{\lambda}{2}} - \sqrt{\frac{18(9\norm{f^*}_\infty^2 + \sigma^2)\paren*{(s+2)\log 3 + \log(2\nfeats/\delta)}}{n}}}_+^2 \\
            & > \gamma,
        \end{split}
    \end{equation}
    which means that $\cell$ does not satisfy the stopping condition and should not be a leaf, giving a contradiction.
    
    \textit{Step 3: Concentration of function values.}
    Let $\leaves$ denote the set of leaves of $\tree_{\operatorname{CART}}$.
    By the previous two steps, we know that $\leaves \subset \secondcellcollection\backslash \cellcollection$, which implies that $\abs*{\cellfeatindices(\cell)} \leq \sparsity$ and that $f^*(\bx) = \E\braces*{Y~|~\bx \in \cell}$ for all $\cell \in \leaves$.
    As a consequence of \eqref{eq:cell_mean_uniform}, for any $\bx \in \binspace$ and letting $\cell(\bx)$ be the leaf that contains it, we have
    \begin{equation}
        \begin{split}
            \abs*{\fcart(\bx) - f^*(\bx)} & = \abs*{\bar Y_{\cellsampindices(\cell)} - f^*(\bx)} \\
            & = \abs*{\bar Y_{\cellsampindices(\cell)} - \E\braces*{Y~|~\bX \in \cell(\bx)}} \\
            & \leq \sqrt{\frac{2(\norm{f^*}_\infty^2 + \sigma^2)\paren*{(s+1)\log 3 +\log(2/\delta)}}{N(\cell(\bx))}} \\
            & \leq \sqrt{\frac{2^{\sparsity+1}(\norm{f^*}_\infty^2 + \sigma^2)\paren*{(s+1)\log 3 +\log(2/\delta)}}{n}}.
        \end{split}
    \end{equation}
    Since $\bx$ was arbitrary, we may integrate this bound over $\bX \in \binspace$ to get
    \begin{equation}
        R\paren*{\fcart(-;\data),f^*} \leq \frac{2^{\sparsity+1}(\norm{f^*}_\infty^2 + \sigma^2)\paren*{(s+1)\log 3 +\log(2/\delta)}}{n},
    \end{equation}
    as we wanted.

    To prove \eqref{eq:msp_upper_expectation_appendix}, we take $\delta = 1/n$ in \eqref{eq:msp_upper_probability}.
    Let $\mathcal{E}$ be the event on which \eqref{eq:msp_upper_probability} holds.
    We then write
    \begin{equation}
        \objective\paren*{\fcart, f^*_0,\nfeats,\nsamples} = \E\braces*{R\paren*{\fcart(-;\data),f^*}\indicator_\mathcal{E}} + \E\braces*{R\paren*{\fcart(-;\data),f^*}\indicator_{\mathcal{E}^c}}.
    \end{equation}
    The first term can be bounded from above by $\frac{2^{\sparsity+1}(\norm{f^*}_\infty^2 + \sigma^2)\paren*{(s+1)\log 3 +\log(2/\delta)}}{n}$.
    To bound the second term, we note that since $\fcart$ makes predictions by averaging responses, we have 
    \begin{equation}
        \|\fcart(-;\data)\|_{\infty} \leq \max_{1 \leq i \leq \nsamples} \abs*{Y_i}, 
    \end{equation}
    so that
    \begin{equation}
    \label{eq:msp_risk_upper_helper}
        R\paren*{\fcart(-;\data),f^*} \leq \paren*{\max_{1 \leq i \leq \nsamples} \abs*{Y_i} + \norm*{f^*}_\infty}^2.
    \end{equation}
    Using Talagrand's comparison inequality (see Exercise 8.6.5 in \cite{vershynin2018high}), we have
    \begin{equation}
        \max_{1 \leq i \leq \nsamples} \abs*{Y_i} \lesssim \sigma\paren*{\sqrt{\log \nsamples} + t}
    \end{equation}
    with probability at least $1-2e^{-t^2}$, for any $t > 0$.
    Plugging this into \eqref{eq:msp_risk_upper_helper} and applying Cauchy-Schwarz twice gives
    \begin{equation}
        R\paren*{\fcart(-;\data),f^*} \lesssim \norm*{f^*}_\infty^2 + \sigma^2\log\nsamples + \sigma^2u
    \end{equation}
    with probability at least $1 - 2e^{-u}$, or alternatively,
    \begin{equation}
        \P\braces*{R\paren*{\fcart(-;\data),f^*} \geq u } \leq 2e^{-C(u-\norm{f^*}_\infty^2 + \sigma^2\log\nsamples)/\sigma^2}
    \end{equation}
    whenever $u \geq \norm*{f^*}_\infty^2 + \sigma^2\log\nsamples$, where $C$ is a universal constant.
    We then compute
    \begin{equation}
        \begin{split}
            \E\braces*{R\paren*{\fcart(-;\data),f^*}\indicator_{\mathcal{E}^c}} & = \int_0^\infty \P\braces*{R\paren*{\fcart(-;\data),f^*}\indicator_{\mathcal{E}^c} \geq u} du \\
            & \leq \int_0^\infty \min\braces*{\P\braces*{R\paren*{\fcart(-;\data),f^*} \geq u},\P\braces*{\mathcal{E}^c}} du \\
            & \leq \int_0^\infty \min\braces*{\frac{1}{n},2e^{-C(u-\norm{f^*}_\infty^2 + \sigma^2\log\nsamples)/\sigma^2}} du \\
            & \lesssim \frac{\sigma^2\log\nsamples + \norm*{f^*}_\infty^2}{\nsamples}.
        \end{split}
    \end{equation}
    This completes the bound on the second term and hence the proof of \eqref{eq:msp_upper_expectation_appendix}.
\end{proof}

\section{Proof of Theorem \ref{thm:robust_lower_bounds}}
\label{sec:robust_lower_bounds_proofs}

\paragraph{Proof of Theorem \ref{thm:robust_lower_bounds}}

Imitating the proof of Theorem \ref{thm:non_msp_lower_bound}, we combine Lemma \ref{lem:covariate_selection_tree_v2} together with an extension of Lemma \ref{lem:nonmsp_feature_selection_prob} to accommodate a vertex cut and the new coupling argument. \qed

\begin{lemma}
    \label{lem:nonmsp_robst_feature_selection_prob}
    Make the same assumptions and use the same notation as in Theorem \ref{thm:robust_lower_bounds}.
    Then for each fixed query point $\bx \in \binspace$, support set $\suppset$, and random seed $\Theta$, the probability of selecting any covariate in the traversal $\traversal$ along the query path $\querypath(\bx;\data,\Theta)$ is upper bounded as $\P_{\data}\braces*{J(\bx;\data,\Theta)\cap \traversal \neq \emptyset } \leq \frac{\traversalsize(\log_2 \nsamples + 2)}{\nfeats - \sparsity_{-\cutset,\MSP}-\traversalsize +1} + \sqrt{\frac{\nsamples w(\cutset)}{4\sigma^2}}$.
\end{lemma}

\begin{proof}
    Denote $\check f^* = f^*_{-\cutset} = \sum_{S \in \fcoefs\backslash\cutset}\alpha_S\monomial_S$.
    As described in Section \ref{sec:robust_lower_bounds}, we define $\modifieddata \coloneqq \braces*{(\bx_i,\check Y_i)}_{i=1}^\nsamples$ with $\check Y_i = \check f^*(\bx_i) + \varepsilon_i$ for $i=1,2,\ldots,\nsamples$.
    By independence of the data points, we have
    \begin{equation}
    \label{eq:KL_divergence_helper}
        \KL\paren*{\data\|\modifieddata} = \sum_{i=1}^\nsamples \KL\paren*{(\bX_i,Y_i)\|(\bX_i,\check Y_i)}.
    \end{equation}
    Using the chain rule for KL divergence, we get
    \begin{equation}
        \KL\paren*{(\bX_i,Y_i)\|(\bX_i,\check Y_i)} = \KL(\bX_i\|\bX_i) + \E\braces*{\KL\paren*{(Y_i|\bX_i)\|(\check Y_i|\bX_i)}}.
    \end{equation}
    The first term on the right hand side is equal to $0$ since the covariate distributions are the same.
    To compute the second term, notice that $Y_i|\bX_i \sim \mathcal{N}(f^*(\bX_i),\sigma^2)$, and $\check Y_i|\bX_i \sim \mathcal{N}(\check f^*(\bX_i),\sigma^2)$.
    The KL divergence between them is thus
    \begin{equation}
        \KL\paren*{(Y_i|\bX_i)\|(\check Y_i|\bX_i)} = \frac{\paren*{f^*(\bX_i)-\check f^*(\bX_i)}^2}{2\sigma^2}.
    \end{equation}
    Taking expectations with respect to $\bX_i$, we get
    \begin{equation}
        \E\braces*{\KL\paren*{(Y_i|\bX_i)\|(\check Y_i|\bX_i)}} = \frac{\E\braces*{\paren*{\sum_{S \in \cutset} \alpha_S \monomial_S(\bX_i)}}^2}{2\sigma^2} = \frac{w(\cutset)}{2\sigma^2}.
    \end{equation}
    Plugging this back into \eqref{eq:KL_divergence_helper} gives
    \begin{equation}
        \KL\paren*{\data\|\modifieddata} \leq \frac{\nsamples w(\cutset)}{2\sigma^2},
    \end{equation}
    and using Pinsker's inequality, we get
    \begin{equation}
        \TV\paren*{\data,\modifieddata} \leq \sqrt{\frac{\nsamples w(\cutset)}{4\sigma^2}}.
    \end{equation}
    Using the coupling interpretation of total variation distance, we may hence a coupling between $\data$ and $\modifieddata$ such that there is an event with probability at least $1 - \sqrt{\frac{\nsamples w(\cutset)}{4\sigma^2}}$ on which $\modifieddata = \data$.
    We then have
    \begin{equation}
        \begin{split}
            \P_{\data}\braces*{J(\bx;\data,\Theta)\cap \traversal \neq \emptyset } & \leq \P_{\data}\braces*{J(\bx;\modifieddata,\Theta)\cap \traversal \neq \emptyset ~\text{or}~ \modifieddata \neq \data} \\
            & \leq \P_{\data}\braces*{J(\bx;\modifieddata,\Theta)\cap \traversal \neq \emptyset} + \P\braces*{\modifieddata \neq \data} \\
            & \frac{\traversalsize(\log_2 \nsamples + 2)}{\nfeats - \sparsity_{-\cutset,\MSP}-\traversalsize +1} + \sqrt{\frac{\nsamples w(\cutset)}{4\sigma^2}}. \qedhere
        \end{split}
    \end{equation}
\end{proof}

\section{Constant Probability Lower Bounds}
\label{sec:high_prob_lower_bounds}

In this section, we show how to obtain lower bounds for greedy trees and ensembles that hold with constant probability rather than in expectation.

\begin{theorem}
\label{thm:high_prob_lower_bound}
    Under Assumption \ref{assumption}, let $f_0^* = \sum_{S \in \fcoefs} \alpha_S\monomial_S$ be any non-\MSP~function, and let $r_{\MSP} \coloneqq \sum_{S\in \fcoefs\backslash\fcoefs_{\MSP}}\alpha_S\monomial_S$ denote its \MSP~residual.
    Let $\traversalsize$ be the minimum size of a traversal for $\fcoefs\backslash\fcoefs_{\MSP}$ and let $\sparsity_{\MSP} \coloneqq \abs*{\suppset_{\MSP}}$.
    Suppose $\log_2\nsamples \leq \frac{\nfeats-s_{\MSP}-\traversalsize+1}{6s_{\MSP}} - 2$, then the $L^2$ risk of any greedy tree model satisfies
    \begin{equation}
        R\paren*{\fgreedy(-;\data,\Theta),f^*} \geq \frac{\Var\braces*{r_{\MSP}(\bX)}}{6}
    \end{equation}
    with probability at least $1/2$.
    If furthermore $\abs{Y} \leq \ybound$ almost surely and 
    \begin{equation}
        \log_2\nsamples \leq \frac{\Var\braces{r_{\MSP}(\bX)}(\nfeats-s_{\MSP}-\traversalsize+1)}{8M^2s_{\MSP}} - 2, 
    \end{equation}
    then the $L^2$ risk of any greedy tree ensemble model satisfies
    \begin{equation}
        \E_{\rfparam}\braces*{R\paren*{\fgre(-;\data,\rfparam)}} \geq \frac{\Var\braces*{r_{\MSP}(\bx)}}{4}.
    \end{equation}
    with probability at least $1/2$.
\end{theorem}


\begin{proof}
    Note that we have
    \begin{equation}
        R\paren*{\fgreedy(-;\data,\Theta),f^*} \geq V(\data,\Theta)\Var\braces*{r_{\MSP}(\bX)},
    \end{equation}
    where
    \begin{equation}
        V(\data,\Theta) \coloneqq \frac{\E_{\bX_{\coordindices\backslash\traversal}}\braces*{\Var\braces*{f^*(\bX)~|~\bX_{\coordindices\backslash\traversal}}\indicator\braces*{J(\bX;\data,\Theta)\cap\traversal = \emptyset}}}{\E_{\bX_{\coordindices\backslash\traversal}}\braces*{\Var\braces*{f^*(\bX)~|~\bX_{\coordindices\backslash\traversal}}}}.
    \end{equation}
    It is clear that $0 \leq V(\data,\Theta) \leq 1$.
    Furthermore, using Lemma \ref{lem:nonmsp_feature_selection_prob}, we see that $\E_{\data}\braces*{V(\data,\Theta)} \geq 1 - \delta$, where $\delta \coloneqq \frac{s_{\MSP}(\log_2\nsamples+2)}{\nfeats - s_{\MSP} - \traversalsize+1}$.
    Next, making use of the simple inequality
    \begin{equation}
        \E_{\data}\braces*{V(\data,\Theta)} \leq \P_{\data}\braces*{V(\data,\Theta) \geq \delta} + \delta,
    \end{equation}
    we get
    \begin{equation}
        \P_{\data}\braces*{V(\data) \geq \delta} \geq 1 - 3\delta.
    \end{equation}
    As such, if $\delta \leq 1/6$, we get a probability at least $1/2$ event over which
    \begin{equation}
        R\paren*{\fgreedy(-;\data,\Theta),f^*} \geq \frac{\Var\braces*{r_{\MSP}(\bX)}}{6}.
    \end{equation}
    This choice of $\delta$ corresponds to $\log_2\nsamples \leq \frac{\nfeats-s_{\MSP}-\traversalsize+1}{6s_{\MSP}} - 2$.
    
    To prove the second statement, write
    \begin{equation}
        h(\bx;\data) \coloneqq \E_{\theta_1}\braces*{\fgreedy(\bx;\data,\theta_1)\indicator\braces*{J(\bx;\data,\theta_1)\cap \traversal \neq \emptyset }}.
    \end{equation}
    A modification of the proof of Lemma \ref{lem:covariate_selection_ensemble_v2} gives
    \begin{equation}
    \begin{split}
        \E_{\rfparam}\braces*{R\paren*{\fgre(-;\data,\rfparam)}} & \geq \E_{\bX}\braces*{\frac{\Var\braces*{f^*(\bX)~|~\bX_{\coordindices\backslash\traversal}}}{2} - \Var\braces*{h(\bX;\data)~|~\bX_{\coordindices\backslash\traversal}}} \\
        & \geq \frac{\Var\braces*{r_{\MSP}(\bX)}}{2} - \E_{\bX}\braces*{\Var\braces*{h(\bX;\data)~|~\bX_{\coordindices\backslash\traversal}}}.
    \end{split}
    \end{equation}
    To bound the second term, we first compute
    \begin{equation}
        \E_{\bX}\braces*{\Var\braces*{h(\bX;\data)~|~\bX_{\coordindices\backslash\traversal}}} \leq \E_\bX\braces*{h(\bX;\data)^2}.
    \end{equation}
    Next observe that
    \begin{equation}
    \begin{split}
        \E_\bX\braces*{h(\bX;\data)^2} & \leq \E_{\bX,\theta_1}\braces*{\fgreedy(\bX;\data,\theta_1)^2\indicator\braces*{J(\bx;\data,\theta_1)\cap \traversal \neq \emptyset }} \\
        & \leq M^2\P_{\bX,\theta_1}\braces*{J(\bX;\data,\theta_1)\cap \traversal \neq \emptyset }.
    \end{split}
    \end{equation}
    Taking expectations with respect to $\data$ gives
    \begin{equation}
        \E_{\data}\braces*{\P_{\bX,\theta_1}\braces*{J(\bX;\data,\theta_1)\cap \traversal \neq \emptyset }} \leq \delta,
    \end{equation}
    which means that there is a probability at least $1/2$ event over which
    \begin{equation}
        \P_{\bX,\theta_1}\braces*{J(\bx;\data,\theta_1)\cap \traversal \neq \emptyset } \leq 2\delta.
    \end{equation}
    If $\delta \leq \frac{\Var\braces*{r_{\MSP}(\bX)}}{8M^2}$, we get a probability at least $1/2$ event over which
    \begin{equation}
        \E_{\rfparam}\braces*{R\paren*{\fgre(-;\data,\rfparam)}} \geq \frac{\Var\braces*{r_{\MSP}(\bX)}}{4}.
    \end{equation}
    This choice of $\delta$ corresponds to $\log_2\nsamples \leq \frac{\Var\braces{r_{\MSP}(\bX)}(\nfeats-s_{\MSP}-\traversalsize+1)}{8M^2s_{\MSP}} - 2$.
\end{proof}

\section{Comparisons with Neural Networks Trained by SGD}
\label{sec:related_work}

In this section, we expand on our discussion in Section \ref{subsec:comparisons_with_NNs} comparing greedy regression trees and ensembles and NNs trained by SGD.
We first remark that the performance of NNs can vary widely according to the choice of architecture, activation function, the exact version of SGD used including the initialization, batch-size, learning rate (or equivalently the time-horizon), and other hyperparameters.
Theoretical analyses of NNs likewise consider a wide variety of design and algorithmic choices, with many taking place under asymptotic limits or eliding the role of optimization altogether.
We focus our comparison to be with respect to practical versions of NNs trained using SGD, for which results have been obtained only very recently.
Specifically, researchers have sought to establish upper and lower bounds on the sample complexity required by these algorithms for “learning” classes of functions from i.i.d. noiseless observations, where “learning” in this context is defined as achieving small $L^2$ estimation error, given a particular covariate distribution, usually chosen to be the uniform distribution on Boolean covariates or the standard normal distribution on continuous covariates.

As mentioned, \cite{abbe2022merged} used the \MSP~to characterize which sparse Boolean functions are learnable by mean-field online SGD.
We now elaborate on this point.
Mean-field online SGD refers to discrete-time online SGD with one sample per iteration with a very small constant step-size, for $O(d)$ iterations, on a very wide two-layer NN.\footnote{The approximation also holds for online batch SGD, under some constraints on the batch size (see \cite{abbe2022merged}).}
In this regime, the behavior of the weights under SGD is well-approximated by a non-linear partial differential equation, referred to as mean-field dynamics \citep{mei2018mean}.
Due to the difficulty in proving uniform-in-time propagation of chaos in this setting, this mean-field approximation holds only for $Cd$ iterations of SGD, where $C$ can be arbitrarily large so long as other parameters are taken to be large (or small) enough.
\cite{abbe2022merged} showed that when $f^*$ does not satisfy \MSP, mean-field dynamics does not converge to any solution with small estimation error, whereas when $f^*$ satisfies \MSP~and is generic, mean-field dynamics does converge to such a solution.
Under the mean-field approximation, the same holds for online SGD, which leads to our claims in Section \ref{subsec:comparisons_with_NNs}.
This creates an interesting analogy with greedy regression trees and random forests, whereby both classes of methods have ``low'' sample complexity for \MSP~functions, but ``high'' sample complexity for non-\MSP~functions.
The limitation of this comparison is that it does not reflect what happens when SGD is run beyond the $O(d)$ iteration horizon when the mean-field approximation no longer holds.

To resolve this limitation with their analysis, \cite{abbe2023sgd} proposed a conjecture phrased in terms of ``leap complexity'', which is also a new concept that they defined.
Given a Boolean function $f^* \coloneqq \sum_{S \in \fcoefs}\alpha_S\monomial_S$, its leap complexity is defined as the minimum value $l$ such that there is an ordering of its coefficient subsets $S_1 \subset S_2 \subset \cdots S_r$ with
\begin{equation}
    \abs*{S_i\backslash\cup_{j =1}^{i-1} S_j} \leq l
\end{equation}
for any $i \in \braces{1,2,\ldots,r}$.
As can easily be seen, \MSP~functions are equivalent to leap-$1$ functions, while non-\MSP~functions are categorized as leap-$2$, leap-$3$, and so on.
A Boolean monomial of degree $k$ (i.e., a $k$-parity) is an example of a leap-$k$ function.
With this definition, \cite{abbe2023sgd} (see Conjecture 1 therein) conjectured that the time and sample complexity required to learn a leap-$k$ function using online SGD on a two-layer NN is $\tilde{\Theta}\paren*{\nfeats^{\max\braces*{k-1,1}}}$.

If this conjecture is true, then combining it with the results in our paper yields rigorous theoretical performance gaps between two-layer NNs trained with online SGD (henceforth denoted as \textsf{2LNNOSGD}) and greedy regression trees and ensembles:
More specifically, for leap-$1$ functions, that is, those satisfying \MSP, CART requires a sample complexity of just $O(2^s\log\nfeats)$ which is much smaller than the $\tilde{\Theta}(\nfeats)$ sample complexity required by \textsf{2LNNOSGD}.
On the other hand, for leap-$k$ functions for any $k > 1$, CART requires a sample complexity of $\exp(\Omega(\nfeats))$, which is much larger than the $\tilde{\Theta}\paren*{\nfeats^{k-1}}$ sample complexity required by \textsf{2LNNOSGD}.
Here we note that our lower bound in Theorem \ref{thm:informal_intro} holds even with noiseless observations and also with constant probability (see Theorem \ref{thm:high_prob_lower_bound}), so that both classes of learning algorithms are put on exactly the same theoretical footing.

\cite{abbe2023sgd} was not able to prove this conjecture and instead proved a version of it for isotropic Gaussian data.
However, several works have produced upper bounds for $k$-parities in support of the conjecture.
Considering the classification setup, \cite{Glasgow2024sgd} showed in particular that a sufficiently finitely polynomial-width ReLU network trained using online minibatch SGD with constant step-size learns the XOR function within $\tilde{O}(\nfeats)$ samples and iterations.
\cite{kou2024matching} showed that for any $k$, online \emph{sign}-SGD with a batch size of $\tilde{O}(\nfeats^{k-1})$ can learn a $k$-parity (i.e., degree $k$ monomial) within $O(\log\nfeats)$ iterations, for a total sample complexity of $\tilde{O}(\nfeats^{k-1})$.
These results thus establish a rigorous performance gap between the two algorithm classes independent of the validity of \cite{abbe2023sgd}'s conjecture.

The lower bounds portion of \cite{abbe2023sgd}'s conjecture is \emph{suggested} by existing Statistical Query (SQ) and Correlational Statistical Query (CSQ) lower bounds.
SQ and CSQ are generic models of constrained computation \citep{kearns1998efficient,feldman2021statistical}.
In the SQ model, the analyst has access to a sequence of queries, each of which returns an approximation of the expectation $\E\braces*{\psi(\textbf{X}, Y)}$ for some real-valued function $\psi$ that may change from query to query. 
The CSQ model further constrains the queries to be of the form $\E\braces*{\phi(\textbf{X})Y}$ for some real-valued function $\phi$.
Under the SQ model, it is known that $\Omega\paren*{\nfeats^k})$ queries are required to learn $k$-parities \citep{barak2022hidden}, while \cite{abbe2023sgd} proved a CSQ lower bound of the same order.
Now, \cite{abbe2021power} showed that SGD is no more powerful than the SQ model if the stochastic gradients have limited precision, but it is unclear to us whether this holds for the SGD regime in \textsf{2LNNOSGD}.
Assuming the reduction holds, then the SQ and CSQ lower bounds translate to a lower sample complexity bound of $\tilde{\Theta}\paren*{\nfeats^{k-1}}$ \citep{kou2024matching}.

We also note an additional performance gap suggested by \cite{abbe2023sgd}'s conjecture.
They in fact consider a generalization of sparse Boolean functions called \emph{multi-index functions}.
These are defined as $f^*(\bx) = g(\bM\bx)$, where $\bM$ is an unknown $\nfeats \times \sparsity$ matrix with orthogonal columns.
In other words, $f^*$ depends only on an $s$-dimensional latent subspace.
The fact that this subspace is not axis-aligned is irrelevant to SGD and neural networks, which are rotation-invariant.
\cite{Glasgow2024sgd}'s upper bound for XOR was furthermore proved in this setting.\footnote{\cite{kou2024matching}'s result required axis alignment, which they exploited using sign SGD.}
On the other hand, having sparsity in the original covariate basis is crucial to the good performance of regression trees and ensembles in general.
This is because every root-to-leaf path can split on at most $O(\log\nfeats)$ covariates, and hence cannot fit to any dense rotation of a latent subspace, no matter how small the latent dimension is \citep{tan2021cautionary}.
This drawback may be overcome using oblique (that is, non-axis-aligned) regression trees and ensembles \citep{cattaneo2022convergence,o2024statistical}, but these have much larger computational overhead. 
Finally, we note that for continuous covariates, the \emph{information exponent} introduced by \cite{arous2024high} and its generalizations by \cite{oko2024learning,ren2024learning} may be more effective in capturing learnability. 
Exploring similar measures of learnability for greedy recursive partitioning estimators presents an interesting direction for future research.

\section{Additional empirical results}
\label{sec:sims_appendix}

\begin{figure}[H]
    \centering
    \includegraphics[width=0.9\linewidth]{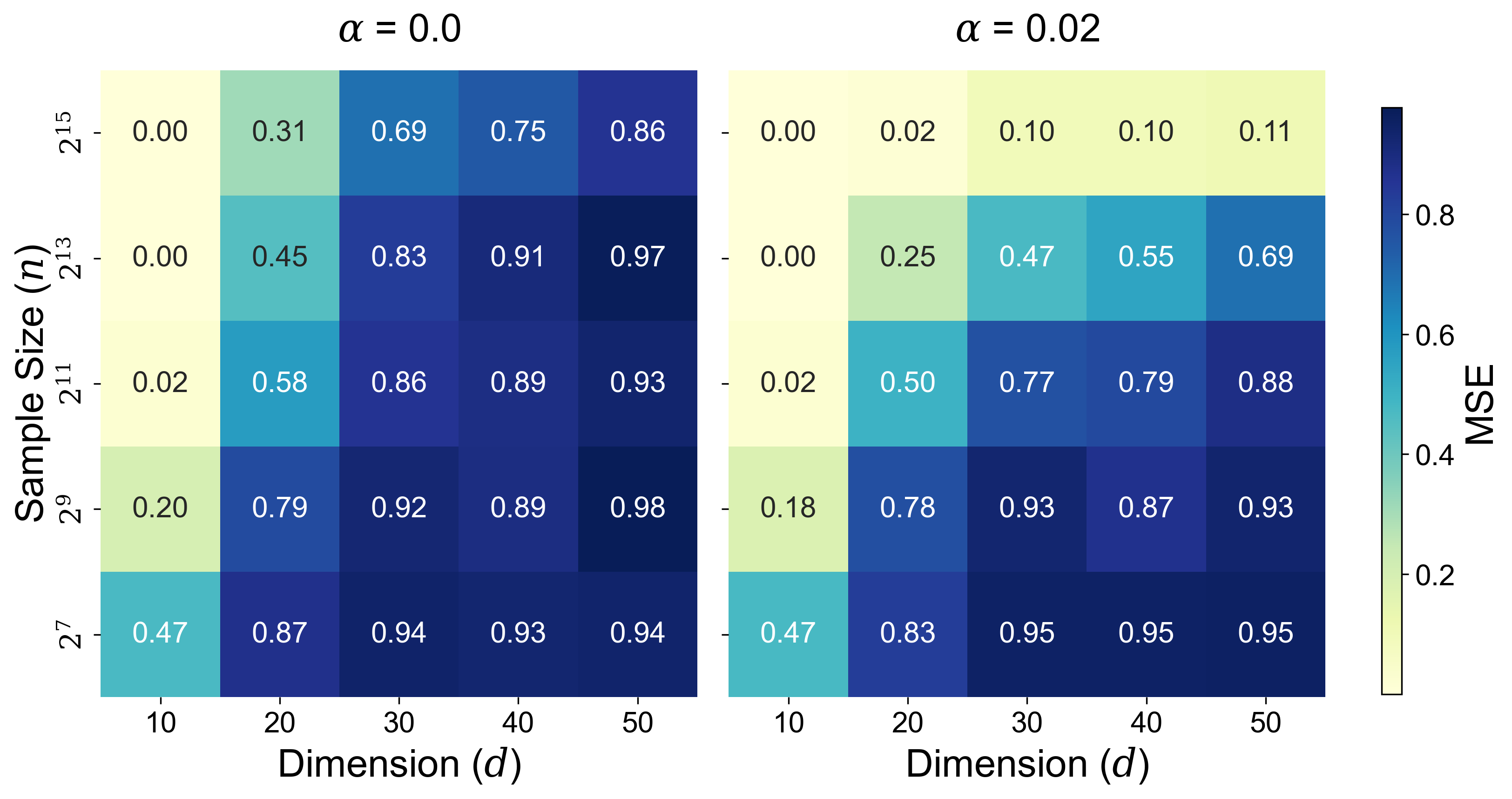}
    \caption{Expected MSE $R(\fcart,f)$ for CART fitted to $Y = X_1X_2 + \alpha X_1 + \epsilon$, with $\alpha = 0$ (left panel) and $\alpha = 0.02$ (right panel) when $\epsilon\sim N(0, 0.01)$. CART is fitted with a minimum impurity decrease stopping criterion.}
    \label{fig:mse_sigma1}
\end{figure}

\begin{figure}[H]
    \centering
    \includegraphics[width=0.9\linewidth]{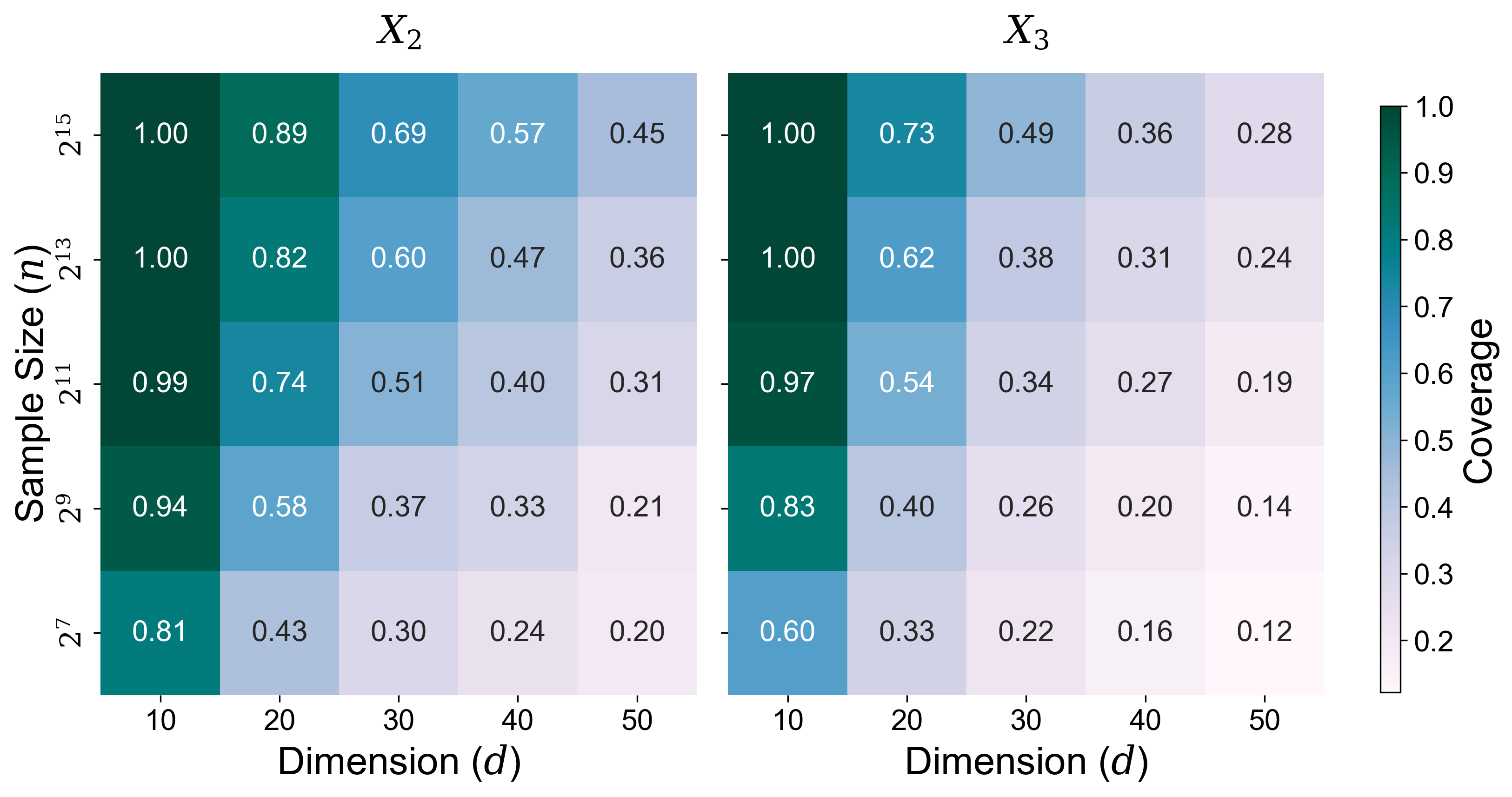}
    \caption{Split coverage for features $X_2$ and $X_3$ for CART fitted to $Y = X_1X_2+\epsilon$ with $\epsilon\sim N(0,0.01)$, without any stopping criterion. Split coverage is defined in \eqref{eq:coverage}.}
    \label{fig:coverage_alpha0_sigma1}
\end{figure}

\begin{figure}[H]
    \centering
    \includegraphics[width=0.9\linewidth]{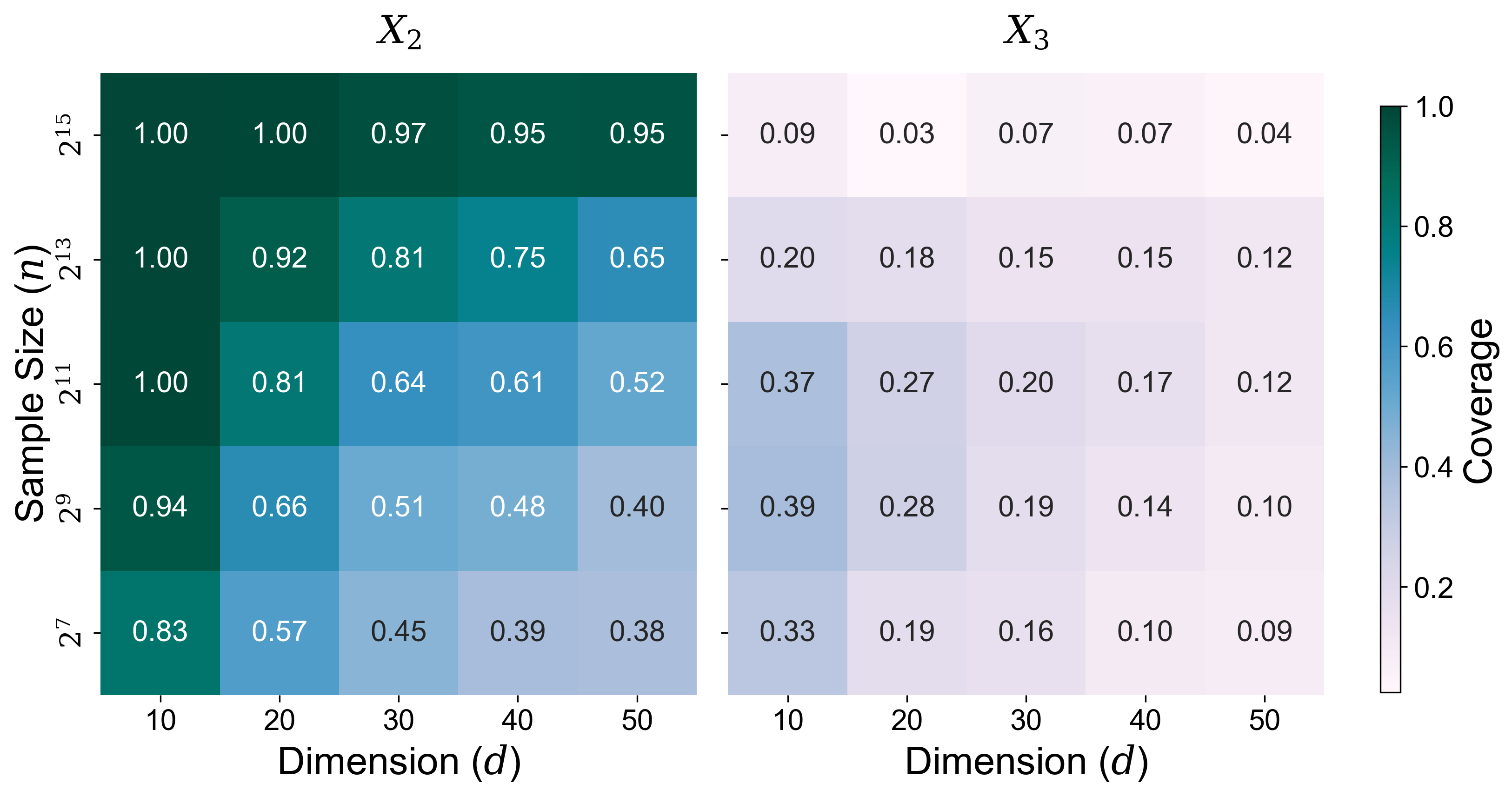}
    \caption{Split coverage for features $X_2$ and $X_3$ for CART fitted to $Y = X_1X_2+\alpha X_1$ with $\alpha=0.02$, without any stopping criterion. Split coverage is defined in \eqref{eq:coverage}.}
    \label{fig:coverage_alpha2_sigma0}
\end{figure}

\begin{figure}[H]
    \centering
    \includegraphics[width=0.9\linewidth]{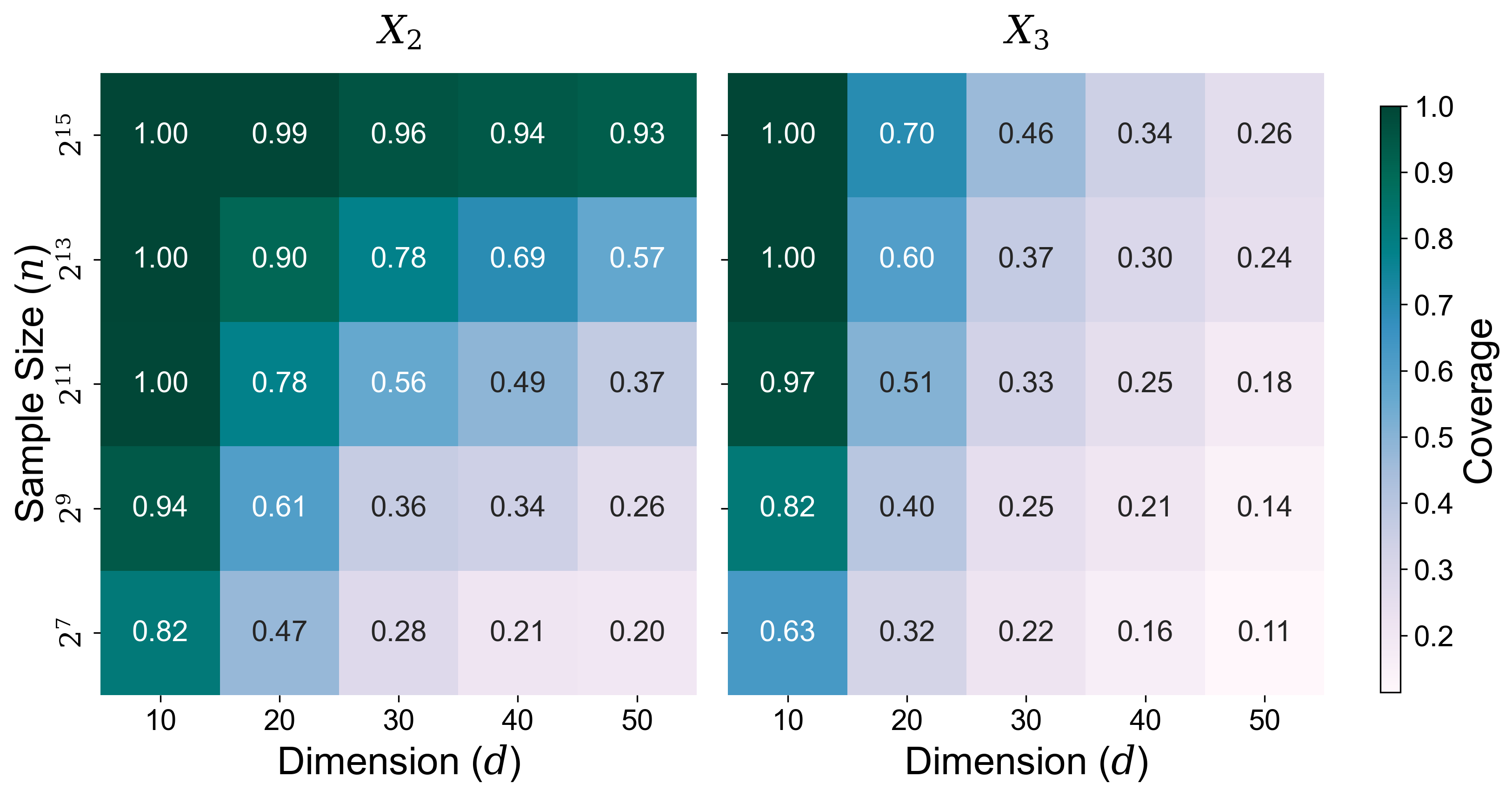}
    \caption{Split coverage for features $X_2$ and $X_3$ for CART fitted to $Y = X_1X_2+\alpha X_1 + \epsilon$ with  $\alpha=0.02$ and $\epsilon\sim N(0,0.01)$, without any stopping criterion. Split coverage is defined in \eqref{eq:coverage}.}
    \label{fig:coverage_alpha2_sigma1}
\end{figure}

\end{document}